\documentclass[twoside]{article}

\usepackage[accepted]{aistats2021}
\usepackage[round]{natbib}

\usepackage[utf8]{inputenc} 
\usepackage[T1]{fontenc}    
\usepackage{url}            
\usepackage{booktabs}       
\usepackage{amsfonts}       
\usepackage{nicefrac}       
\usepackage{microtype}      

 \usepackage{selectp}

\usepackage{amsmath, amssymb}
\usepackage{amsthm}
\usepackage{pifont} 

\usepackage{xspace}

\usepackage{booktabs}
\usepackage{multirow}

\usepackage{algorithm}
\usepackage{algpseudocode}

\usepackage[inline]{enumitem}

\usepackage{graphicx}
\usepackage{subcaption}
\usepackage{caption} 

\usepackage{todonotes}

\DeclareRobustCommand{\eg}{e.g.,\@\xspace}
\DeclareRobustCommand{\ie}{i.e.,\@\xspace}

\usepackage{graphicx}
\usepackage{amsmath}
\usepackage{pifont}

\makeatletter
\def\widebreve{\mathpalette\wide@breve}
\def\wide@breve#1#2{\sbox\z@{$#1#2$}%
	\mathop{\vbox{\m@th\ialign{##\crcr
				\kern0.08em\brevefill#1{0.8\wd\z@}\crcr\noalign{\nointerlineskip}%
				$\hss#1#2\hss$\crcr}}}\limits}
\def\brevefill#1#2{$\m@th\sbox\tw@{$#1($}%
	\hss\resizebox{#2}{\wd\tw@}{\rotatebox[origin=c]{90}{\upshape(}}\hss$}
\makeatletter 

\usepackage{xcolor}
\definecolor{Bleu}{RGB}{0,0,204}
\definecolor{Violet}{RGB}{102,0,204}
\definecolor{Rouge}{RGB}{204,0,0}
\definecolor{Highlight}{RGB}{251,0,0}
\definecolor{darkgreen1}{RGB}{0,100,0}
\definecolor{darkgreen2}{RGB}{0,125,0}
\definecolor{burntorange}{RGB}{150,64,0}
\definecolor{darkgreenblue}{RGB}{0,75,75}
\definecolor{darkpurple}{RGB}{80,20,70}
\definecolor{darkred}{RGB}{100,0,0}
\definecolor{darkblue}{RGB}{0,0,100}
\definecolor{orangered}{RGB}{255,70,0}
\definecolor{deeppink}{RGB}{255,20,147}

\usepackage{hyperref}
\hypersetup{
	colorlinks,
	citecolor=Bleu,
	linkcolor=darkblue,
	urlcolor=Violet}

\def \regret {\mathcal{R}}

\newcommand{\Sdist}[1]{ \rho_{\cX}\left(#1\right) }
\newcommand{\Adist}[1]{ \rho_{\cA}\left(#1\right) }
\newcommand{\Wassdist}[1]{ \mathbb{W}_1\left(#1\right) }

\def \lipschitz {Lipschitz\;}
\def \Sdistfunc {\rho_{\cX} }
\def \Adistfunc {\rho_{\cA} }
\def \distfunc {\rho }
\newcommand{\dist}[1]{ \rho\left[#1\right] }

\newcommand{\dirac}[1]{ \delta_{#1} }

\def \mdp {\mathcal{M}}

\def \actionspace {\mathcal{A}}
\def \statespace {\mathcal{X}}
\def \stateactionspace{\statespace\times\actionspace}
\def \transition {\mathrm{P}}
\def \reward {r}  
\def \obsreward {\widetilde{r}}  

\def \ouralgo {{\color{Violet}\texttt{KeRNS}}\xspace}
\def \kerns {{\color{Violet}\texttt{KeRNS}}\xspace}
\def \RSkerns {{\color{Violet}\texttt{RS}-}{\color{Violet}\texttt{KeRNS}}\xspace}

\def \bonus {\;\mathbf{\mathtt{B}}}    
\def \rbonus {{}^r\mathbf{\mathtt{B}}} 
\def \pbonus {{}^p\mathbf{\mathtt{B}}} 

\def \weight {w}                  
\def \normweight {\widetilde{w}}  
\def \estP {\widehat{P}}          
\def \estR {\widehat{r}}          
\def \kbeta {{\color{orangered}\beta}} 
\def \estMDP {{\color{violet}\widehat{\mathcal{M}}}}          
\def \kernsMDP {{\color{darkred}\widebreve{\mathcal{M}}}}          

\def \kernsR {{\color{darkred}\widebreve{r}}}          
\def \kernsP {{\color{darkred}\widebreve{P}}}          
\def \kernsBonus {{\color{darkred}\widebreve{\bonus}}}          
\def \kernsW{{\color{darkred}\widebreve{W}}} 
\def \kernsN{{\color{darkred}\widebreve{N}}} 
\def \kernsSumR{{\color{darkred}\widebreve{S}}} 
 
\def \kernsV {{\color{darkred}\widebreve{V}}}          
\def \kernsQ {{\color{darkred}\widebreve{Q}}}          

\def \kernsPolicy {{\color{darkred}\widebreve{\pi}}} 
\def \kernsDelta {{\color{darkred}\widebreve{\delta}}} 
\def \history {\cH}

\def \fullkernel {{\color{violet}\Gamma}}

\def \kernel {\overline{{\color{violet}\Gamma}}}
\def \spacekernel {\phi}

\def \spacekernelconstA {C_1}
\def \spacekernelconstB {C_2}
\def \ksigma {{\color{red}\sigma}}

\def \keta {{\color{burntorange}\eta}}  
\def \kW {{\color{darkgreenblue}W}}  
\def \timekernel {\chi_{(\keta, \kW)}}
\def \timekernelshort {\chi}
\def \timekernelconstA {C_3}
\def \timekernelconstB {G}

\def \lipP {{\color{darkgreen1}L_\mathrm{p}}} 
\def \lipR {{\color{darkgreen2}L_\mathrm{r}}} 
\def \lipQ {{\color{blue}L}} 

\def \gencount {\mathbf{C}}
\def \partitioncount {\mathbf{N}}

\def \cdelta{{\color{cyan}\delta}}

\def \timeForMaxA {{\color{darkgreen1}\tau_A}}

\def \goodevent {{\color{blue}\cG}}

\def \XAcovdim {{\color{Violet}d_1}}
\def \Xcovdim {{\color{Violet}d_2}}
\def \totalcovdim {{\color{Violet}d}}

\def \Nstates  {{\color{darkred}X}}
\def \Nactions {{\color{darkred}A}}

\def \NreprStates  {{\color{darkred}\bar{X}}}
\def \NreprActions {{\color{darkred}\bar{A}}}
\def \NreprNextStates {{\color{darkred}\bar{Y}}}

\def \ReprStates     {{\color{darkred}\bar{\mathcal{X}}}}
\def \ReprNextStates {{\color{darkred}\bar{\mathcal{Y}}}}
\def \ReprActions    {{\color{darkred}\bar{\mathcal{A}}}}

\def \bernbiasone {\theta_{\mathrm{b}}^3}
\def \bernbiastwo {\theta_{\mathrm{b}}^4}

\def \bonusbiasP { \mathbf{b}_{\mathrm{p}} }

\def \bonusbiasR { \mathbf{b}_{\mathrm{r}} }

\def \loghoeffdingR {\square_1^\mathrm{r}} 
\def \loghoeffdingP {\square_1^\mathrm{p}} 
\def \unifloghoeffdingP {\square_2^\mathrm{p}} 

\def \unifbiashoeffdingONE{\theta_{\mathrm{b}}^1}
\def \unifbiashoeffdingTWO{\theta_{\mathrm{b}}^2}

\def \logbernstein {\square_3}

\def \bias {\;\mathbf{bias}}
\def \biasR {\;\mathbf{bias}_\mathbf{r}}
\def \biasP {\;\mathbf{bias}_\mathbf{p}}

\def \tQ{\widetilde{Q}}   
\def \trueQ {\mathrm{Q}}  
\def \trueV {\mathrm{V}}  
\def \algQ {Q}  
\def \algV {V}  
\def \upperQ {Q^+} 
\def \upperV {V^+} 

\def \trueP{\transition}
\def \trueR{\reward}

\def \avmdp {\mdp^{\mathrm{av}}}
\def \avP {\overline{P}}  
\def \avR {\overline{r}}  

\def \variationR {{\color{darkpurple}\Delta^\mathrm{r}}}
\def \variationP {{\color{darkpurple}\Delta^\mathrm{p}}}
\def \variationMDPtotal {{\color{darkpurple}\Delta}}

\def \tx {\widetilde{x}}
\def \ta {\widetilde{a}}
\def \ty {\widetilde{y}}

\def \bx {\overline{x}}
\def \ba {\overline{a}}
\def \by {\overline{y}}

\def \maptoReprS {{\color{deeppink}\overline{\zeta}}}
\def \maptoReprSA {{\color{deeppink}\zeta}}
\def \mapzeta {{\color{deeppink}\zeta}}

\def \maxdist {{\color{red}\varepsilon}}
\def \Smaxdist {{\color{red}\varepsilon_{\cX}}}
\def \MAXDISTsigmacov { \left|{\color{red}\cC_{\maxdist}}\right| } 
\def \SMAXDISTsigmacov { \left|{\color{red}\cC_{\Smaxdist}'}\right| } 

\def \logplus{\log^+}

\newcommand{\boundedlipschitzclass}[2]{ \cL\left(#1, #2\right) }

\newcommand{\XAcovnumber}[1]{ \mathcal{N}\left(#1, \stateactionspace,  \distfunc\right) }

\newcommand{\Xcovnumber}[1]{ \mathcal{N}\left(#1, \statespace,  \Sdistfunc\right) }
\def \Xsigmacov { \left|{\color{red}\cC_\sigma'}\right| } 
\def \sigmacov { \left|{\color{red}\cC_\sigma}\right| } 
\def \sigmacovset { {\color{red}\cC_\sigma} } 

\newcommand{\XAcoverset}[1]{ \mathcal{C}_{\stateactionspace}\left(#1\right) }

\def \kernelucbvi {\texttt{Kernel-UCBVI}\xspace}
\def \RSkernelucbvi {\texttt{RS-Kernel-UCBVI}\xspace}
\def \RestartBaseline {\texttt{RestartBaseline}\xspace}

\def \dingone {\text{\ding{192}}\xspace}
\def \dingtwo {\text{\ding{193}}\xspace}

\def \termA {\mathbf{(A)}}
\def \termB {\mathbf{(B)}}
\def \termC {\mathbf{(C)}}
\def \termD {\mathbf{(D)}}

\def \eqdef { \overset{\mathrm{def}}{=} }
\newcommand{\pa}[1]{ \left(#1\right) }
\newcommand{\abs}[1]{ \left|#1\right| }
\newcommand{\braces}[1]{ \left\lbrace#1\right\rbrace  }
\newcommand{\sqrbrackets}[1]{\left[ #1 \right]}

\newcommand{\given}{\Big|}

\newcommand{\prob}[2][]{ \mathbb{P}_{#1} \left[ #2 \right] }
\newcommand{\expect}[2][]{ \mathbb{E}_{#1} \left[ #2 \right] }
\newcommand{\variance}[2][]{ \mathbb{V}_{#1} \left[ #2 \right] }

\def \rmd {\mathrm{d}}

\newcommand{\BigO}[1]{ \mathcal{O}\pa{#1} }
\newcommand{\BigOtilde}[1]{ \widetilde{\mathcal{O}}\pa{#1} }

\newcommand{\indic}[1]{ \mathbb{I}\braces{#1} }
\newcommand{\ceil}[1]{\left\lceil #1 \right\rceil}

\DeclareMathOperator*{\argmin}{argmin}
\DeclareMathOperator*{\argmax}{argmax}
\newcommand{\norm}[1]{ \left\Vert #1 \right\Vert }
\newcommand{\normm}[1]{ \Vert #1 \Vert }

\newtheorem{assumption}{Assumption}
\newtheorem{fact}{Fact}
\newtheorem{lemma}{Lemma}
\newtheorem{proposition}{Proposition}
\newtheorem{theorem}{Theorem}
\newtheorem{definition}{Definition}
\newtheorem{corollary}{Corollary}
\newtheorem{remark}{Remark}
\newtheorem{example}{Example}

\usepackage[most]{tcolorbox}
\newtcolorbox{blockquote}{colback=orange!15!white,boxrule=0pt}

\newenvironment{fcorollary}
{\begin{blockquote}\begin{corollary}}
		{\end{corollary}\end{blockquote}}
\newenvironment{fdefinition}
{\begin{blockquote}\begin{definition}}
		{\end{definition}\end{blockquote}}
\newenvironment{flemma}
{\begin{blockquote}\begin{lemma}}
		{\end{lemma}\end{blockquote}}
\newenvironment{ftheorem}
{\begin{blockquote}\begin{theorem}}
		{\end{theorem}\end{blockquote}}

\newcommand{\cA}{\mathcal{A}}

\newcommand{\cC}{\mathcal{C}}

\newcommand{\cF}{\mathcal{F}}
\newcommand{\cG}{\mathcal{G}}
\newcommand{\cH}{\mathcal{H}}

\newcommand{\cL}{\mathcal{L}}

\newcommand{\cN}{\mathcal{N}}

\newcommand{\cT}{\mathcal{T}}

\newcommand{\cX}{\mathcal{X}}

\newcommand{\NN}{\mathbb{N}}
\newcommand{\RR}{\mathbb{R}}

\usepackage{minitoc}

\begin{document}

\doparttoc 
\faketableofcontents 

%
\runningtitle{A Kernel-Based Approach to Non-Stationary RL in Metric Spaces}

%
\runningauthor{Omar D.\,Domingues, Pierre M\'enard,  Matteo Pirotta,  Emilie Kaufmann, Michal Valko}

\twocolumn[

\aistatstitle{A Kernel-Based Approach to Non-Stationary \\ Reinforcement Learning in Metric Spaces}

\aistatsauthor{ Omar D.\,Domingues${}^{1,2}$ \And Pierre M\'enard${}^{3}$\And  Matteo Pirotta${}^{4}$ \And  Emilie Kaufmann${}^{1,2,5}$ \And Michal Valko${}^{1,6}$ }
\aistatsaddress{${}^1$Inria Lille \ ${}^2$Universit\'e de Lille \ ${}^3$OvGU \ ${}^4$Facebook AI Research \ ${}^5$CNRS \  ${}^6$DeepMind Paris}
]

\begin{abstract}
  	In this work, we propose \kerns: an algorithm for episodic reinforcement learning in non-stationary Markov Decision Processes (MDPs) whose state-action set is endowed with a metric.
  	Using a non-parametric model of the MDP built with time-dependent kernels, we prove a regret bound that scales with the covering dimension of the state-action space and the total variation of the MDP with time, which quantifies its level of non-stationarity.  Our method generalizes previous approaches based on sliding windows and exponential discounting used to handle changing environments. We further propose a practical implementation of \kerns, we analyze its regret and validate it experimentally.
\end{abstract}

\section{Introduction}

In reinforcement learning (RL), an agent interacts with an environment by sequentially taking actions, receiving rewards and observing state transitions. One of the main challenges in RL is the trade-off between exploration, the act of gathering information about the environment, and exploitation, the act of using the current knowledge to maximize the sum of rewards. In non-stationary environments, handling this trade-off becomes much harder: what has been learned in the past may no longer be valid in the present. Therefore, the agent needs to constantly re-explore previously known parts of the environment to discover possible changes. In this work, we propose \ouralgo,\footnote{meaning \underline{Ke}rnel-based \underline{R}einforcement Learning in \underline{N}on-\underline{S}tationary environments.} an algorithm that handles this problem by acting optimistically and by forgetting data that are far in the past, which naturally causes the agent to keep exploring to discover changes. \ouralgo relies on non-parametric kernel estimators of the MDP, and the non-stationarity is handled by using time-dependent kernels.

The regret of an algorithm, defined as the difference between the rewards obtained by an optimal agent and the ones obtained by the algorithm, allows us to quantify how well an agent balances exploration and exploitation. We prove a regret bound for \ouralgo that holds in a challenging setting, where the state-action space can be continuous and the environment can change in every episode, as long as the cumulative changes remain small when compared to the total number of episodes. 

\textbf{Related work}~~ Regret bounds for RL in stationary environments have been extensively studied in finite (tabular) MDPs \citep{jaksch2010near,Azar2017,Dann2017,Jin2018,Zanette2019}, and also in metric spaces under \lipschitz continuity assumptions \citep{Ortner2013a,Song2019,Sinclair2019, domingues2020,sinclair2020adaptiveModelBased}. 
Recent works provide algorithms with regret bounds for non-stationary RL in the tabular setting \citep{Gajane2018,Gajane2019,Cheung2019}. These algorithms estimate the transitions and the rewards in an episode $k$ using the data observed up to episode $k-1$. However, since the MDP can change from one episode to another, these estimators are \emph{biased}. If nothing is done to handle this bias, the algorithms will suffer a linear regret \citep{Gajane2019} that depends on the magnitude of the bias. To deal with this issue, different approaches have been proposed: \cite{Gajane2018} and \cite{Cheung2019} use sliding windows to compute estimators that use only the most recently observed transitions, whereas \cite{Gajane2019} restart the algorithm periodically and, after each restart, new estimators are build and past data are discarded. In the multi-armed bandit literature, in addition to sliding windows, exponential discounting has also been used as a mean to give more importance to recent data \citep{kocsis06DUCB,Garivier11switching,russac2019weighted}.
In this paper, we study the \emph{dynamic regret} of the algorithm, where, in each episode $k$, we compare the learner to the optimal policy of the MDP in episode $k$. A related approach consists in comparing the performance of the learner to the best stationary policy in hindsight, \eg \citep{even2009online, yu2009online, neu2013online, dick2014online}, which is less suited to non-stationary environments, since the performance of any fixed policy can be very bad.
Non-stationary RL has also been studied outside the regret minimization framework, without, however, tackling the issue of exploration. For instance, \cite{Choi2000} propose a model where the MDP varies according to a sequence of tasks whose changes form a Markov chain.  \cite{Szita2002} and~\citet{Csaji2008} study the convergence of Q-learning when the environment changes but remain close to a fixed MDP. Assuming full knowledge of the MDP at each time step, but with unknown evolution, \cite{lecarpentier2019non} introduce a risk-averse approach to planning in slowly changing environments. In a related setting, \cite{lykouris2019corruption} study episodic RL problems where the MDP can be corrupted by an adversary and provide regret bounds in this case.  

\textbf{Contributions}~~ We provide the first regret bound for non-stationary RL in continuous environments. More precisely, we show that the \kernelucbvi algorithm of \citet{domingues2020}, based on non-parametric kernel smoothing, can be modified to tackle non-stationary environments by using appropriate time- and space-dependent kernels. We analyze the resulting algorithm, \kerns, under mild assumptions on the kernel, which in particular recover previously studied forgetting mechanisms to tackle non-stationarity in bandits and RL:  sliding windows \citep{Gajane2018} and exponential discounting \citep{kocsis06DUCB,Garivier11switching,russac2019weighted}, and allow for combinations between those. 
On the practical side, kernel-based approaches can be very computationally demanding since their complexity grows with the number of data points. Building on the notion of representative states, promoted in previous work on practical kernel-based RL  \citep{kveton2012kernel, barreto2016practical} we propose an efficient version of \kerns, called \RSkerns, which has constant runtime per episode. We analyze the regret of \RSkerns, showing that it enables a trade-off between regret and runtime, and we validate this algorithm empirically.

\section{Setting}

\paragraph{Notation} For any $n \in \NN^*$, let $[n] \eqdef \braces{1,\ldots, n}$. If $\mu$ and $P(\cdot|x, a)$ are measures for any $(x,a)$ and $f$  is an arbitrary function, we define $\mu f \eqdef \int f(y)\rmd\mu(y)$ and $Pf(x,a) \eqdef \int f(y)\rmd P(y|x, a)$.\footnote{See also Table \ref{tab:notations} in Appendix~\ref{app:prelimiaries}  summarizing the main notations used in the paper and in the proofs.}

\paragraph{Non-stationary MDPs} We consider an episodic RL setting where, in each episode $k \in [K]$, an agent interacts with the environment for $H \in \NN^*$ time steps. The time is indexed by $(k, h)$, where $k$ represents an episode and $h$ the time step within the episode.
The environment is modeled as a non-stationary MDP, defined by the tuple $\pa{\statespace, \actionspace, \reward, \transition}$, where $\statespace$ is the state space, $\actionspace$ is the action space, $\reward = \braces{\reward_h^k}_{k, h}$ and $\transition = \braces{\transition_h^k}_{k,h}$ are sets of reward functions and transition kernels, respectively. More precisely, when taking action $a$ in state $x$ at time $(k, h)$, the agent observes a random reward $ \obsreward_h^k\in[0, 1]$ with mean $\reward_h^k(x, a)$ and makes a transition to the next state according to the probability measure $\transition_h^k(\cdot|x, a)$. A deterministic policy $\pi$ is a mapping from $[H]\times\statespace$ to $\actionspace$, and we denote by $\pi(h, x)$ the action chosen in state $x$ at step $h$. The action-value function of a policy $\pi$ in step $h$ of episode $k$ is defined as 
\begin{align*}
	\trueQ_{k,h}^\pi(x, a) \eqdef \expect{\sum_{h'=h}^H \reward_{h'}^k(x_{h'}, a_{h'})\given x_h=x, a_h=a}
\end{align*}
where $x_{h'+1}\sim \transition_{h'}^k(\cdot|x_{h'}, a_{h'}),\; a_{h'} = \pi(h',x)$, and its value function is defined by $\trueV_{k,h}^\pi(x)= \trueQ_{k,h}^\pi(x, \pi(h, x))$. The optimal value functions, $\trueV_{k,h}^*(x) \eqdef \sup_\pi \trueV_{k,h}^\pi(x)$ satisfy the Bellman equations \citep{puterman2014markov}
\begin{align*}
	& \trueV_{k,h}^*(x) = \max_{a\in\actionspace}\trueQ_{k,h}^*(x,a),  \text{ where } \\
	& \trueQ_{k,h}^*(x,a) \eqdef \reward_h^k(x,a) + \transition_h^k \trueV_{k,h+1}^*(x, a)
\end{align*}
and where $\trueV_{k,H+1}^* = 0$ by definition.

\paragraph{Dynamic regret} The agent interacts with the environment in a sequence of episodes and, in each episode $k$, it uses a policy $\pi_k$ that can be chosen based on its observations from previous episodes. We measure its performance by the dynamic regret, defined as the sum over all episodes of the difference between the optimal value function in episode $k$ and the value of $\pi_k$:
\begin{align*}
 \regret(K) \eqdef \sum_{k=1}^K \pa{ \trueV_{k,1}^*(x_1^k) - \trueV_{k,1}^{\pi_k}(x_1^k)} 
\end{align*}
where $x_1^k$ is the starting state in each episode, which is chosen arbitrarily and given to the learner.

\paragraph{Assumptions} Since regret lower bounds scale with the number of states and actions \citep{jaksch2010near}, structural assumptions are needed in order to enable learning in continuous MDPs. A common assumption is that rewards and transitions are Lipschitz continuous with respect to some known metric \citep{Ortner2013a, Song2019, domingues2020, sinclair2020adaptiveModelBased}, which is the approach that we follow in this work. We make no assumptions regarding how the MDP changes, and our regret bounds will be expressed in terms of its total variation over time. 

\begin{assumption}
	\label{assumption:metric-state-space}
	The state-action space $\stateactionspace$ is equipped with a metric $\distfunc: (\stateactionspace)^2 \to \RR_+$, which is given to the learner. Also, we assume that there exists a metric $\Sdistfunc$ on $\statespace$ such that, for all $(x, x',a)$,
	$\dist{(x, a), (x', a)} \leq \Sdist{x, x'}$.\footnote{If $(\actionspace, \Adistfunc)$ is also a metric space, we can take $\dist{(x, a), (x', a')} = \Sdist{x, x'}+\Adist{a, a'}$, for instance. See Section 2.3 of \cite{Sinclair2019} for more examples and a discussion.}
\end{assumption}

\begin{assumption}
	\label{assumption:lipschitz-rewards-and-transitions}
	The reward functions are $\lipR$-\lipschitz and the transition kernels are $\lipP$-\lipschitz with respect to the 1-Wasserstein distance: $\forall (x, a, x', a')$ and $\forall (k, h) \in [K]\times[H]$,
	\begin{align*}
	& \abs{\reward_h^k(x,a) - r_h^k(x',a')} \leq \lipR \dist{(x,a), (x',a')}, \text{ and } \\
	& \Wassdist{\transition_h^k(\cdot|x, a), \transition_h^k(\cdot|x', a')} \leq \lipP \dist{(x,a), (x',a')}
	\end{align*}
	where, for two measures $\mu$ and $\nu$, we have
	$\Wassdist{\mu, \nu} \eqdef \sup_{f: \mathrm{Lip}(f) \leq 1}  \int_{\statespace} f(y)(\mathrm{d}\mu(y)-\mathrm{d}\nu(y))$
	and where, for any \lipschitz function $f:\statespace\to\RR$ with respect to $\Sdistfunc$, $\mathrm{Lip}(f)$ denotes its \lipschitz constant.
\end{assumption}

\begin{assumption}
	\label{assumption:q-function-is-lipschitz}
	For any $(k, h)$, the optimal $Q$-function $\trueQ_{k,h}^*$ is $\lipQ$-\lipschitz with respect to $\distfunc$. Assumptions \ref{assumption:metric-state-space} and \ref{assumption:lipschitz-rewards-and-transitions} imply that $\lipQ \leq \sum_{h=1}^H\lipR\lipP^{H-h}$ (Lemma \ref{lemma:value-functions-are-lipschitz} in the Appendix).
\end{assumption}

	\section{An Algorithm for Kernel-Based RL in Non-Stationary Environments}

	In this section, we introduce \ouralgo, a model-based RL algorithm for learning in non-stationary MDPs. In each episode $k$, we estimate the transitions and the rewards using the data observed up to episode $k-1$. Using exploration bonuses that represent the uncertainty in the estimated model, \ouralgo builds a $Q$-function $\algQ_h^k$, and plays the greedy policy with respect to it. \ouralgo generalizes sliding-window and exponential discounting approaches by considering time-dependent kernel functions, which also allow us to handle exploration in continuous environments \citep{domingues2020}.

	\subsection{Kernel-Based Estimators for Changing MDPs}


Let $\fullkernel: \NN \times (\stateactionspace)^2  \to [0, 1]$ be a \emph{non-stationary kernel function}, where $\fullkernel(t, u, v)$ represents the similarity between two state action pairs $u, v$ in $\stateactionspace$ visited at an interval $t$.

	\begin{definition}[kernel weights]
		\label{definition:kernel-weights}
		 Let $(x_h^s, a_h^s)$ be the state-action pair visited at time $(s,h)$. For any $(x, a) \in \stateactionspace$ and $s < k$, we define the weights and the normalized weights at time $(k, h)$ as
		\begin{align*}
		& \weight_h^{k, s}(x, a) \eqdef \fullkernel\pa{k-s-1, (x, a), (x_h^s, a_h^s)}
		\end{align*}
		and $\normweight_h^{k, s}(x, a) \eqdef \weight_h^{k, s}(x, a) /\gencount_h^{k}(x, a)$, 
		where $\gencount_h^{k}(x, a) \eqdef \kbeta + \sum_{s=1}^{k-1} \weight_h^{k, s}(x, a)$ and $\kbeta > 0$ is a regularization parameter.
	\end{definition}

	 Using the kernel function $\fullkernel$ and past data, \ouralgo builds estimators $\estR_h^k$ of the reward function and $\estP_h^k$ of the transitions at time $(k, h)$, which are defined below.

	\begin{definition}[empirical MDP]
		\label{definition:reward-and-transition-estimator}
			At time $(s, h) \in [K]\times[H]$, let $(x_h^s, a_h^s, x_{h+1}^s, \obsreward_h^s)$  represent the state, the action, the next state and the reward observed by the algorithm. Before each episode $k$, \ouralgo estimates the rewards and transitions using the data observed up to episode $k-1$:
			\begin{align*}
			& \estR_h^k(x, a) \eqdef \sum_{s=1}^{k-1} \normweight_h^{k, s}(x, a) \obsreward_h^s
			, 
			\\
			& 
			\estP_h^k (y|x, a) \eqdef  \sum_{s=1}^{k-1} \normweight_h^{k, s}(x, a) \dirac{x_{h+1}^s}(y)
		\end{align*}
		where $\dirac{x}$ is the Dirac measure at $x$. Let $\estMDP_k$ be the MDP whose rewards and transitions at step $h$ are $\estR_h^k(x, a)$ and $\estP_h^k (y|x, a)$.%
		\footnote{Since the normalized weights do not sum to 1, $\estP_h^k$ is not a probability kernel. In this case, we suffer a bias of order $\kbeta$ and the property that $\estP_h^k$ is a sub-probability measure is enough for the analysis.}
	\end{definition}

	The weights $\weight_h^{k, s}(x, a)$  measure the influence that the transitions and rewards observed at time $(s, h)$ will have on the estimators for the state-action pair $(x, a)$ at time $(k, h)$. Their sum, $\gencount_h^{k}(x, a)$, is a proxy for the number of visits to $(x, a)$.  Intuitively, the kernel function $\fullkernel$ must be designed in order to ensure that $\weight_h^{k, s}(x, a)$ is small when $(x, a)$ is very far from $(x_h^s, a_h^s)$, with respect to the distance $\distfunc$. It must also be small when $k-s-1$ is large, which means that the sample $(x_h^s, a_h^s)$ was collected too far in the past and should have a small impact on the estimators. For our theoretical analysis, we will need the assumptions below on the kernel function $\fullkernel$.

	\begin{assumption}[kernel properties]
		\label{assumption:kernel-properties}
		Let $\ksigma > 0$, $\keta \in ]0, 1[$ and $\kW \in \NN$ be the kernel parameters. For each set of parameters, we assume that we have access to a base kernel function $\kernel_{(\keta, \kW)} : \NN \times \RR \to [0, 1]$ and we define, for any $t, u, v \in \NN^* \times \stateactionspace$,
		\begin{align*}
			\fullkernel(t, u, v) = \kernel_{(\keta, \kW)}\pa{t, \dist{u, v}/\ksigma}.
		\end{align*}
		We assume that $z \mapsto \kernel_{(\keta, \kW)}(t, z)$ is non-increasing for any $t \in \NN$. Additionally, we assume that there exists positive constants $\spacekernelconstA, \spacekernelconstB$, a constant $\timekernelconstA\geq 0$ and an arbitrary function $\timekernelconstB:\RR \to \RR_{\geq0}$ that satisfies $\timekernelconstB(4) > 0$ such that
		\begin{align*}
			& 
			\mathbf{(1)}\quad 
			  \forall (t, z), \; \kernel_{(\keta, \kW)}(t, z) \leq \spacekernelconstA \exp\pa{-z^2/2} \\
			& 
			\mathbf{(2)}\quad
			  \forall (t, y, z),\; \abs{\kernel_{(\keta, \kW)}(t, y) - \kernel_{( \keta, \kW)}(t, z)} \leq \spacekernelconstB\abs{y-z} \\
			& 
			\mathbf{(3)}\quad
			  \forall z, \; \kernel_{(\keta, \kW)}(t, z) \leq  \timekernelconstA \keta^{t}, \quad \text{for all  } t \geq \kW \\
			& 
			\mathbf{(4)}\quad
			 \forall z, \; \kernel_{(\keta, \kW)}(t, z)  \geq  \timekernelconstB(z) \keta^{t}, \quad \text{for all  } t < \kW.
		\end{align*}
	\end{assumption}

	We now provide some justification for these conditions. (1) ensures that the bias due to kernel smoothing remains bounded by $\BigOtilde{\ksigma}$ (Lemma~\ref{lemma:kernel-bias}); (2) ensures smoothness conditions that are needed to provide concentration inequalities for the rewards and transitions (Lemma~\ref{lemma:weighted-average-and-bonuses-are-lipschitz}); (3) and (4) allow us to control the bias and the variance due to non-stationarity, respectively (Lemmas~\ref{lemma:bouding-the-temporal-bias} and~\ref{lemma:sum-of-bonus}). Intuitively, (3) says the algorithm should forget data further than $\kW$ episodes in the past, and  (4) says that recent data in the $\kW$ most recent episodes must have a minimum weight. The condition $G(4)>0$ is mostly technical: it is used to ensure that $\gencount_{h}^k(x,a)$ is not too small in a $4\ksigma$-neighborhood of $(x, a)$ (see lemmas \ref{lemma:lower-bound-generalized-counts} and \ref{lemma:sum-of-bonus}). The kernels in the example below satisfy our conditions, and show that they indeed generalize sliding-window and exponential discounting approaches:

	\begin{example}[sliding-window and exponential discount]
		The kernels $\kernel_{(\keta, \kW)}(t, z) = \indic{ t < \kW}\exp(-\abs{z}^p/2)$ (sliding-window)  and $\kernel_{(\keta, \kW)}(t, z) = \keta^{t}\exp(-\abs{z}^p/2)$ (exponential discount) satisfy Assumption \ref{assumption:kernel-properties} for $p \geq 2$.
	\end{example}

	The conditions in Assumption \ref{assumption:kernel-properties} are needed to prove our regret bounds. However, if one has further knowledge about the MDP and its changes, this information can also be integrated to the kernel function $\fullkernel$. For example, if the MDP only changes in certain region of the state-action space, the kernel can be designed to forget past data only in that region. Also, the kernel $\fullkernel$ can be designed to enforce restarts, as proposed by \cite{Gajane2019} for finite MDPs, by setting $\fullkernel(t, u, v)$ to zero every time $t$ exceeds a certain threshold. Although this would require a separate analysis, our proof could be combined to the one of \citep{Gajane2019} to obtain a regret bound in this case.

	\subsection{Algorithm}
	\ouralgo is presented in Algorithm~\ref{alg:kerns}. At time $(k, h)$, let $\bonus_h^k(x, a)$ be the exploration bonus at $(x, a)$ representing the uncertainty of $\estMDP_k$ with respect to the true MDP:
	\begin{align}
		\label{eq:main_text_bonus}
		\bonus_h^k(x, a) =  \BigOtilde{\frac{H}{\sqrt{\gencount_h^k(x, a)}} + \frac{\kbeta H}{\gencount_h^k(x, a)} + \lipQ \ksigma}
	\end{align}
	where $\BigOtilde{\cdot}$ hides logarithmic terms. The exact expression for the bonuses is given in Def.~\ref{def:exploration-bonuses} in Appendix~\ref{app:prelimiaries}. Before starting episode $k$, \ouralgo computes, for all $h \in [H]$, the values $\algQ_h^k$ by running backward induction on $\estMDP_k$, with the bonus $\bonus_h^k(x, a)$ added to the rewards, followed by an interpolation step:
	\begin{align*}
		& 
		\tQ_h^k(x, a) = \estR_h^k(x, a) + \estP_h^k \algV_{h+1}^k(x, a) + \bonus_h^k(x, a) \\
		& 
		\algQ_h^k(x, a) =\textrm{min}_{s\in[k-1]}\pa{\tQ_{h}^k(x_h^s, a_h^s)+\lipQ \dist{(x,a), (x_h^s, a_h^s)}} \\
		&  \algV_h^k(x) = \textrm{min}\pa{H-h+1, \textrm{max}_{a}\algQ_h^k(x, a)}
	\end{align*}

	where $\algV_{H+1}^k \eqdef 0$. The interpolation is needed to ensure that $\algQ_{h}^k$ and $\algV_h^k$ are $\lipQ$-Lipschitz.
	This procedure is defined in detail in Algorithm~\ref{alg:kernel_backward_induction} in Appendix~\ref{app:prelimiaries}, which is the same kind of backward induction used by \kernelucbvi \citep{domingues2020}. Once $\algQ_h^k$ is computed, \ouralgo plays the greedy policy associated to it. Notice that, although $\algQ_{h}^k(x, a)$ and $\algV_h^k(x)$ are defined for all $(x, a)$, they only need to be computed for the states and actions observed by the algorithm up to episode $k$.

	\begin{algorithm}[H]
		\centering
		\caption{\ouralgo}\label{alg:kerns}
		{\footnotesize
		\begin{algorithmic}[1]
			\State {\bfseries Input:} $K$, $H$, $\lipQ$, $\lipR$ , $\lipP$, $\kbeta$, $\cdelta$, $\totalcovdim$, $\ksigma$, $\keta$, $\kW$.
			\State Initialize history: $\cT_h = \emptyset$ for all $h\in[H]$.
			\For{episode $k=1, \ldots, K$}
			\State get initial state $x_1^k$
			\State { \color{darkgreen2}// Run kernel backward induction}
			\State compute $(\algQ_h^k)_h$ using $(\cT_h)_h$ and Algorithm \ref{alg:kernel_backward_induction}.
			\For{$h = 1, \ldots, H$}
			\State execute $a_h^k = \argmax_a \algQ_h^k(x_h^k, a)$
			\State observe reward $\obsreward_h^k$ and next state $x_{h+1}^k$
			\State store transition $\cT_h = \cT_h \cup \braces{ x_h^k, a_h^k, x_{h+1}^k, \obsreward_h^k}$
			\EndFor
			\EndFor
		\end{algorithmic}
		}
	\end{algorithm}

	\subsection{Theoretical guarantees}
	We introduce $\variationMDPtotal$, the total variation of the MDP in $K$ episodes:
	\begin{definition}[MDP variation]
	\label{def:mdp-variation}	
	 We define $\variationMDPtotal = \variationR + \lipQ\variationP$ , where
		\begin{align*}
		& \variationR \eqdef \sum_{i=1}^{K}\sum_{h=1}^H \sup_{x, a}\abs{\trueR_h^i(x, a) - \trueR_h^{i+1}(x, a)}
		,
		\\
		& 
		\variationP \eqdef \sum_{i=1}^{K}\sum_{h=1}^H \sup_{x, a}\Wassdist{\trueP_h^i(\cdot|x, a), \trueP_h^{i+1}(\cdot|x, a)}
		\end{align*}
	\end{definition}
	A similar notion has been introduced, for instance, by \cite{Gajane2019, li2019online} for MDPs and by \cite{besbes2014stochastic} for multi-armed bandits. Here, the difference is that we use the Wasserstein distance to define the variation of the transitions, instead of the total variation (TV) distance $\normm{\trueP_h^i(\cdot|x, a) - \trueP_h^{i+1}(\cdot|x, a)}_1$. This choice was made in order to take into account the metric $\distfunc$ when measuring changes in the environment: our results would be analogous if we had chosen the TV distance.\footnote{More precisely, in the proof of Corollary \ref{corollary:bias-between-avmdp-and-true-mdp}, the Wasserstein distance could be replaced by the TV distance.}


	Using the same algorithm, we provide two regret bounds for \kerns, which are given below. The notation $\lesssim$ omits constants and logarithmic terms (see Definition~\ref{def:lesssim-notation} in Appendix~\ref{app:prelimiaries}).

	\begin{theorem}
		\label{theorem:regret_main_text}
		The regret of \ouralgo is bounded as
		$\regret^{\kerns}(K) \lesssim
		\min\pa{\regret_1(K), \regret_2(K)} + \bias(\ksigma,\keta,\kW,\variationMDPtotal)$,
		where
		\begin{align*}
			& \regret_1(K) = H^2 K \sqrt{\log\frac{1}{\keta}}\sqrt{\Xsigmacov\sigmacov} + H^2\sigmacov K\log\frac{1}{\keta} \\
			&  \regret_2(K) =H^2 K \sqrt{\log\frac{1}{\keta}}\sqrt{\sigmacov}+H^3 \sigmacov\Xsigmacov K\log\frac{1}{\keta} \\
			& \bias(\ksigma,\keta,\kW,\variationMDPtotal) =  \kW\variationMDPtotal H + \frac{\keta^\kW}{1-\keta}KH^3 + \lipQ KH \ksigma
		\end{align*}
		with probability at least $1-\cdelta$. Here, $\Xsigmacov$ and $\sigmacov$ are the $\ksigma$-covering numbers of $(\statespace, \Sdistfunc)$ and $(\stateactionspace,\distfunc)$ respectively, $(\ksigma, \keta, \kW)$ are the kernel parameters.
	\end{theorem}
		\textit{Proof.}\ \ This result comes from combining theorems  \ref{theorem:regret-bound-ucrl-type} and \ref{theorem:regret-bound-ucbvi-type} in Appendix~\ref{app:regret_bounds}. See Section~\ref{sec:proof-outline-main} for a proof outline.

	As discussed below, after optimizing the kernel parameters (Table \ref{tab:regret_with_optimized_constants}), the bound $\regret_1$ has a worse dependence on $K$, and a better dependence on $\variationMDPtotal$. On the other hand, $\regret_2$ is better with respect to $K$, but worse in $\variationMDPtotal$. Concretely, this trade-off may give hints on how to choose the kernel parameters according to the amount of variation that we expect to see in the environment. Technically, the difference comes from how we handle the concentration of the transitions in the proof. To obtain $\regret_1$, we use concentration inequalities on the term $|(\estP_h^k-\trueP_h^k) f|$ for \emph{all} functions $f$ that are bounded and \lipschitz continuous. To obtain $\regret_2$, the concentration is done only for $f = \trueV_{k,h+1}^*$, but this results in larger second-order terms, as in \citep{Azar2017, domingues2020}.

	\begin{corollary}
		Let $\totalcovdim$ be the covering dimension of $(\stateactionspace,\distfunc)$. By optimizing the kernel parameters, we obtain the regret bounds in Table \ref{tab:regret_with_optimized_constants}. Table \ref{tab:regret_with_optimized_constants-full} in Appendix \ref{sec:full-table} gives the values of $(\ksigma, \keta,  \kW)$ that yield these bounds.
	\end{corollary}

	\textit{Proof.}\ \ Assuming that $\Xsigmacov\leq\sigmacov$, we have that $\sigmacov$ and $\Xsigmacov$ are $\BigO{1/\ksigma^\totalcovdim}$. Then, the bounds follow from Theorem \ref{theorem:regret_main_text}.  The general case, handling separately the covering dimensions of $(\stateactionspace,\distfunc)$ and $(\statespace,\Sdistfunc)$, is stated in corollaries~\ref{corollary:regret-bound-ucrl-type-with-optimized-constants} and~\ref{corollary:regret-bound-ucbvi-type-with-optimized-constants} in Appendix \ref{app:regret_bounds}.

	 \textbf{Discussion} \ \ We now discuss regret bounds for optimized kernel parameters, according to the covering dimension of $(\stateactionspace,\distfunc)$, denoted by $\totalcovdim$. Roughly, the covering dimension is the smallest number $\totalcovdim \geq 0$ such that the $\ksigma$-covering number $\sigmacov$ is $\BigO{1/\ksigma^\totalcovdim}$.\footnote{For more details about covering numbers and covering dimension, see Section 3 of \citet{kleinberg2019bandits} and Section 2.2 of \citet{Sinclair2019}.}	 
	 We consider two cases: the tabular (finite MDP) case, where the covering dimension of $(\stateactionspace,\distfunc)$ is $\totalcovdim = 0$, and the continuous case, where $\totalcovdim > 0$.

	\textbf{Tabular case} \ \ Let $\Nstates = \abs{\statespace}$ and $\Nactions = \abs{\actionspace}$. By taking $\ksigma = 0$, we have $\Xsigmacov =  \Nstates$ and $\sigmacov = \Nstates\Nactions$. As shown in Table \ref{tab:regret_with_optimized_constants}, the $\regret_1$ bound states that the regret of \kerns is $\BigOtilde{H^2 \Nstates \sqrt{\Nactions} \variationMDPtotal^{\frac{1}{3}} K^{\frac{2}{3}}}$. This bound matches the one proved by \cite{Gajane2019} for the average reward setting using restarts, up to a factor of $H^{\frac{2}{3}}$ coming from our episodic setting, where the transitions $\trueP_h^k$ depend on $h$. The $\regret_2$ bound states that the regret of \kerns can be improved to $\BigOtilde{H^2 \sqrt{\Nstates\Nactions } \variationMDPtotal^{\frac{1}{3}}K^{\frac{2}{3}}}$, up to second-order terms. In the bandit case ($H=1$), these bounds are \emph{optimal} in terms of $K$ and $\variationMDPtotal$ \citep{besbes2014stochastic}.

	\textbf{Continuous case} \ \ For $\totalcovdim>0$, we prove the first dynamic regret bounds in our setting, which are of order $H^2 \variationMDPtotal^{\frac{1}{3}} K^{\frac{2\totalcovdim+2}{2\totalcovdim+3}}$ (better in $\variationMDPtotal$) or $H^2 \variationMDPtotal^{\frac{1}{2}} K^{\frac{2\totalcovdim+1}{2\totalcovdim+2}}$ (better in $K$) for two different tunings of the kernel.
    Deriving a lower bound in the non-stationary case for $\totalcovdim > 0$ is an open problem, even for multi-armed bandits.
	As a sanity-check, we note that in stationary MDPs, for which $\variationMDPtotal = 0$, we recover the regret bound of \kernelucbvi\footnote{Another choice of $\keta$ might allow us to avoid the dependence on $H^3$ of \kernelucbvi and get $H^2$ instead.}~\citep{domingues2020} of $H^3K^{\frac{2\totalcovdim}{2\totalcovdim+1}}$ from the bound $\regret_2$ with $\log(1/\keta)=1/K$, $\kW\to\infty$ and $\ksigma=K^{-\frac{1}{2\totalcovdim+1}}$, which is optimal for $\totalcovdim = 1$ in the (stationary) bandit case \citep{bubeck2011x}.

	In tabular MDPs, we may achieve sub-linear regret as long as $\variationMDPtotal < K$.\footnote{Notice that, if $\variationMDPtotal$ scales linearly with the number of episodes $K$, we cannot expect to learn. Indeed, according to the lower bound \citep{besbes2014stochastic}, the regret is necessarily linear in this case.} In the continuous case however, our bounds show that we might need $\variationMDPtotal < K^{\frac{3}{2\totalcovdim+3}}$ (for the $\regret_1$ bound) or $\variationMDPtotal < K^{\frac{1}{\totalcovdim+1}}$ (for the $\regret_2$ bound)  in order to avoid a linear regret, which is an immediate consequence of the bounds in Table \ref{tab:regret_with_optimized_constants}.
	
	\begin{table}[ht!]
		\centering
		\caption{Regret for optimized kernel parameters.}
		\label{tab:regret_with_optimized_constants}
		\begin{tabular}{@{}l|cl@{}}
			\toprule
			  & bound & regret \\
			\midrule
			\multirow{2}{*}{$\totalcovdim = 0$}
			    & $\regret_1$      & $H^2 \Nstates \sqrt{\Nactions} \variationMDPtotal^{\frac{1}{3}} K^{\frac{2}{3}}$        \\
			    & $\regret_2$       &  $H^2 \sqrt{\Nstates\Nactions } \variationMDPtotal^{\frac{1}{3}}K^{\frac{2}{3}}    + H^3 \Nstates^2\Nactions\variationMDPtotal^{\frac{2}{3}}K^{\frac{1}{3}}$      \\
			\midrule
			\multirow{2}{*}{$\totalcovdim > 0$}
			  & $\regret_1$         &  $H^2 \variationMDPtotal^{\frac{1}{3}} K^{\frac{2\totalcovdim+2}{2\totalcovdim+3}}$      \\
			  & $\regret_2$ & $H^2 \variationMDPtotal^{\frac{1}{2}} K^{\frac{2\totalcovdim+1}{2\totalcovdim+2}} + H^{\frac{3}{2}} \variationMDPtotal^{\frac{1}{4}}K^{\frac{3}{4}}$        \\
			\cmidrule(l){1-3}
		\end{tabular}
	\end{table}

	\textbf{Knowledge of $\variationMDPtotal$} \ \ To optimally choose the kernel parameters, \kerns requires an upper bound on the variation $\variationMDPtotal$. Recent work has started to tackle this issue in bandit algorithms \citep{chen2019new,auer2019adaptively}, and finite MDPs using sliding windows \citep{Cheung2019}. Their extension to continuous MDPs is left to future work.

\section{Efficient Implementation}

Since \ouralgo uses non-parametric kernel estimators, its computational complexity scales with the number of observed transitions. Let $\timeForMaxA$ be the time required to compute the maximum of $a\mapsto\algQ_h^k(x, a)$. Similarly to \kernelucbvi, its total space complexity is $\BigO{KH}$ and its time complexity per episode $k$ is $\BigO{H k^2 + H\timeForMaxA k}$, resulting in a total runtime of $\BigO{HK^3 + H\timeForMaxA K^2}$. This runtime is very prohibitive in practice, especially in changing environments, where we might need to run the algorithm for a very long time. \cite{domingues2020} propose a version of \kernelucbvi with improved per-episode time complexity of $\BigO{H\timeForMaxA k}$ based on real-time dynamic programming (RTDP) \citep{Barto1995rtdp,efroni2019tight}. However, this requires the upper bounds $\algV_h^k$ to be non-increasing, which is not the case in \ouralgo, since $\algV_h^k$ increases in regions that were not visited recently. This property is necessary to promote extra exploration and adapt to possible changes. Additionally, the RTDP-based approach of \cite{domingues2020} still has a time complexity that scales with time, which can be a considerable issue in practice.
Here, we propose an alternative to run \kerns in \emph{constant} time per episode, while controlling the impact of this speed-up on the regret.

\subsection{Using Representative States and Actions}
As proposed by \cite{kveton2012kernel} and \cite{barreto2016practical}, we take an approach based on using \emph{representative states} to construct an algorithm called \RSkerns (for \kerns on Representative States). In each episode $k$, \RSkerns keeps and updates sets of representative states $\ReprStates_h$, actions $\ReprActions_h$ and next-states $\ReprNextStates_h$, for each $h$, whose cardinalities are denoted by  $\NreprStates_h, \NreprActions_h$ and $\NreprNextStates_h$, respectively. For simplicity, we omit the dependence on $k$ of these sets and their cardinalities.
Every time a new transition $\braces{x_h^k, a_h^k, x_{h+1}^k, \obsreward_h^k}$ is observed, the representative sets are updated using Algorithm \ref{alg:update-representative-states-main-text}, which ensures that any two representative state-action pairs are at a distance greater than $\maxdist$ from each other. Similarly, it ensures that any pair of representative next-states are at a distance greater than $\Smaxdist$ from each other. Then,  $(x_h^k, a_h^k)$ and $x_{h+1}^k$ are mapped to their nearest neighbors in $\ReprStates_h\times\ReprActions_h$ and $\ReprNextStates_h$, respectively, and the estimators of the rewards and transitions are updated. Consequently, we build a finite MDP, denoted by $\kernsMDP_k$, with $\NreprStates_h$ states, $\NreprActions_h$ actions and $\NreprNextStates_h$ next-states, \emph{per stage $h$}. The rewards and transitions of $\kernsMDP_k$ can be stored in arrays of size $\NreprStates_h\NreprActions_h$ and $\NreprStates_h\NreprActions_h\NreprNextStates_h$, for each $h$. 

	\RSkerns is described precisely in Algorithm~\ref{alg:RS-kerns-ONLINE} in Appendix~\ref{app:rs_kerns_new}. It computes a $Q$-function for all $(\bx,\ba)\in\cup_h\ReprStates_h\times\ReprActions_h$ by running backward induction in $\kernsMDP_k$, which is then extended to any $(x,a)\in\stateactionspace$ by performing an interpolation step, as in \kerns. In Appendix~\ref{app:rs_kerns_new}, we explain  how the rewards and transitions estimators of $\kernsMDP_k$ can be updated online. Below, we provide regret and runtime guarantees for this efficient implementation.

\begin{algorithm}[H]
	\centering
	\caption{Update Representative Sets}\label{alg:update-representative-states-main-text}
	{\footnotesize
		\begin{algorithmic}[1]
			\State {\bfseries Input: $k$, $h$, $\ReprStates_h$, $\ReprActions_h$, $\ReprNextStates_h$, $\braces{x_h^k, a_h^k, x_{h+1}^k}$}, $\maxdist$, $\Smaxdist$.
			\If{$\min_{(\bx,\ba)\in\ReprStates_h\times\ReprActions_h}\dist{(\bx,\ba), (x_h^k, a_h^k)} > \maxdist$}
			\State $\ReprStates_h = \ReprStates_h \cup \braces{x_h^k}$, \ \  $\ReprActions_h = \ReprActions_h \cup \braces{a_h^k}$
			\EndIf
			\If{$\min_{\by\in\ReprNextStates_h}\Sdist{\bx, x_{h+1}^k} > \Smaxdist$}
			\State $\ReprNextStates_h = \ReprNextStates_h \cup \braces{x_{h+1}^k}$
			\EndIf
		\end{algorithmic}
	}
\end{algorithm}

\subsection{Theoretical Guarantees \& Runtime}

Theorem \ref{theorem:rs-kerns-regret-main-text} states that \RSkerns enjoys the same regret bounds as \kerns plus a bias term that can be controlled by $\maxdist$ and $\Smaxdist$, as long as we use a Gaussian kernel.

\begin{theorem}[regret of \RSkerns]
	\label{theorem:rs-kerns-regret-main-text}
	Let $\timekernel:\NN \to [0, 1]$, $u, v\in\stateactionspace$, and consider the kernel
	\begin{align*}
	\fullkernel(t, u, v) = \timekernel(t) \exp\pa{-\dist{u, v}^2/(2\ksigma^2)}
	\end{align*}
	assumed to satisfy Assumption \ref{assumption:kernel-properties}. With this choice of kernel, the regret of \RSkerns is bounded by
	\begin{align*}
		\regret(K) \lesssim  \regret^{\kerns}(K)+  \lipQ(\maxdist + \Smaxdist)KH^2 + \frac{\maxdist}{\ksigma}KH^3  
	\end{align*}
	with probability at least $1-\cdelta$, where $\regret^{\kerns}$ is regret bound of \kerns given in Theorem \ref{theorem:regret_main_text}.
\end{theorem}
\textit{Proof.} \ \ This result comes from theorems~\ref{theorem:regret-RS-kerns-ucrl-type} and~\ref{theorem:regret-RS-kerns-ucbvi-type} in Appendix~\ref{app:rs_kerns_new}. See Appendix~\ref{app:proof_outline} for a proof outline.

\begin{lemma}[runtime of \RSkerns]
	Consider the kernel defined in Theorem \ref{theorem:rs-kerns-regret-main-text}, and 
	let $\keta \in ]0, 1]$. 
	If we use the exponential-discount kernel
	$\timekernel(t) = \keta^t$, the per-episode runtime of \RSkerns is bounded by 
	$$\BigO{H \min\pa{k^2, \MAXDISTsigmacov\SMAXDISTsigmacov} + H\min\pa{k,\SMAXDISTsigmacov}\timeForMaxA},$$
	where $\MAXDISTsigmacov$ is the $\maxdist$-covering number of $(\stateactionspace, \distfunc)$, $\SMAXDISTsigmacov$ is the $\Smaxdist$-covering number of $(\statespace, \distfunc)$, and $\timeForMaxA$ is the time required to compute the maximum of $a\mapsto\algQ_h^k(x, a)$.
\end{lemma}
\textit{Proof.} \ \
	By construction, in any episode $k$, we have $\NreprStates_h\NreprActions_h \leq \min\pa{k, \MAXDISTsigmacov}$ and $\NreprNextStates_h \leq \min\pa{ k,\SMAXDISTsigmacov}$. Backward induction (Algorithm~\ref{alg:RS-kernel_backward_induction-ONLINE}) is performed in $\BigO{\sum_h\NreprStates_h\NreprActions_h\NreprNextStates_h + \timeForMaxA\sum_h\NreprNextStates_h }$ time, and the choice of $\timekernel(t) $ implies that the model updates can be done in $\BigO{\sum_h\NreprStates_h\NreprActions_h\NreprNextStates_h}$ time, as detailed in Appendix~\ref{app:online-updates}.

Consequently, the constants $\maxdist$ and $\Smaxdist$  provide a trade-off between regret and computational complexity. Since $\MAXDISTsigmacov= \BigO{\maxdist^{-\XAcovdim}
}$ and $\SMAXDISTsigmacov= \BigO{\Smaxdist^{-\Xcovdim}}$, increasing $(\maxdist, \Smaxdist)$ may reduce \emph{exponentially} the runtime of \RSkerns, while having only a \emph{linear} increase in its regret.

%

\citet{kveton2012kernel} and \citet{barreto2016practical} studied the idea of using representative states to accelerate kernel-based RL (KBRL), but we provide the first regret bounds in this setting.  More precisely, our result improves previous work in the following aspects:
\begin{enumerate*}[label=(\roman*)]
	\item \cite{kveton2012kernel} and \citet{barreto2016practical} do not tackle exploration and do not have finite-time analyses: they provide approximate versions of the KBRL algorithm of \cite{Ormoneit2002} which has asymptotic guarantees assuming that transitions are generated from \emph{independent} samples;
	\item The error bounds of \cite{kveton2012kernel} scale with $\exp(1/\ksigma^2)$. In our online setting, $\ksigma$ can be chosen as a function of the horizon $K$, and their bound could result in an error that scales exponentially with $K$, instead of linearly, as ours. Our result comes from an improved analysis of the smoothness of kernel estimators, that leverages the regularization constant $\kbeta$ (Lemma \ref{lemma:gaussian-case-weighted-average-and-bonuses-are-lipschitz});
	\item \cite{barreto2016practical} propose an algorithm that also builds a set of representative states in an online way. However, their theoretical guarantees only hold when this set is \emph{fixed}, \ie cannot be updated during exploration, whereas our bounds hold in this case; \item unlike \citep{kveton2012kernel, barreto2016practical}, our theoretical results also hold in continuous action spaces.
\end{enumerate*}

\subsection{Numerical Validation}

	To illustrate the behavior of \RSkerns, we consider a continuous MDP whose state-space is the unit ball in $\RR^2$ with four actions, representing a move to the right, left, up or down. The agent starts at $(0, 0)$. Let $b_i^k \in \braces{0, 0.25, 0.5, 0.75, 1}$ and $x_i \in \braces{(0.8, 0.0), (0.0,0.8),(-0.8, 0.0), (0.0,-0.8)}$. We consider the following mean reward function:
	\begin{align*}
		\trueR_h^k(x,a) = \sum_{i=1}^4 b_i^k \max\pa{0, 1 - \frac{\norm{x-x_i}_2}{0.5}} 
	\end{align*}
	which do not depend on $h$. Every $N$ episodes, the coefficients $b_i^k$ are changed, which impact the optimal policy.

\begin{figure}[ht]
	\centering
	\includegraphics[width=0.8\linewidth]{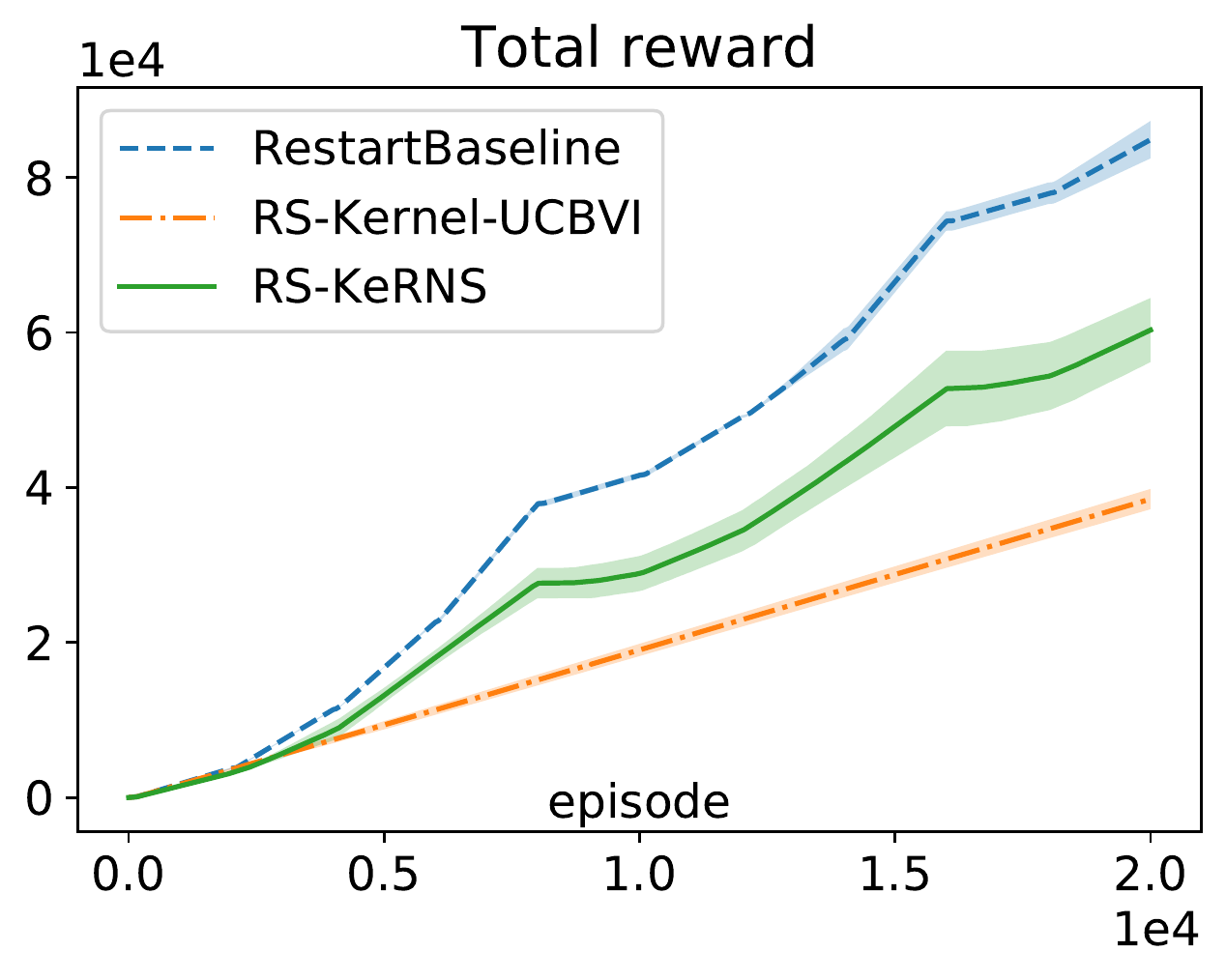}
	\captionof{figure}{ {Performance of \RSkerns compared to baselines for $N=2000$. Average over 4 runs.}}
	\label{fig:experiments-total-reward}
\end{figure}

Taking $\keta = \exp(-(1/N)^{2/3})$, we used the kernel $\fullkernel(t, u, v) = \keta^t \exp\pa{-(\dist{u, v}/\ksigma)^4/2}$. We set $\ksigma = 0.05$, $\maxdist=\Smaxdist=0.1$, $\kbeta=0.01$, $H=15$ and ran the algorithm for $2\times10^4$ episodes. \kerns was compared to two baselines:
\begin{enumerate*}[label=(\roman*)]
	\item \kernelucbvi combined with representative states, that we call \RSkernelucbvi, which is designed for stationary environments. This corresponds to \RSkerns with $\timekernelshort(t)= 1$, that is, $\keta=1$;
	\item A restart-based algorithm, called \RestartBaseline, which is implemented as \RSkernelucbvi, but it has \emph{full information} about when the environment changes, and, at every change, it restarts its reward estimator and its bonuses. We can see that, as expected, \RSkerns outperforms \RSkernelucbvi, which was not designed for non-stationary environments, and is able to ``track'' the behavior of the restart-based algorithm which has full information about how the environment changes. In Appendix~\ref{app:experiments}, we give more details about the experimental setup and provide more experiments, varying the period $N$ of changes in the MDP and the kernel function.
\end{enumerate*}

\section{Proof Outline}
\label{sec:proof-outline-main}

We now outline the proof of Theorem~\ref{theorem:regret_main_text} assuming, for simplicity, that the rewards are known. 

\paragraph{Bias due to non-stationarity}To bound the bias, we introduce an average MDP with transitions $\avP_h^k$: 
\begin{align*}
\avP_h^k (y|x, a) \eqdef \sum_{s=1}^{k-1} \normweight_h^{k, s}(x, a)\transition_h^s(y|x, a) + \frac{\kbeta\, \transition_h^k(y|x, a)}{\gencount_h^k(x, a)},
\end{align*}
where $(\trueP_h^s)_{s,h}$ are the true transitions at time $(s, h)$.  We prove that, for any $\lipQ$-\lipschitz function $f$ bounded by $H$, (Corollary~\ref{corollary:bias-between-avmdp-and-true-mdp}):
\begin{align*}
 &\abs{ \pa{\trueP_h^k- \avP_h^k}f(x, a) } \leq  \biasP(k, h)
\end{align*}
where the term $\biasP(k, h)$ is defined as 
\begin{align*}
	& \lipQ\sum_{i=1\vee(k-\kW)}^{k-1} \sup_{x, a}\Wassdist{\trueP_h^{i}(\cdot|x, a), \trueP_h^{i+1}(\cdot|x, a)}\\
	& + \frac{2\timekernelconstA H}{\kbeta} \frac{\keta^\kW}{1-\keta}\,\cdot
\end{align*}
\paragraph{Concentration} Using concentration inequalities for weighted sums, we prove that $\estP_h^k$ is close to the average transition $\avP_h^k$ using Hoeffding- and Bernstein-type inequalities (lemmas \ref{lemma:transition-hoeffding}, \ref{lemma:reward-hoeffding}, \ref{lemma:transitions-uniform-hoeffding}, and \ref{lemma:transitions-bernstein}), and define an event $\goodevent$ where our confidence sets hold (Lemma \ref{lemma:good-event}), such that $\prob{\goodevent} \geq 1 - \cdelta/2$. For instance, Lemma \ref{lemma:transition-hoeffding} gives us
\begin{align*}
\abs{(\estP_h^k - \avP_h^k)\trueV_{k, h+1}^{*}(x, a)} \lesssim 
\sqrt{  \frac{H^2}{\gencount_h^k(x, a)} } + \frac{\kbeta H}{\gencount_h^k(x, a)} + \lipQ \ksigma \,,
\end{align*}
which explains the form of the exploration bonuses.

\paragraph{Upper bound on the true value function} On the event $\goodevent$, we show that (Lemma \ref{lemma:upper-bound-on-q-functions}):
\begin{align*}
\algQ_h^k(x, a) + \sum_{h'=h}^H \bias(k, h)  \geq \trueQ_{k,h}^*(x, a)
\end{align*}
where the term $\bias(k, h)$ is the sum of $\biasP(k, h)$ defined above, and a similar term representing the bias in the reward estimation.

\paragraph{Regret bounds} Let $(\tx_h^k, \ta_h^k)$ be the state-action pair among the previously visited ones that is the closest to $(x_h^k, a_h^k)$:
\begin{align*}
(\tx_h^k, \ta_h^k)  \eqdef \argmin_{(x_h^s, a_h^s): s<k, h\in[H]} \dist{(x_h^k, a_h^k), (x_h^s, a_h^s)}.
\end{align*}
We show that (see proof of Lemma~\ref{lemma:regret-in-terms-of-sum-of-bonus-ucrl-type}):
\begin{align*}
H \sum_{k=1}^K\sum_{h=1}^H \indic{\dist{(x_h^k, a_h^k), (\tx_h^k, \ta_h^k)} >  2 \ksigma}
\leq H^2\sigmacov.
\end{align*}
Thus, to simplify the outline, for all $(k, h)$, we assume that $\dist{(x_h^k, a_h^k), (\tx_h^k, \ta_h^k)} \leq  2 \ksigma$ and add $H^2\sigmacov$ to the final regret bound.  On the event $\goodevent$, we prove that the regret of \kerns is bounded by (lemmas \ref{lemma:regret-in-terms-of-sum-of-bonus-ucrl-type} and \ref{lemma:regret-in-terms-of-sum-of-bonus}):
\begin{align*}
\regret(K) \lesssim & 
\sum_{k=1}^K\sum_{h=1}^H \pa{\frac{H}{\sqrt{\gencount_h^k(\tx_h^k, \ta_h^k)}} 
	+  \frac{\kbeta H}{\gencount_h^k(\tx_h^k, \ta_h^k)} } \\
& +  \sum_{k=1}^K\sum_{h=1}^H \bias(k, h) 
+ \lipQ K H\ksigma + H^2\sigmacov
\end{align*}
where we omitted factors involving $\sigmacov$ and $\Xsigmacov$ (which depend on the type of bound considered, $\regret_1$ or $\regret_2$), and martingale terms (which are bounded by $\approx H^{3/2}\sqrt{K}$ with probability at least $1-\cdelta/2$).

Using the properties of the kernel $\fullkernel$ (Assumption \ref{assumption:kernel-properties}), we prove that (Lemma \ref{lemma:sum-of-bonus}):

\begin{align*}
& \sum_{k=1}^K\sum_{h=1}^H  \frac{1}{\sqrt{\gencount_h^k(\tx_h^k, \ta_h^k)}} \lesssim HK\log\frac{1}{\keta}\pa{ \sigmacov + \sqrt{\frac{\sigmacov}{\log(1/\keta)}}} \\
& \sum_{k=1}^K\sum_{h=1}^H \frac{1}{\gencount_h^k(\tx_h^k, \ta_h^k)} \lesssim H \sigmacov K\log\frac{1}{\keta}
\end{align*}

Finally, in Corollary \ref{corollary:bound-on-temporal-bias}, we prove that the bias  $\sum_{k=1}^K\sum_{h=1}^H \bias(h, k)$ is bounded by 
\begin{align*}
2\kW \pa{\variationR +  \lipQ \variationP} + \frac{2\timekernelconstA (H+1)KH}{\kbeta}\frac{\keta^\kW}{1-\keta} \cdot
\end{align*}

Putting these bounds together, we prove Theorem \ref{theorem:regret_main_text}. The proof of Theorem~\ref{theorem:rs-kerns-regret-main-text} is outlined in Appendix~\ref{app:proof_outline}.

\section{Conclusion}

In this paper, we introduced and analyzed \kerns, the first algorithm for continuous MDPs with dynamic regret guarantees in changing environments. Building upon previous work on using representative states for kernel-based RL, we proposed \RSkerns, a practical version of \kerns that runs in constant time per episode. Moreover, we provide the first analysis that quantifies the trade-off between the regret and the computational complexity of this approach. In discrete environments, our regret bound matches the existing lower bound for multi-armed bandits in terms of the number of episodes and the variation of MDP, whereas finding a lower bound in continuous environments remains an open problem. 

We believe that the ideas introduced in this paper might be useful for large-scale problems. Indeed, we provide stronger online guarantees for practical kernel-based RL, which has already been shown to be empirically successful in medium-scale environments ($\totalcovdim \approx 10$) \citep{kveton2012kernel, barreto2016practical}, and we show that kernel-based RL is naturally suited to tackle non-stationarity. In larger dimension,  kernel-based exploration bonuses have been recently shown to enhance exploration in RL for Atari games \citep{badia2020never}, and our approach might inspire the design of bonuses for high-dimensional non-stationary environments.

\subsubsection*{Acknowledgements} 
The research presented was supported by European CHIST-ERA project DELTA, French Ministry of Higher Education and Research, Nord-Pas-de-Calais Regional Council,  French National Research Agency project BOLD (ANR19-CE23-0026-04), FMJH PGMO project 2018-0045, and the SFI Sachsen-Anhalt for the project RE-BCI ZS/2019/10/102024 by the Investitionsbank Sachsen-Anhalt.

\bibliography{./source_arxiv_march_2022/library2.bib}
\bibliographystyle{apalike}

\newpage
\onecolumn
\appendix

\part{Appendix}
\parttoc

\newpage
\section{Preliminaries}
\label{app:prelimiaries}

\subsection{Notation}

Throughout the proof, we use the following notation when omitting constants and logarithmic terms:

\begin{fdefinition}
	\label{def:lesssim-notation}
	\begin{align*}
	A \lesssim B \iff A \leq B \times \mathrm{polynomial}\pa{\XAcovdim, \Xcovdim, \log(K), \log(1/\cdelta), \kbeta, 1/\kbeta , \lipR, \lipP}.
	\end{align*}
\end{fdefinition}

Table \ref{tab:notations} summarizes the main notations used in the paper and in the proofs.

\begin{table}[h]
	\centering
	\caption{Table of notations}\label{tab:notations}
	\begin{tabular}{@{}l|l@{}}
		\toprule
		Notation & Meaning \\ \midrule
		$\distfunc: (\stateactionspace)^2 \to \RR_+$ &  metric on the state-action space $\stateactionspace$ \\
		$\Sdistfunc: \statespace^2 \to \RR_+$ &  metric on the state space $\statespace$ \\
		$\cN(\epsilon, \stateactionspace, \distfunc)$ & $\epsilon$-covering number of the metric space $(\stateactionspace, \distfunc)$ \\
		$K$        & number of episodes \\
		$H$        & horizon, length of each episode \\
		$\cdelta$  & confidence parameter \\
 		$\ksigma$  & kernel bandwidth parameter \\
		$\keta$    & kernel temporal decay parameter \\ 
		$\kW$      & kernel temporal window parameter \\ 
		$\kbeta$   & regularization parameter \\
		$\sigmacov$, $\Xsigmacov$  & $\ksigma$-covering numbers of $(\stateactionspace,\distfunc)$ and $(\statespace,\Sdistfunc)$, respectively\\
		$\lipR,\lipP,\lipQ$ & \lipschitz constants of the rewards, transitions and value functions \\
		$\fullkernel$ & kernel function from $\NN^*\times(\stateactionspace)^2$ to $[0, 1]$  \\
		$\kernel_{( \keta, \kW)}$ & parameterized kernel function from $\NN^*\times\RR_+$ to $[0, 1]$  \\
		$\spacekernelconstA, \spacekernelconstB, \timekernelconstA$ & constants related to kernel properties, see Assumption \ref{assumption:kernel-properties}\\ 
		$\timekernelconstB: \RR_+\to\RR_+$ & function related to kernel properties, see Assumption \ref{assumption:kernel-properties}\\
		$\mdp_k$   & true MDP at episode $k$, with rewards $\reward_h^k$ and transitions $\trueP_h^k$\\
		$\estMDP_k$ & empirical MDP built by \kerns in episode $k$ \\
		$\weight_h^{k, s}(x,a)$ & weight at $(x,a)$ in at time $(k,h)$ w.r.t the sample $(x_h^s, a_h^s)$ (Def.~\ref{definition:kernel-weights})\\
		$\normweight_h^{k, s}(x,a)$ & normalized version of $\weight_h^{k, s}(x,a)$ (Def.~\ref{definition:kernel-weights}) \\  
		$(\trueQ_{k,h}^*, \trueV_{k,h}^*)_{h\in[H]}$  & true value functions in episode $k$\\
		$(\algQ_h^k, \algV_h^k)_{h\in[H]}$  & value functions computed by \kerns in episode $k$\\
		$(\kernsQ_{h}^k, \kernsV_{h}^k)_{h\in[H]}$  & value functions computed by \RSkerns in episode $k$\\
		$\ReprStates_h^k, \ReprActions_h^k, \ReprNextStates_h^k$ & sets of representative states, actions and next states, at stage $h$ of episode $k$ \\
		$\maptoReprSA_h^k$ & mapping from $\stateactionspace$ to $\ReprStates_h^{k}\times\ReprActions_h^{k}$\\ 
		$\maptoReprS_h^k$  & mapping from $\statespace$ to $\ReprNextStates_h^{k}$ \\
		$\variationR, \variationP$  & temporal variation of the rewards and transitions (Def. \ref{def:mdp-variation})\\
		$\variationMDPtotal$ & temporal variation of the MDP $= \variationR+\lipQ\variationP$ \\
		$\totalcovdim$ $(=\XAcovdim)$ & covering dimension of $(\stateactionspace,\distfunc)$ \\
		$\Xcovdim$     & covering dimension of $(\statespace, \Sdistfunc)$\\
		$\maxdist$     & threshold distance to add a new representative state-action pair \\
		$\Smaxdist$    & threshold distance to add a new representative state \\
		$\boundedlipschitzclass{\lipQ}{H}$  & set of $\lipQ$-\lipschitz functions from $\statespace$ to $\RR$ bounded by $H$ \\
		$\biasP(k, h)$ & bias in transition estimation at time $(k,h)$, (Def. \ref{def:temporal-bias})\\
		$\biasR(k, h)$ & bias in reward estimation at time $(k,h)$ (Def. \ref{def:temporal-bias})\\
		$\bias(k, h)$ & sum of biases $\biasR(k, h)+\biasP(k, h)$ (Def. \ref{def:temporal-bias})\\
		$\goodevent$  & good event, on which confidence intervals hold (Lemma \ref{lemma:good-event}) \\
		$\logplus(z)$ & equal to $\log(z+e)$ for any $z\in\RR$\\
		\bottomrule
	\end{tabular}
\end{table}
\newpage

\subsection{Probabilistic model}

The interaction between the algorithm and the MDP defines a stochastic process $(x_h^s, a_h^s, x_{h+1}^s, \obsreward_h^s)$ for $h \in [H]$ and $s \in \NN^*$, representing the state, the action, the next state and the reward at step $h$ of episode $s$. Let $\history_h^s \eqdef \braces{x_{h'}^{s'}, a_{h'}^{s'}, x_{h'+1}^{s'}, \obsreward_{h'}^{s'}}_{s' < s, h' \in [H]} \cup \braces{ x_{h'}^s, a_{h'}^s, x_{h'+1}^{s},\obsreward_{h'}^{s}}_{h' < h}$ be the history of the process up to time $(s, h)$. 

We define $\cF_h^s$ as the $\sigma$-algebra generated by $\history_h^s$, and denote its corresponding filtration by $(\cF_h^s)_{s, h}$.

\subsection{Exploration Bonuses and Kernel Backward Induction}

A reinforcement learning algorithm can be seen as a mapping from the set of possible histories $\bigcup_{h\in[H], k\in \NN^*}  \pa{\stateactionspace\times\statespace\times[0, 1]}^{kh-1} $ to the set of actions $\actionspace$.

For a time $(k, h)$, \ouralgo\, performs this mapping in the following way:
\begin{enumerate}
	\item Build $\estR_h^k$ and $\estP_h^k$ as in Definition \ref{definition:reward-and-transition-estimator}, which are $\cF_h^{k-1}$-measurable.
	\item For each $h\in[H]$, with $V_{H+1}^k = 0$,
	
	\begin{itemize}
		\item Compute
		\begin{align*}
		\tQ_h^k(x, a) = \estR_h^k(x, a) + \estP_h^k V_{h+1}^k(x, a) + \bonus_h^k(x, a)  \quad\text{for all } (x, a) \in \braces{ (x_h^s, a_h^s) }_{s < k} \\
		\end{align*}
		\item Define, for any $(x, a)$,
		\begin{align*}
		& \algQ_h^k(x, a) = \min_{s\in[k-1]} \sqrbrackets{ \tQ_h^k(x_h^s, a_h^s) + \lipQ \dist{(x, a), (x_h^s, a_h^s)}  } \\
		& \algV_h^k (x) = \min\pa{H-h+1, \max_{a'}\algQ_h^k(x, a')}
		\end{align*}			
	\end{itemize} 
	
	\item Choose the action $ a_h^k \in \argmax_a \algQ_h^k(x_h^k, a)$.
\end{enumerate}

Notice the algorithmic structure of \ouralgo is the same as \kernelucbvi\citep{domingues2020}. However, \ouralgo uses non-stationary kernels to be able to adapt to changing environments.

It can be checked that Algorithm \ref{alg:kernel_backward_induction} returns the functions $\algQ_h^k$ described above.

The exploration bonus is defined below:
\begin{fdefinition}[exploration bonuses]
	\label{def:exploration-bonuses}
	The exploration bonus in $(x,a)$ at time $(k, h)$ is defined as
	\begin{align*}
		& \bonus_h^k(x, a) \eqdef \rbonus_h^k(x, a) + \pbonus_h^k(x, a)
		, \quad\text{where} \\ 
		& \rbonus_h^k(x, a) \eqdef  \sqrt{  \frac{2 \loghoeffdingR(k,\cdelta/8)}{\gencount_h^k(x, a)} } + \frac{\kbeta }{\gencount_h^k(x, a)} + \bonusbiasR(k, \cdelta/8)\ksigma
		, \quad\text{and}\\
		& \pbonus_h^k(x, a)\eqdef  \sqrt{  \frac{2H^2 \loghoeffdingP(k,\cdelta/8)}{\gencount_h^k(x, a)} } + \frac{\kbeta H}{\gencount_h^k(x, a)} + \bonusbiasP(k, \cdelta/8)\ksigma
	\end{align*}
	where 
	\begin{align*}
		& \loghoeffdingR(k,\cdelta)  = \BigOtilde{\XAcovdim} = \log\pa{\frac{ \XAcovnumber{\ksigma^2/K}\sqrt{1+k/\kbeta}}{\cdelta}} \\ 
		& \bonusbiasR(k, \cdelta) = \BigOtilde{\lipQ + \sqrt{\XAcovdim}}  = \pa{ \frac{\spacekernelconstB }{2\kbeta^{3/2}}\sqrt{ 2 \loghoeffdingR(k,\cdelta)} 
			+ \frac{4 \spacekernelconstB}{\kbeta} } + 2 \lipR \lipQ \pa{ 1+ \sqrt{\logplus(\spacekernelconstA k/\kbeta)} } \\
		& \loghoeffdingP(k,\cdelta) = \BigOtilde{\XAcovdim}  = \log\pa{\frac{H \XAcovnumber{\ksigma^2/KH}\sqrt{1+k/\kbeta}}{\cdelta}} \\ 
		& \bonusbiasP(k, \cdelta) = \BigOtilde{\lipQ + \sqrt{\XAcovdim}}  = \pa{ \frac{\spacekernelconstB }{2\kbeta^{3/2}}\sqrt{ 2 \loghoeffdingP(k,\cdelta)} 
			+ \frac{4 \spacekernelconstB}{\kbeta} } + 2 \lipP \lipQ \pa{ 1+ \sqrt{\logplus(\spacekernelconstA k/\kbeta)} }.
	\end{align*}
	and where $\XAcovdim$ is the covering dimension of $(\stateactionspace, \distfunc)$ and, for any $z\in\RR$, $\logplus(z) = \log(z+e) $.
\end{fdefinition}

\begin{algorithm}[H]
	\centering
	\caption{Kernel Backward Induction with Exploration  Bonuses}\label{alg:kernel_backward_induction}
	\begin{algorithmic}[1]
		\State {\bfseries Input:} $k$, $H$, $\fullkernel$, $\lipQ$, $\kbeta$ , transitions $(x_h^s, a_h^s, x_{h+1}^s, \obsreward_h^s)_{s=1}^{k-1}$ for all $h \in [H]$. 
		\State {\bfseries Initialization: } $\algV_{H+1}(x) = 0$ for all $x\in\statespace$
		\For{$h = H,\ldots, 1$}
		\For{$m=1, \ldots, k-1$}
		\State {\color{darkgreen2} //  Using weights given in Def.\ref{definition:kernel-weights} and bonus given in Def. \ref{def:exploration-bonuses}, compute: }
		\State $\tQ_{h}(x_h^m, a_h^m) = \sum_{s=1}^{k-1}\normweight_h^{k, s}(x_h^m, a_h^m)\pa{\obsreward_h^s+\algV_{h+1}(x_{h+1}^s)} + \bonus_h^k(x_h^m, a_h^m)$
		\EndFor
		\State{\color{darkgreen2} // Interpolated $Q$-function. Defined, but not computed, for all $(x, a)$}
		\State $\algQ_h(x, a) = \underset{m\in[k-1]}{\min}\pa{\tQ_{h}(x_h^m, a_h^m)+\lipQ \dist{(x,a), (x_h^m, a_h^m)}}$ 
		\For{$m=1, \ldots, k-1$}
		\State $\algV_{h}(x_h^m) = \min\pa{H-h+1, \max_{a}\algQ_h(x_h^m, a)}$ 
		\EndFor
		\EndFor 
		\State {\bfseries Return:} $(\algQ_h)_{h\in[H]}$
	\end{algorithmic}
\end{algorithm}

\newpage
\section{Proof Outline}
\label{app:proof_outline}


In this section, we outline the proof of the regret bound of \RSkerns (Theorem \ref{theorem:rs-kerns-regret-main-text}).

\subsection{Theorem \ref{theorem:rs-kerns-regret-main-text}}

To prove the regret bound in Theorem \ref{theorem:rs-kerns-regret-main-text} for \RSkerns, we consider the kernel:
\begin{align*}
	\fullkernel(t, u, v) = \timekernel(t) \spacekernel\pa{u, v}
	, \text{ where }\quad 
	\spacekernel\pa{u, v} \eqdef \exp\pa{-\dist{u, v}^2/(2\ksigma^2)},
\end{align*}
for a given function $\timekernel:\NN \to [0, 1]$. In each episode $k$, \RSkerns has build representative sets of states $\ReprStates_h^k$, actions $\ReprActions_h^k$ and next states $\ReprNextStates_h^k$, for each $h\in[H]$. We define of the projections:
\begin{align*}
\maptoReprSA_h^k(x, a) \eqdef \argmin_{(\bx,\ba) \in \ReprStates_h^k\times\ReprActions_h^k} \dist{(x,a),  (\bx,\ba)}
,\quad 
\maptoReprS_h^k(y) \eqdef \argmin_{\by \in \ReprNextStates_h^k} \Sdist{y, \by}.
\end{align*}
from any $(x, a, y)$ to their representatives.

Let $\kernsW_{h}^{k+1}(x, a) = \sum_{s=1}^{k}\timekernelshort(k-s)\spacekernel\pa{\maptoReprSA_h^{k+1}(x, a), \maptoReprSA_h^{s+1}(x_h^s, a_h^s)}$.
In episode $k+1$, \RSkerns computes the following  estimate of the rewards
\begin{align*}
\kernsR_{h}^{k+1}(x, a) 
& = \frac{1}{ \kbeta + \kernsW_{h}^{k+1}(x, a)} 
\sum_{s=1}^{k} \timekernelshort(k-s)\spacekernel\pa{\maptoReprSA_h^{k+1}(x, a), \maptoReprSA_h^{s+1}(x_h^s, a_h^s)}  \obsreward_h^s
\end{align*}
and the following estimate of the transitions
\begin{align*}
\kernsP_{h}^{k+1}(y|x,a) & = 
\frac{1}{ \kbeta + \kernsW_{h}^{k+1}(x, a)}
\sum_{s=1}^{k} \timekernelshort(k-s)\spacekernel\pa{\maptoReprSA_h^{k+1}(x, a), \maptoReprSA_h^{s+1}(x_h^s, a_h^s)} \dirac{\maptoReprS_h^{s+1}(x_{h+1}^s)}(y)   .
\end{align*}
which are similar to the estimates that would be computed by \kerns, but using the projections $\maptoReprSA$ and $\maptoReprS$ to the representative states and actions. The values of $\kernsR_{h}^{k+1}(x, a) $ and $\kernsP_{h}^{k+1}(y|x,a)$ are defined for all $(x, a, y) \in \stateactionspace\times\statespace$, but they only need to be stored for $(x,a,y)\in\ReprStates_h^k\times\ReprActions_h^k\times\ReprNextStates_h^k$, which corresponds to storing a \emph{finite} representation of the MDP. The exploration bonuses of \RSkerns are defined similarly:
\begin{align*}
& \kernsBonus_{h}^{k+1}(x, a)
\eqdef
\BigOtilde{
	\frac{H}{ \sqrt{\kbeta+\kernsW_{h}^{k+1}(x, a)}}
	+ \frac{\kbeta H}{ \kbeta + \kernsW_{h}^{k+1}(x, a)}
	+ \lipQ \ksigma 
}
\end{align*}

We prove that the estimates used by \RSkerns are close to the ones used by \kerns up to bias terms. Then, this result is used to prove that the regret bound of \RSkerns is the same as \kerns, but adding a bias term multiplied by the number of episodes. For any $(x_h^s, a_h^s)$ with $s<k$ and $h\in[H]$, we show that (consequence of Lemma \ref{lemma:error-in-P-when-usin-representative-states_NEW}):
\begin{align*}
\abs{\pa{\estP_h^k - \kernsP_{h}^k}V(x_h^s, a_h^s)} 
\lesssim & \lipQ \Smaxdist + 8H\frac{\maxdist}{\ksigma}
\end{align*}
and similar bounds are obtained for the rewards $\kernsR_{h}^k(x, a)$ (Lemma \ref{lemma:error-in-R-when-usin-representative-states_NEW}) and the exploration bonuses (Lemma \ref{lemma:error-in-bonus-when-usin-representative-states_NEW}). This allows us to prove that the regret of \RSkerns is bounded as (theorems \ref{theorem:regret-RS-kerns-ucrl-type} and \ref{theorem:regret-RS-kerns-ucbvi-type})
\begin{align*}
	\regret^{\RSkerns}(K) \lesssim  \regret^{\kerns}(K) +  \lipQ(\maxdist + \Smaxdist)KH^2 + \frac{\maxdist}{\ksigma}KH^3.
\end{align*}

If we choose  $\timekernel(t) = \keta^t$, the estimators used by \RSkerns can be updated online. Indeed, as detailed in Appendix \ref{app:rs_kerns_new}, we can related the estimates at time $(k+1, h)$ to the ones at time $(k, h)$:
\begin{align*}
&\kernsW_{h}^{k+1}(\bx, \ba) = \spacekernel\pa{(\bx, \ba), \maptoReprSA_h^{k+1}(x_h^k, a_h^k)} + \keta\kernsW_{h}^{k}(\bx, \ba),
\\
&
\kernsR_{h}^{k+1}(\bx, \ba) 
= \frac{\spacekernel\pa{(\bx, \ba), \maptoReprSA_h^{k+1}(x_h^k, a_h^k)}}{\kbeta +\kernsW_{h}^{k+1}(\bx, \ba)} \obsreward_h^k
+ \keta\cdot\pa{\frac{\kbeta +\kernsW_{h}^{k}(\bx, \ba)}{\kbeta +\kernsW_{h}^{k+1}(\bx, \ba)}} \kernsR_{h}^{k}(\bx, \ba), 
\quad\text{and}
\\
& 
\kernsP_{h}^{k+1}(y | \bx, \ba) =  \frac{\spacekernel\pa{(\bx, \ba), \maptoReprSA_h^{k+1}(x_h^k, a_h^k)}}{\kbeta +\kernsW_{h}^{k+1}(\bx, \ba)} \dirac{\maptoReprS_h^{k+1}(x_{h+1}^k)}(y)
+ \keta\cdot\pa{\frac{\kbeta +\kernsW_{h}^{k}(\bx, \ba)}{\kbeta +\kernsW_{h}^{k+1}(\bx, \ba)}} \kernsP_{h}^{k}(y | \bx, \ba).
\end{align*}

One issue that we need to solve is that $\kernsW_{h}^{k}(\bx, \ba)$, $\kernsR_{h}^{k}(\bx, \ba)$ and $\kernsP_{h}^{k}(y | \bx, \ba)$ were not necessarily computed before episode $k+1$. This happens when $(\bx,\ba)$ is a new representative state-action pair added in episode $k$. In Section \ref{app:online-updates}, we show that this can be easily handled by defining some auxiliary quantities that can be updated online and that can be used to initialize the values $\kernsW_{h}^{k}(\bx, \ba)$, $\kernsR_{h}^{k}(\bx, \ba)$ and $\kernsP_{h}^{k}(y | \bx, \ba)$ when necessary, with little overhead to the runtime of the algorithm.

\subsection{Optimized Kernel Parameters and Regret Bounds}
\label{sec:full-table}
	\begin{table}[ht!]
		\centering
		\caption{Regret bound for optimized kernel parameters, for $\kW =  \log_{\keta}\pa{\pa{1-\keta}/K}$.}
		\label{tab:regret_with_optimized_constants-full}
		\begin{tabular}{@{}c|c|c|l|cl@{}}
			\toprule
			& $\ksigma$ & $\log\pa{\frac{1}{\keta}}$ & condition & bound & regret \\
			\midrule
			\multirow{2}{*}{$\totalcovdim = 0$}
			& $0$   & $\variationMDPtotal^{\frac{2}{3}} K^{-\frac{2}{3}}$    & $\variationMDPtotal < K$ & $\regret_1$      & $H^2 \Nstates \sqrt{\Nactions} \variationMDPtotal^{\frac{1}{3}} K^{\frac{2}{3}}$        \\
			& $0$   &  $\variationMDPtotal^{\frac{2}{3}} K^{-\frac{2}{3}}$  & $\variationMDPtotal < K$    & $\regret_2$       &  $H^2 \sqrt{\Nstates\Nactions } \variationMDPtotal^{\frac{1}{3}}K^{\frac{2}{3}}    + H^3 \Nstates^2\Nactions\variationMDPtotal^{\frac{2}{3}}K^{\frac{1}{3}}$      \\
			\midrule
			\multirow{2}{*}{$\totalcovdim > 0$}
			& $\pa{\frac{1}{K}}^{\frac{1}{2\totalcovdim + 3}}$ &  $\variationMDPtotal^{\frac{2}{3}}K^{-\frac{2\totalcovdim+2}{2\totalcovdim+3}}$ & $\variationMDPtotal <  K^{\frac{3}{2\totalcovdim+3}}$  & $\regret_1$         &  $H^2 \variationMDPtotal^{\frac{1}{3}} K^{\frac{2\totalcovdim+2}{2\totalcovdim+3}}$      \\
			& $\pa{\frac{1}{K}}^{\frac{1}{2\totalcovdim + 2}}$  & $ \frac{\variationMDPtotal^{\frac{1}{2}}}{H}K^{-\frac{2\totalcovdim+1}{2\totalcovdim+2}} $ & $\variationMDPtotal <  K^{\frac{1}{\totalcovdim+1}}$   & $\regret_2$ & $H^2 \variationMDPtotal^{\frac{1}{2}} K^{\frac{2\totalcovdim+1}{2\totalcovdim+2}} + H^{\frac{3}{2}} \variationMDPtotal^{\frac{1}{4}}K^{\frac{3}{4}}$        \\
			\cmidrule(l){1-6}
		\end{tabular}
	\end{table}

\newpage
\section{Handling the bias due to non-stationarity}
\label{app:handling_bias}

\begin{flemma}[temporal bias]
	\label{lemma:bouding-the-temporal-bias}
	Let $(F_h^s)_{h, s}$ be an arbitrary sequence of functions from $\stateactionspace$ to $\RR$ bounded by $M$. Then,
	\begin{align*}
	\abs{\frac{1}{\gencount_h^k(x, a)}   \sum_{s=1}^{k-1} \weight_h^{k, s}(x, a)   \pa{F_h^s(x, a)-F_h^k(x, a)}} \leq \sum_{i=1\vee(k-\kW)}^{k-1}\abs{F_h^{i}(x, a)-F_h^{i+1}(x, a)}
	+ \frac{2M\timekernelconstA}{\kbeta} \frac{\keta^\kW}{1-\keta}\,.
	\end{align*}
\end{flemma}
\begin{proof}
	The result is straightforward when $k \leq \kW$. Assuming $k > \kW$, we have
	\begin{align*}
	& \frac{1}{\gencount_h^k(x, a)}   \sum_{s=1}^{k-1} \weight_h^{k, s}(x, a)   \pa{F_h^s(x, a)-F_h^k(x, a)}\\
	& = \frac{1}{\gencount_h^k(x, a)}   \sum_{s=k-\kW}^{k-1} \weight_h^{k, s}(x, a)  \pa{F_h^s(x, a)-F_h^k(x, a)}
	+ \frac{1}{\gencount_h^k(x, a)}   \sum_{s=1}^{k-\kW-1} \weight_h^{k, s}(x, a)   \pa{F_h^s(x, a)-F_h^k(x, a)} \\
	& \leq \frac{1}{\gencount_h^k(x, a)}   \sum_{s=k-\kW}^{k-1} \weight_h^{k, s}(x, a)\sum_{i=s}^{k-1}\pa{F_h^{i}(x, a)-F_h^{i+1}(x, a)}
	+ \frac{2M\timekernelconstA}{\kbeta}   \sum_{s=1}^{k-\kW-1} \keta^{k-1-s}   \\
	& \leq \frac{1}{\gencount_h^k(x, a)}   \sum_{i=k-\kW}^{k-1}\pa{\sum_{s=k-\kW}^{i} \weight_h^{k, s}(x, a)}\pa{F_h^{i}(x, a)-F_h^{i+1}(x, a)}
	+ \frac{2M\timekernelconstA}{\kbeta} \frac{\keta^\kW - \keta^{k-1}}{1-\keta} \\
	& \leq   \sum_{i=k-\kW}^{k-1}\abs{F_h^{i}(x, a)-F_h^{i+1}(x, a)}
	+ \frac{2M\timekernelconstA}{\kbeta} \frac{\keta^\kW}{1-\keta}\,,
	\end{align*}
	where in the first inequality we used by Assumption~\ref{assumption:kernel-properties} that
	\begin{align*}
	 \weight_h^{k, s}(x, a) &=\fullkernel\pa{k-s-1, (x, a), (x_h^s, a_h^s)}\\
	 &= \kernel_{(\keta, \kW)}(k-s-1, \dist{(x,a), (x_h^s,a_h^s)}) \\
	 &\leq  \timekernelconstA \keta^{k-s-1}\,.
 \end{align*}
	By symmetry, we obtain
	\begin{align*}
	\frac{1}{\gencount_h^k(x, a)}   \sum_{s=1}^{k-1} \weight_h^{k, s}(x, a)   \pa{F_h^k(x, a)-F_h^s(x, a)} \leq \sum_{i=k-\kW}^{k-1}\abs{F_h^{i}(x, a)-F_h^{i+1}(x, a)}
	+ \frac{2M\timekernelconstA}{\kbeta} \frac{\keta^\kW}{1-\keta}.
	\end{align*}
	which concludes the proof.
\end{proof}

\begin{fdefinition}[temporal bias of the MDP]
	\label{def:temporal-bias}
	The temporal bias at time $(k, h)$ is defined by
	\begin{align*}
	\bias(k, h) = \biasR(k, h) + \biasP(k, h)
	\end{align*}
	where
	\begin{align*}
	\biasR(k, h) = \sum_{i=1\vee(k-\kW)}^{k-1}\sup_{x, a}\abs{\reward_h^{i}(x, a)-\reward_h^{i+1}(x, a)}
	+ \frac{2\timekernelconstA}{\kbeta} \frac{\keta^\kW}{1-\keta}
	\end{align*}
	\begin{align*}
	\biasP(k, h) = \lipQ\sum_{i=1\vee(k-\kW)}^{k-1} \sup_{x, a}\Wassdist{\trueP_h^{i}(\cdot|x, a), \trueP_h^{i+1}(\cdot|x, a)}
	+ \frac{2\timekernelconstA H}{\kbeta} \frac{\keta^\kW}{1-\keta}
	\end{align*}
\end{fdefinition}

\begin{fdefinition}[Average MDP at episode $k$]
	\label{def:average_mdp_at_episode_k}
	Let
	\begin{align*}
	\avR_h^k (x, a) \eqdef  \sum_{s=1}^{k-1} \normweight_h^{k, s}(x, a) \reward_h^s(x, a) + \frac{\kbeta}{\gencount_h^k(x, a)}\reward_h^k(y|x, a)
	\end{align*}
	\begin{align*}
	\avP_h^k (y|x, a) \eqdef  \sum_{s=1}^{k-1} \normweight_h^{k, s}(x, a) \transition_h^s(y|x, a) + \frac{\kbeta}{\gencount_h^k(x, a)}\transition_h^k(y|x, a)
	\end{align*}
	and let $\avmdp_k$ be the MDP with transitions $\braces{\avP_h^k}_h$ and rewards $\braces{\avR_h^k}_h$.
\end{fdefinition}

\begin{fcorollary}
	\label{corollary:bias-between-avmdp-and-true-mdp}
	Let $\boundedlipschitzclass{\lipQ}{H}$ be the class of $\lipQ$-\lipschitz functions from $\statespace$ to $\RR$ bounded by $H$. Then,
	\begin{align*}
	\sup_{x, a} & \abs{ \reward_h^k(x, a) - \avR_h^k(x, a) } \leq \biasR(k, h) \\
	\sup_{f \in \boundedlipschitzclass{\lipQ}{H}} &\abs{ \pa{\trueP_h^k- \avP_h^k}f(x, a) } \leq \biasP(k, h)
	\end{align*}
\end{fcorollary}
\begin{proof}
	For the reward term, we have from Lemma \ref{lemma:bouding-the-temporal-bias}:
	\begin{align*}
	\abs{ \reward_h^k(x, a) - \avR_h^k(x, a) }
	& = \abs{ \sum_{s=1}^{k-1}\normweight_h^{k, s}(x, a) \pa{\reward_h^k(x, a) - \avR_h^k(x, a)} } \\
	& \leq \sum_{i=1\vee(k-\kW)}^{k-1}\abs{\reward_h^{i}(x, a)-\reward_h^{i+1}(x, a)}
	+ \frac{2\timekernelconstA}{\kbeta} \frac{\keta^\kW}{1-\keta} \\
	& \leq \biasR(k, h).
	\end{align*}
	For the transitions term, we also apply Lemma \ref{lemma:bouding-the-temporal-bias} and the definition of the $1$-Wasserstein distance:
	\begin{align*}
	\abs{ \pa{\trueP_h^k- \avP_h^k}f(x, a) }
	& = \abs{ \sum_{s=1}^{k-1}\normweight_h^{k, s}(x, a) \pa{\trueP_h^k f(x, a) - \avP_h^k f(x, a)} } \\
	& \leq  \sum_{i=1\vee(k-\kW)}^{k-1}\abs{\trueP_h^{i} f(x, a)-\trueP_h^{i+1} f(x, a)}
	+ \frac{2\timekernelconstA H}{\kbeta} \frac{\keta^\kW}{1-\keta} \\
	& \leq \lipQ\sum_{i=1\vee(k-\kW)}^{k-1} \Wassdist{\trueP_h^{i}(\cdot|x, a), \trueP_h^{i+1}(\cdot|x, a)}
	+ \frac{2\timekernelconstA H}{\kbeta} \frac{\keta^\kW}{1-\keta} \\
	& \leq \biasP(k, h)\,.
	\end{align*}

\begin{remark}
	Since the functions in $\boundedlipschitzclass{\lipQ}{H}$ are bounded, the 1-Wasserstein distance could be replaced by the total variation (TV) distance $\normm{\trueP_h^i(\cdot|x, a) - \trueP_h^{i+1}(\cdot|x, a)}_1$.
\end{remark}
\end{proof}

\newpage
\section{Concentration}

In this Section, we provide confidence intervals that will be used to prove our regret bounds. The main concentration results are presented in Lemma \ref{lemma:good-event}, which defines an event $\goodevent$ where all the confidence intervals hold, and we show that $\prob{\goodevent} \geq 1-\cdelta/2$.

\subsection{Concentration inequalities for weighted sums}

We reproduce here the concentration inequalities for weighted sums proved by \cite{domingues2020}, which we will need.

\begin{lemma}[Hoeffding type inequality \citep{domingues2020}]
	\label{lemma:hoeffding-weighted-sum}
	Consider the sequences of random variables $(w_t)_{t\in\NN^*}$ and $(Y_t)_{t\in\NN^*}$ adapted to a filtration $(\cF_t)_{t\in\NN}$. Assume that, for all $t\geq 1$, $w_t$ is $\cF_{t-1}$ measurable and $\expect{ \exp(\lambda Y_t)\given \cF_{t-1}} \leq \exp(\lambda^2 c^2/2)$ for all $\lambda > 0$.

	Let $S_t \eqdef \sum_{s=1}^t w_s Y_s$ and $V_t \eqdef \sum_{s=1}^t w_s^2$, and assume $w_s \leq 1$ almost surely for all $s$.	Then,for any $\beta > 0$, with probability at least $1-\delta$, for all $t\geq 1$,
	\begin{align*}
	\frac{\abs{S_t}}{\sum_{s=1}^t w_s + \beta} \leq \sqrt{ 2c^2  \log\pa{\frac{\sqrt{1 + t/\beta}}{\delta}}  \frac{1}{\sum_{s=1}^t w_s+\beta}}\,.
	\end{align*}
\end{lemma}
\begin{proof}
	See Lemma 2 of \cite{domingues2020}.
\end{proof}

\begin{lemma}[Bernstein type inequality \citep{domingues2020}]
	\label{lemma:bernstein-freedman-weighted-sum}
	Consider the sequences of random variables $(w_t)_{t\in\NN^*}$ and $(Y_t)_{t\in\NN^*}$ adapted to a filtration $(\cF_t)_{t\in\NN}$. Let
	\begin{align*}
	S_t \eqdef \sum_{s=1}^t w_s Y_s, \quad V_t \eqdef \sum_{s=1}^t w_s^2\expect{Y_s^2\given \cF_{s-1}} \quad \mbox{and} \quad W_t \eqdef \sum_{s=1}^t w_s\,,
	\end{align*}
	Assume that, for all $t\geq 1$,
	\begin{enumerate*}[label=(\roman*)]
		\item $w_t$ is $\cF_{t-1}$ measurable,
		\item $\expect{Y_t\given \cF_{t-1}} = 0$,
		\item $w_t \in [0, 1]$ almost surely,
		\item there exists $b > 0$ such that $\abs{Y_t} \leq b$ almost surely.
	\end{enumerate*}
	Then, for all $\beta > 0$, with probability at least $1-\delta$, for all $t\geq 1$,
	\[
	\frac{|S_t|}{\beta + \sum_{s=1}^t w_s}\leq \sqrt{2\log\big(4e(2t+1)/\delta\big) \frac{V_t+b^2}{\left(\beta + \sum_{s=1}^t w_s\right)^2}} + \frac{2b}{3}\frac{\log\!\big(4e(2t+1)/\delta\big)}{\beta + \sum_{s=1}^t w_s}\,.
	\]
\end{lemma}
\begin{proof}
	See Lemma 3 of \cite{domingues2020}.
\end{proof}


\subsection{Hoeffding-type concentration inequalities}

\begin{flemma}
	\label{lemma:transition-hoeffding}
	For all $(x, a, k, h) \in \stateactionspace \times [K] \times [H]$, we have
	\begin{align*}
	\abs{(\estP_h^k - \avP_h^k)\trueV_{k, h+1}^{*}(x, a)} \leq \sqrt{  \frac{2H^2 \loghoeffdingP(k,\cdelta)}{\gencount_h^k(x, a)} } + \frac{\kbeta H}{\gencount_h^k(x, a)} + \bonusbiasP(k, \cdelta)\ksigma
	\end{align*}
	with probability at least $1-\cdelta$, where
	\begin{align*}
	& \loghoeffdingP(k,\cdelta) = \BigOtilde{\XAcovdim}  = \log\pa{\frac{KH \XAcovnumber{\ksigma^2/(KH)}\sqrt{1+k/\kbeta}}{\cdelta}} \\
	& \bonusbiasP(k, \cdelta) = \BigOtilde{\lipQ + \sqrt{\XAcovdim}}  = \pa{ \frac{\spacekernelconstB }{2\kbeta^{3/2}}\sqrt{ 2 \loghoeffdingP(k,\cdelta)}
		+ \frac{4 \spacekernelconstB}{\kbeta} } + 2 \lipP \lipQ \pa{ 1+ \sqrt{\log(\spacekernelconstA k/\kbeta)} }
	\end{align*}
	and where $\XAcovdim$ is the covering dimension of $(\stateactionspace, \distfunc)$.
\end{flemma}
\begin{proof}
	Let $V = \trueV_{k, h+1}^{*}$. For fixed $(x, a, h)$, we have
	\begin{align*}
	& \abs{(\estP_h^k - \avP_h^k)\trueV_{k, h+1}^{*}(x,a)} \\
	& = \abs{\sum_{s=1}^{k-1} \normweight_h^{k,s}(x, a) \pa{ V(x_{h+1}^s) - \int_{\statespace} V(y) \rmd \transition_h^s(y|x, a)  } - \frac{\kbeta}{\gencount_h^k(x, a)}  \int_{\statespace} V(y) \rmd \transition_h^k(y|x, a)} \\
	& \leq \underbrace{
		\abs{ \sum_{s=1}^{k-1} \normweight_h^{k,s}(x, a) \pa{ V(x_{h+1}^s) - \int_{\statespace} V(y) \rmd \transition_h^s(y|x_h^s, a_h^s)}}
	    }_{\dingone} \\
	&\quad\;  + \underbrace{
		\abs{ \sum_{s=1}^{k-1} \normweight_h^{k,s}(x, a) \pa{\int_{\statespace} V(y) \rmd \transition_h^s(y|x_h^s, a_h^s) - \int_{\statespace} V(y) \rmd \transition_h^s(y|x, a)}}
		}_{\dingtwo} \\
	&\quad\quad\;    + \frac{\kbeta H}{\gencount_h^k(x, a)}\,.
	\end{align*}

	\paragraph{Bounding \dingone (martingale term)} Let $Y_s = V(x_{h+1}^s) -   \transition_h^s V(x_h^s, a_h^s)$. From Lemma~\ref{lemma:hoeffding-weighted-sum}, we have, for a fixed tuple $(x, a, k, h)$,
	\begin{align*}
	\dingone = \abs{ \sum_{s=1}^{k-1} \normweight_h^{k,s}(x, a) Y_s} \leq \sqrt{ 2H^2 \log\pa{\frac{\sqrt{1+k/\kbeta}}{\cdelta}} \frac{1}{\gencount_h^k(x, a)} }
	\end{align*}
	with probability at least $1-\delta$, since $(Y_s)_s$ is a martingale difference sequence with respect to $(\cF_h^s)_s$.

	From Lemma \ref{lemma:weighted-average-and-bonuses-are-lipschitz}, we verify that the functions
	\begin{align*}
	(x, a) \mapsto \sqrt{1/\gencount_h^k(x, a)}
	\quad\text{and}\quad
	(x, a) \mapsto \sum_{s=1}^{k-1} \normweight_h^{k,s}(x, a) Y_s
	\end{align*}
	are \lipschitz continuous, with \lipschitz constants bounded by $\spacekernelconstB k/(2\ksigma\kbeta^{3/2})$ and $4H \spacekernelconstB k/(\kbeta\ksigma)$, respectively. Let $\XAcoverset{\ksigma^2/KH}$ be a $(\ksigma^2/KH)$-covering of $\stateactionspace$. Using the Lipschitz continuity of the functions above and a union bound over $\XAcoverset{\ksigma^2/(KH)}$ and over $h \in [H]$, we have
	\begin{align*}
	\dingone =  \abs{ \sum_{s=1}^{k-1} \normweight_h^{k,s}(x, a) Y_s}
	& \leq
	\sqrt{ 2H^2 \log\pa{\frac{\sqrt{1+k/\kbeta}}{\cdelta}} \frac{1}{\gencount_h^k(x, a)} } \\
	& + \pa{ \frac{\spacekernelconstB k}{2\ksigma\kbeta^{3/2}}\sqrt{ 2H^2 \log\pa{\frac{\sqrt{1+k/\kbeta}}{\cdelta}}}
		+ \frac{4H \spacekernelconstB k}{\kbeta\ksigma} } \frac{\ksigma^2}{KH}
	\end{align*}
	for all $(x, a, k, h)$ with probability at least $1 -  \cdelta K H \XAcovnumber{\ksigma^2/(KH)}$.

	\paragraph{Bounding \dingtwo (spatial bias term)} We have
	\begin{align*}
	\dingtwo  & = \abs{ \sum_{s=1}^{k-1} \normweight_h^{k,s}(x, a) \pa{\int_{\statespace} V(y) \rmd \transition_h^s(y|x_h^s, a_h^s) - \int_{\statespace} V(y) \rmd \transition_h^s(y|x, a)}} \\
	& \leq \lipQ \sum_{s=1}^{k-1} \normweight_h^{k,s}(x, a) \Wassdist{\transition_h^s(\cdot|x_h^s, a_h^s), \transition_h^s(\cdot|x, a)}  \quad \text{by the definition of } \Wassdist{\cdot, \cdot}\\
	& \leq \lipP \lipQ  \sum_{s=1}^{k-1} \normweight_h^{k,s}(x, a) \dist{(x_h^s, a_h^s), (x, a)}  \quad \text{by Assumption \ref{assumption:lipschitz-rewards-and-transitions}} \\
	& \leq 2 \ksigma \lipP \lipQ \pa{ 1+ \sqrt{\logplus(\spacekernelconstA k/\kbeta)} } \quad \text{by Lemma \ref{lemma:kernel-bias}}.
	\end{align*}
	Putting together the bounds for \dingone and \dingtwo concludes the proof.
\end{proof}

\begin{flemma}
	\label{lemma:reward-hoeffding}
	For all $(x, a, k, h) \in \stateactionspace \times [K] \times [H]$, we have
	\begin{align*}
	\abs{\estR_h^k(x, a) - \avR_h^k(x, a)} \leq \sqrt{  \frac{2 \loghoeffdingR(k,\cdelta)}{\gencount_h^k(x, a)} } + \frac{\kbeta }{\gencount_h^k(x, a)} + \bonusbiasR(k, \cdelta)\ksigma
	\end{align*}
	with probability at least $1-\cdelta$, where
	\begin{align*}
	& \loghoeffdingR(k,\cdelta)  = \BigOtilde{\XAcovdim} = \log\pa{\frac{ \XAcovnumber{\ksigma^2/K}\sqrt{1+k/\kbeta}}{\cdelta}} \\
	& \bonusbiasR(k, \cdelta) = \BigOtilde{\lipQ + \sqrt{\XAcovdim}}  = \pa{ \frac{\spacekernelconstB }{2\kbeta^{3/2}}\sqrt{ 2 \loghoeffdingR(k,\cdelta)}
		+ \frac{4 \spacekernelconstB}{\kbeta} } + 2 \lipR \lipQ \pa{ 1+ \sqrt{\log(\spacekernelconstA k/\kbeta)} }
	\end{align*}
	and where $\XAcovdim$ is the covering dimension of $(\stateactionspace, \distfunc)$.
\end{flemma}
\begin{proof}
	Almost identical to the proof of Lemma \ref{lemma:transition-hoeffding}, except for the fact that the rewards are bounded by $1$ instead of $H$.
\end{proof}

\begin{flemma}
	\label{lemma:transitions-uniform-hoeffding}
	Let $\boundedlipschitzclass{2\lipQ}{2H}$ be the class of $2\lipQ$-\lipschitz functions from $\statespace$ to $\RR$ bounded by $2H$. With probability at least $1-\cdelta$, for all $(x, a, h, k) \in \stateactionspace \times [K] \times [H]$ and for all $f \in \boundedlipschitzclass{2\lipQ}{2H}$, we have
	\begin{align*}
		\abs{(\estP_h^k - \avP_h^k)f(x, a)} \leq \sqrt{  \frac{8H^2 \unifloghoeffdingP(k,\cdelta)}{\gencount_h^k(x, a)} } + \frac{2\kbeta H}{\gencount_h^k(x, a)} + \unifbiashoeffdingONE(k, \cdelta)\ksigma^{1+\Xcovdim/2} +  \unifbiashoeffdingTWO(k, \cdelta)\ksigma
	\end{align*}
	where
	\begin{align*}
		& \unifloghoeffdingP(k,\cdelta) = \BigOtilde{\Xsigmacov + \XAcovdim\Xcovdim}  = \log\pa{\frac{K H \XAcovnumber{\ksigma^{2+\Xcovdim/2}/KH}\sqrt{1+k/\kbeta}}{\cdelta} \pa{\frac{2H}{\lipQ\ksigma}}^{\Xcovnumber{\ksigma}} } \\
		& \unifbiashoeffdingONE(k, \cdelta)  = \BigOtilde{\sqrt{\Xsigmacov} + \sqrt{\XAcovdim\Xcovdim}}  =   \frac{4\spacekernelconstB }{\kbeta} + \frac{\spacekernelconstB }{2\kbeta^{3/2}}\sqrt{ 8\unifloghoeffdingP(k,\cdelta)} \\
		& \unifbiashoeffdingTWO(k, \cdelta) = \BigOtilde{\lipQ} = 4 \lipP \lipQ \pa{ 1+ \sqrt{\logplus(\spacekernelconstA k/\kbeta)}} + 32\lipQ
	\end{align*}
	and where $\XAcovdim$ is the covering dimension of $(\stateactionspace, \distfunc)$, $\Xcovdim$ is the covering dimension of $(\statespace, \Sdistfunc)$ and $\Xsigmacov = \Xcovnumber{\ksigma}$.
\end{flemma}
\begin{proof}
	Fix a function $f \in \boundedlipschitzclass{2\lipQ}{2H}$. Proceeding as in the proof of Lemma \ref{lemma:transition-hoeffding}, we have
	\begin{align*}
		& \abs{(\estP_h^k - \avP_h^k)f(x,a)} \\
		& = \abs{\sum_{s=1}^{k-1} \normweight_h^{k,s}(x, a) \pa{ f(x_{h+1}^s) - \int_{\statespace} f(y) \rmd \transition_h^s(y|x, a)  } - \frac{\kbeta}{\gencount_h^k(x, a)}  \int_{\statespace} f(y) \rmd \transition_h^k(y|x, a)} \\
		& \leq  \abs{ \sum_{s=1}^{k-1} \normweight_h^{k,s}(x, a)Y_s(f)}   + 4 \ksigma \lipP \lipQ \pa{ 1+ \sqrt{\logplus(\spacekernelconstA k/\kbeta)} }   + \frac{2\kbeta H}{\gencount_h^k(x, a)}.
	\end{align*}
	where $Y_s(f) = f(x_{h+1}^s) -   \transition_h^s f(x_h^s, a_h^s)$.

	Now, for fixed $(x, a, k, h, f)$, Lemma~\ref{lemma:hoeffding-weighted-sum} gives us
	\begin{align*}
	\abs{ \sum_{s=1}^{k-1} \normweight_h^{k,s}(x, a) Y_s(f)} \leq \sqrt{ 8H^2 \log\pa{\frac{\sqrt{1+k/\kbeta}}{\cdelta}} \frac{1}{\gencount_h^k(x, a)} }
	\end{align*}
	with probability at least $1-\delta$, since $(Y_s(f))_s$ is a martingale difference sequence with respect to $(\cF_h^s)_s$ bounded by $4H$.

	 \paragraph{Covering $\boundedlipschitzclass{2\lipQ}{2H}$} Now let $\cC_{\cL}$ be a $8\lipQ\ksigma$-covering of $\pa{\boundedlipschitzclass{2\lipQ}{2H}, \norm{\cdot}_\infty}$. Using the fact that the function $f \mapsto Y_s(f)$ is 2-\lipschitz with respect to $\norm{\cdot}_\infty$, we do a union bound over  $\cC_{\cL}$ to obtain

	\begin{align*}
		\abs{ \sum_{s=1}^{k-1} \normweight_h^{k,s}(x, a) Y_s(f)} \leq \sqrt{ 8H^2 \log\pa{\frac{\sqrt{1+k/\kbeta}}{\cdelta}} \frac{1}{\gencount_h^k(x, a)} } +  32\lipQ\ksigma
	\end{align*}
	for all $k$ and all $f \in \boundedlipschitzclass{2\lipQ}{2H}$, with probability at least $1-\delta\pa{\frac{2H}{\lipQ\ksigma}}^{\Xcovnumber{\ksigma}}$, since the $8\lipQ\ksigma$-covering  number of  $\cC_{\cL}$ is bounded by $\pa{\frac{2H}{\lipQ\ksigma}}^{\Xcovnumber{\ksigma}}$, by Lemma 5 of \cite{domingues2020}.

	\paragraph{Covering $(\stateactionspace,\distfunc)$} 	By Lemma \ref{lemma:weighted-average-and-bonuses-are-lipschitz}, the functions
	\begin{align*}
	(x, a) \mapsto \abs{ \sum_{s=1}^{k-1} \normweight_h^{k,s}(x, a)Y_s(f)}
	\quad\text{and}\quad
	(x, a) \mapsto \sqrt{\frac{1}{\gencount_h^k(x, a)}}
	\end{align*}
	are $4H\spacekernelconstB k/(\kbeta\ksigma)$-\lipschitz and $\spacekernelconstB k/(2\kbeta^{3/2}\ksigma)$, respectively, with respect to the distance $\distfunc$. Let $\cC_{\stateactionspace}$ be a
	$\ksigma^{2+\Xcovdim/2}/KH$ covering of $(\stateactionspace,\distfunc)$. Using the continuity of the functions above, a union bound over $\cC_{\stateactionspace}$ gives us\footnote{see, for instance, Lemma 6 of \cite{domingues2020}.}
	\begin{align*}
		\abs{ \sum_{s=1}^{k-1} \normweight_h^{k,s}(x, a) Y_s(f)} & \leq \sqrt{ 8H^2 \log\pa{\frac{\sqrt{1+k/\kbeta}}{\cdelta}} \frac{1}{\gencount_h^k(x, a)} } +  32\lipQ\ksigma \\
		& + \frac{\ksigma^{2+\Xcovdim/2}}{KH}\pa{  \frac{4H\spacekernelconstB k}{\kbeta\ksigma} + \frac{\spacekernelconstB k}{2\kbeta^{3/2}\ksigma}\sqrt{ 8H^2 \log\pa{\frac{\sqrt{1+k/\kbeta}}{\cdelta}}}  }
	\end{align*}
	for all $k$, all $f \in \boundedlipschitzclass{2\lipQ}{2H}$ and all $(x, a)\in\stateactionspace$, with probability at least $$1-\delta\pa{\frac{2H}{\lipQ\ksigma}}^{\Xcovnumber{\ksigma}}\XAcovnumber{\ksigma^{2+\Xcovdim/2}/KH}$$
	and a union bound over $(k, h)\in [K]\times[H]$ concludes the proof.
\end{proof}
\subsection{Bernstein-type concentration inequality}
\begin{flemma}
	\label{lemma:transitions-bernstein}
	Let $\boundedlipschitzclass{2\lipQ}{2H}$ be the class of $2\lipQ$-\lipschitz functions from $\statespace$ to $\RR$ bounded by $2H$. With probability at least $1-\cdelta$, for all $(x, a, h, k) \in \stateactionspace \times [K] \times [H]$ and for all $f \in \boundedlipschitzclass{2\lipQ}{2H}$, we have
	\begin{align*}
	\abs{(\estP_h^k - \avP_h^k)f(x,a)}  & \leq
	\frac{1}{H} \trueP_h^k\abs{f}(x,a)
	+  \frac{14 H^2\spacekernelconstB\logbernstein(k,\cdelta) + 2\kbeta H }{\gencount_h^k(x, a)} \\
	&   + \bernbiasone(k, \cdelta) \ksigma^{1+\Xcovdim}
	+ \bernbiastwo(k, \cdelta) \ksigma
	+ \frac{2}{H}\biasP(k, h)
	\end{align*}
	where $\XAcovdim$ is the covering dimension of $(\stateactionspace, \distfunc)$, $\Xcovdim$ is the covering dimension of $(\statespace, \Sdistfunc)$ and
	\begin{align*}
	& \logbernstein(k,\cdelta) = \BigOtilde{\Xsigmacov + \XAcovdim\Xcovdim} = \log\pa{\tfrac{4e(2k+1) }{\cdelta}KH\XAcovnumber{\frac{\ksigma^{2+\Xcovdim}}{H^2K}} \pa{\frac{2H}{\lipQ\ksigma}}^{\Xcovnumber{\ksigma}}} \\
	& \bernbiasone(k, \cdelta)  = \BigOtilde{\Xsigmacov + \XAcovdim\Xcovdim + \lipQ \ksigma} =  \frac{2\lipP\lipQ\ksigma}{H^2K} + \frac{4\spacekernelconstB}{H\kbeta} + \frac{14\logbernstein(k,\cdelta)  \spacekernelconstB}{\kbeta^2} \\
	& \bernbiastwo(k, \cdelta)= \BigOtilde{\lipQ} = 32\lipQ + 6\lipP \lipQ \pa{ 1+ \sqrt{\logplus(\spacekernelconstA k/\kbeta)} }
	\end{align*}
	where $\Xsigmacov = \BigO{1/\ksigma^\Xcovdim}$ is the $\ksigma$-covering number of $(\statespace, \Sdistfunc)$.
\end{flemma}
\begin{proof}
	We have
	\begin{align*}
	& \abs{(\estP_h^k - \avP_h^k)f(x,a)} \\
	& = \abs{ \sum_{s=1}^{k-1} \normweight_h^{k, s}(x, a) \pa{ f(x_{h+1}^s) - \transition_h^s f(x, a)}  -  \frac{\kbeta \transition_h^k f(x, a) }{\gencount_h^k(x, a)} } \\
	&  \leq
	    \underbrace{
		\abs{\sum_{s=1}^{k-1} \normweight_h^{k, s}(x, a) \pa{ f(x_{h+1}^s) - \int_{\statespace} f(y)\rmd P_h^s(y|x_h^s, a_h^s) }}
	          }_{\dingone} \\
	& \quad + \underbrace{
		 \abs{\sum_{s=1}^{k-1} \normweight_h^{k, s}(x, a) \pa{ \int_{\statespace} f(y)\rmd P_h^s(y|x_h^s, a_h^s) - \int_{\statespace} f(y)\rmd P_h^s(y|x, a) }}
		}_{\dingtwo}
	& + \frac{2 \kbeta H }{\gencount_h^k(x, a)}\,.
	\end{align*}

	\paragraph{Bounding \dingtwo (spatial bias term)} As in the proof of Lemma \ref{lemma:transition-hoeffding}, we can show that
	\begin{align*}
	\dingtwo  & = \abs{ \sum_{s=1}^{k-1} \normweight_h^{k,s}(x, a) \pa{\int_{\statespace} f(y) \rmd \transition_h^s(y|x_h^s, a_h^s) - \int_{\statespace} f(y) \rmd \transition_h^s(y|x, a)}}
	\leq 4 \ksigma \lipP \lipQ \pa{ 1+ \sqrt{\logplus(\spacekernelconstA k/\kbeta)} }
	\end{align*}

	\paragraph{Bounding the martingale term (\dingone) with a Bernstein-type inequality}

	Notice that $(x, a) \mapsto \int_{\statespace} f(y) \rmd \trueP_h^k(y|x,a)$ is bounded by $2H$ and $$\expect{f(x_{h+1}^s) | \cF_h^s} = \int_{\statespace} f(y) \rmd \trueP_h^s(y|x_h^s,a_h^s).$$
	The conditional variance of $f(x_{h+1}^s)$ is bounded as follows
	\begin{align*}
	\variance{f(x_{h+1}^s) | \cF_h^s}
	& = \expect{f(x_{h+1}^s)^2 | \cF_h^s} - \pa{\int_{\statespace} f(y) \rmd \trueP_h^s(y|x_h^s,a_h^s)}^2 \\
	& \leq 2H  \expect{ \abs{f(x_{h+1}^s)} | \cF_h^s} \\
	& = 2H \int_{\statespace} \abs{f(y)} \rmd \trueP_h^s(y|x_h^s,a_h^s)
	\end{align*}
	which we use to bound its weighted average
	\begin{align*}
	& \frac{1}{\gencount_h^k(x, a)} \sum_{s=1}^{k-1} \weight_h^{k, s}(x, a)^2 \variance{f(x_{h+1}^s) | \cF_h^s} \\
	& \leq \frac{1}{\gencount_h^k(x, a)} \sum_{s=1}^{k-1} \weight_h^{k, s}(x, a) \variance{f(x_{h+1}^s) | \cF_h^s} \\
	& \leq \frac{2H}{\gencount_h^k(x, a)} \sum_{s=1}^{k-1} \weight_h^{k, s}(x, a) \int_{\statespace} \abs{f(y)} \rmd \trueP_h^s(y|x_h^s,a_h^s) \\
	& = \frac{2H}{\gencount_h^k(x, a)} \sum_{s=1}^{k-1} \weight_h^{k, s}(x, a) \trueP_h^s \abs{f}(x, a)
	+ \frac{2H}{\gencount_h^k(x, a)} \sum_{s=1}^{k-1} \weight_h^{k, s}(x, a)\pa{\trueP_h^s \abs{f}(x_h^s, a_h^s) - \trueP_h^s \abs{f}(x, a)}
	\\
	& \leq 2H \pa{  \avP_h^k \abs{f}(x, a) - \frac{\kbeta \trueP_h^k\abs{f}(x, a) }{\gencount_h^k(x, a)}  }
	+ \frac{4H\lipP\lipQ}{\gencount_h^k(x, a)} \sum_{s=1}^{k-1} \weight_h^{k, s}(x, a) \dist{(x_h^s, a_h^s), (x, a)}
	\\
	& \leq 2H \avP_h^k \abs{f}(x, a) + 8H\lipP\lipQ \ksigma  \pa{1 + \sqrt{\logplus (\spacekernelconstA k/\kbeta)}}
	\end{align*}
	where, in the last inequality, we used Lemma \ref{lemma:kernel-bias}.

	Let $\triangle(k,\cdelta) = \log\pa{4e(2k+1)/\cdelta}$. Let $Y_s(f) = f(x_{h+1}^s) -   \transition_h^s f(x_h^s, a_h^s)$. By Lemma~\ref{lemma:bernstein-freedman-weighted-sum}, we have
	\begin{align*}
	\dingone = \abs{ \sum_{s=1}^{k-1} \normweight_h^{k,s}(x, a) Y_s(f)}  \leq \sqrt{2 \triangle(k,\cdelta)  \frac{\sum_{s=1}^{k-1} \weight_h^{k, s}(x, a)^2 \variance{f(x_{h+1}^s) | \cF_h^s}}{\gencount_h^k(x, a)^2}} + \frac{10H \triangle(k,\cdelta) }{\gencount_h^k(x, a)}
	\end{align*}
	with probability at least $1-\cdelta$, since, for a fixed $f$, $(Y_s(f))_s$ is a martingale difference sequence with respect to $(\cF_h^s)_s$. Using the fact that  $\sqrt{uv}\leq (u+v)/2$ for all $u, v > 0$,
	\begin{align*}
	\abs{ \sum_{s=1}^{k-1} \normweight_h^{k,s}(x, a) Y_s(f)} & \leq
	\frac{4H^2\triangle(k,\cdelta) }{\gencount_h^k(x, a)}
	+ 	\frac{1}{4H^2} \frac{\sum_{s=1}^{k-1} \weight_h^{k, s}(x, a)^2 \variance{f(x_{h+1}^s) | \cF_h^s}}{\gencount_h^k(x, a)}
	+ \frac{10 H \triangle(k,\cdelta) }{\gencount_h^k(x, a)} \\
	&  \leq \frac{1}{H} \int_{\statespace} \abs{f(y)} \rmd \avP_h^k(y|x,a)
	+  \frac{(4H^2+10H)\triangle(k,\cdelta) }{\gencount_h^k(x, a)} +
	 \frac{2 \lipP\lipQ \ksigma}{H}   \pa{1 + \sqrt{\logplus (\spacekernelconstA t/\kbeta)}}
	\end{align*}

	From Corollary \ref{corollary:bias-between-avmdp-and-true-mdp}, we have
	\begin{align*}
	\int_{\statespace} \abs{f(y)} \rmd \avP_h^k(y|x,a) = (\avP_h^k-\trueP_h^k)\abs{f(y)}(x,a) + \trueP_h^k\abs{f(y)}(x,a) \leq 2\biasP(k, h) + \trueP_h^k\abs{f(y)}(x,a)
	\end{align*}
	which gives us
	\begin{align*}
	\abs{ \sum_{s=1}^{k-1} \normweight_h^{k,s}(x, a) Y_s(f)} & \leq \frac{1}{H} \trueP_h^k\abs{f(y)}(x,a)
	+  \frac{(4H^2+10H)\triangle(k,\cdelta) }{\gencount_h^k(x, a)} \\
	& + \frac{2}{H}\biasP(k, h) +
	\frac{2 \lipP\lipQ \ksigma}{H}   \pa{1 + \sqrt{\logplus (\spacekernelconstA t/\kbeta)}}
	\end{align*}
	with probability $1-\cdelta$.

	\paragraph{Covering of $\stateactionspace$} As a consequence of Assumption \ref{assumption:lipschitz-rewards-and-transitions}, the function $(x, a) \mapsto (1/H) \trueP_h^k\abs{f(y)}(x,a)$ is $2\lipP\lipQ$-\lipschitz. Also, the functions
	\begin{align*}
	(x, a) \mapsto \abs{ \sum_{s=1}^{k-1} \normweight_h^{k,s}(x, a) Y_s(f)}
	\quad\text{and}\quad
	(x, a) \mapsto  \frac{1}{\gencount_h^k(x, a)}
	\end{align*}
	are $4H\spacekernelconstB k/(\kbeta\ksigma)$-\lipschitz and $\spacekernelconstB k/(\kbeta^2\ksigma)$, respectively, by Lemma \ref{lemma:weighted-average-and-bonuses-are-lipschitz}.
	Consequently, a union bound over a $(\ksigma^{2+\Xcovdim}/(H^2K))$-covering of $(\stateactionspace, \distfunc)$ and over $[H]$ gives us
	\begin{align*}
	\abs{ \sum_{s=1}^{k-1} \normweight_h^{k,s}(x, a) Y_s(f)} & \leq \frac{1}{H} \trueP_h^k\abs{f(y)}(x,a)
	+  \frac{(4H^2+10H)\triangle(k,\cdelta) }{\gencount_h^k(x, a)} \\
	&  + \frac{2}{H}\biasP(k, h) +
	\frac{2 \lipP\lipQ \ksigma}{H}   \pa{1 + \sqrt{\logplus (\spacekernelconstA t/\kbeta)}}\\
	& + \pa{ 2\lipP\lipQ + \frac{4H\spacekernelconstB k}{\kbeta\ksigma} + \frac{(4H^2+10H)\triangle(k,\cdelta)  \spacekernelconstB k}{\kbeta^2\ksigma} }\frac{\ksigma^{2+\Xcovdim}}{H^2K}
	\end{align*}
	for all $(x, a, h, k)$ with probability at least $1 - \cdelta KH\XAcovnumber{\frac{\ksigma^{2+\Xcovdim}}{H^2K}}$.

	\paragraph{Covering of $\boundedlipschitzclass{2\lipQ}{2H}$} The bounds for \dingone and \dingtwo give us
	\begin{align*}
	\abs{(\estP_h^k - \avP_h^k)f(x,a)}  & \leq
	\frac{1}{H} \trueP_h^k\abs{f(y)}(x,a)
	+  \frac{(4H^2+10H)\triangle(k,\cdelta) }{\gencount_h^k(x, a)} \\
	&  + \frac{2}{H}\biasP(k, h) + \pa{ 2\lipP\lipQ + \frac{4H\spacekernelconstB k}{\kbeta\ksigma} + \frac{(4H^2+10H)\triangle(k,\cdelta)  \spacekernelconstB k}{\kbeta^2\ksigma} }\frac{\ksigma^{2+\Xcovdim}}{H^2K} \\
	&  + 6 \ksigma \lipP \lipQ \pa{ 1+ \sqrt{\logplus(\spacekernelconstA k/\kbeta)} }
	+ \frac{2 \kbeta H }{\gencount_h^k(x, a)}\,.
	\end{align*}
	The $8\lipQ\ksigma$-covering number of $\boundedlipschitzclass{2\lipQ}{2H}$ with respect to the infinity norm is bounded by $(2H/(\lipQ\ksigma))^{\Xcovnumber{\ksigma}}$, by Lemma 5 of \cite{domingues2020}.  The functions $f \mapsto \abs{(\transition_h^k - \estP_h^k)f(x,a)}$ and $f \mapsto \frac{1}{H} \int_{\statespace} \abs{f(y)} \rmd \avP_h^k(y|x,a)$ are $2$-Lipschitz with respect to $\norm{\cdot}_\infty$. Consequently, with probability at least
	$$1- \cdelta KH\XAcovnumber{\frac{\ksigma^{2+\Xcovdim}}{H^2K}} \pa{\frac{2H}{\lipQ\ksigma}}^{\Xcovnumber{\ksigma}}, $$
	for all $\boundedlipschitzclass{2\lipQ}{2H}$ and for all $(x, a, h,k)$, we have
	\begin{align*}
	\abs{(\estP_h^k - \avP_h^k)f(x,a)}  & \leq
	\frac{1}{H} \trueP_h^k\abs{f(y)}(x,a)
	+  \frac{(4H^2+10H)\triangle(k,\cdelta) }{\gencount_h^k(x, a)} \\
	&  + \frac{2}{H}\biasP(k, h) + \pa{ 2\lipP\lipQ + \frac{4H\spacekernelconstB k}{\kbeta\ksigma} + \frac{(4H^2+10H)\triangle(k,\cdelta)  \spacekernelconstB k}{\kbeta^2\ksigma} }\frac{\ksigma^{2+\Xcovdim}}{H^2K} \\
	&  + 6 \ksigma \lipP \lipQ \pa{ 1+ \sqrt{\logplus(\spacekernelconstA k/\kbeta)} }
	+ \frac{2 \kbeta H }{\gencount_h^k(x, a)} 
	 + 32\lipQ\ksigma
	\end{align*}
	which concludes the proof.
\end{proof}
\subsection{Good event}

\begin{flemma}
	\label{lemma:good-event}
	Let $\goodevent = \goodevent_1 \cap \goodevent_2 \cap \goodevent_3 \cap \goodevent_4$, where
	\begin{align*}
	& \goodevent_1 \eqdef \braces{\forall (x, a, k, h),\; \abs{\estR_h^k(x, a) - \avR_h^k(x, a)} \leq \sqrt{  \frac{2 \loghoeffdingR(k,\cdelta/8)}{\gencount_h^k(x, a)} } + \frac{\kbeta }{\gencount_h^k(x, a)} + \bonusbiasR(k, \cdelta/8)\ksigma} \\
	& \goodevent_2 \eqdef \braces{\forall (x, a, k, h),\; \abs{(\estP_h^k - \avP_h^k)\trueV_{k, h+1}^{*}(x, a)} \leq \sqrt{  \frac{2H^2 \loghoeffdingP(k,\cdelta/8)}{\gencount_h^k(x, a)} } + \frac{\kbeta H}{\gencount_h^k(x, a)} + \bonusbiasP(k, \cdelta/8)\ksigma } \\
	& \goodevent_3 \eqdef \Bigg\lbrace \forall (x, a, k, h, f),\;    \abs{(\estP_h^k - \avP_h^k)f(x,a)}   \leq
	\sqrt{  \frac{2H^2 \unifloghoeffdingP(k,\cdelta/8)}{\gencount_h^k(x, a)} } + \frac{\kbeta H}{\gencount_h^k(x, a)} \\
	& \quad\quad\quad\quad\quad\quad\quad\quad\quad\quad\quad\quad\quad
	+ \unifbiashoeffdingONE(k, \cdelta/8)\ksigma^{1+\Xcovdim/2} +  \unifbiashoeffdingTWO(k, \cdelta/8)\ksigma \Bigg\rbrace \\
	& \goodevent_4 \eqdef \Bigg\lbrace \forall (x, a, k, h, f),\;    \abs{(\estP_h^k - \avP_h^k)f(x,a)}   \leq
	\frac{1}{H} \trueP_h^k\abs{f(y)}(x,a)
	+  \frac{14 H^2\spacekernelconstB\logbernstein(k,\cdelta/8) + 2\kbeta H  }{\gencount_h^k(x, a)} \\
	& \quad\quad\quad\quad\quad\quad\quad\quad\quad\quad\quad\quad\quad
	+ \bernbiasone(k, \cdelta/8) \ksigma^{1+\Xcovdim}
	+ \bernbiastwo(k, \cdelta/8) \ksigma
	+ \frac{2}{H}\biasP(k, h)  \Bigg\rbrace
	\end{align*}
	for $(x, a, k, h) \in \stateactionspace\times[K]\times[H]$ and $f \in \boundedlipschitzclass{2\lipQ}{2H}$, and where
	\begin{align*}
	& \loghoeffdingP(k,\cdelta) = \BigOtilde{\XAcovdim},  \quad
	\bonusbiasP(k, \cdelta) = \BigOtilde{\lipQ + \sqrt{\XAcovdim}}, \quad
	\loghoeffdingR(k,\cdelta)  = \BigOtilde{\XAcovdim}, \quad
	\bonusbiasR(k, \cdelta) = \BigOtilde{\lipQ + \sqrt{\XAcovdim}} \\
	& \unifloghoeffdingP(k,\cdelta) = \BigOtilde{\Xsigmacov + \XAcovdim\Xcovdim}
	,\quad
	 \unifbiashoeffdingONE(k, \cdelta)  = \BigOtilde{\sqrt{\Xsigmacov} + \sqrt{\XAcovdim\Xcovdim}}
	,\quad
	\unifbiashoeffdingTWO(k, \cdelta) = \BigOtilde{\lipQ} \\
	& \logbernstein(k,\cdelta) = \BigOtilde{\Xsigmacov+ \XAcovdim\Xcovdim}, \quad
	\bernbiasone(k, \cdelta)  = \BigOtilde{\Xsigmacov + \XAcovdim\Xcovdim + \lipQ \ksigma}, \quad
	\bernbiastwo(k, \cdelta)= \BigOtilde{\lipQ}
	\end{align*}
	are defined in Lemmas \ref{lemma:transition-hoeffding}, \ref{lemma:reward-hoeffding}, \ref{lemma:transitions-uniform-hoeffding} and \ref{lemma:transitions-bernstein}, respectively. Then,
	\begin{align*}
	\prob{\goodevent} \geq 1 - \cdelta/2.
	\end{align*}
\end{flemma}

\begin{proof}
	Immediate consequence of Lemmas \ref{lemma:transition-hoeffding}, \ref{lemma:reward-hoeffding}, \ref{lemma:transitions-uniform-hoeffding} and \ref{lemma:transitions-bernstein}.
\end{proof}

\newpage
\section{Upper bound on true value function}

In this section, we show that the true value functions can be upper bounded by the value functions computed by \kerns plus a bias term. This result will be used to upper bound the regret in the next section.

\begin{flemma}[upper bound on $Q$ functions]
	\label{lemma:upper-bound-on-q-functions}
	On $\goodevent$, for all $(x, a, k, h) \in \stateactionspace \times [K] \times [H]$, we have
	\begin{align*}
		\algQ_h^k(x, a) + \sum_{h'=h}^H \bias(k, h)  \geq \trueQ_{k,h}^*(x, a)
	\end{align*}
	where $\bias(k, h) = \biasR(k, h) + \biasP(k, h)$ is the temporal bias of the algorithm at time $(k, h)$ (see Definition \ref{def:temporal-bias}).
\end{flemma}
\begin{proof}
	We proceed by induction on $h$. For $h = H+1$, both quantities are zero, so the inequality is trivially verified. Now, assume that it is true for $h+1$ and let's prove it for $h$.

	From the induction hypothesis, we have $\algV_{h+1}^k(x)  + \sum_{h'=h+1}^H \bias(k, h) \geq \trueV_{k, h+1}^*(x)$. Indeed,
	\begin{align*}
		\max_a \algQ_{h+1}^k(x, a) + \sum_{h'=h+1}^H \bias(k, h)  \geq \max_a \trueQ_{k,h}^*(x, a) = \trueV_{k,h+1}^*(x)
	\end{align*}
	and, since $ \trueV_{k,h+1}^*(x)  \leq H-h$, we have
	\begin{align*}
		\algV_{h+1}^k(x) + \sum_{h'=h+1}^H \bias(k, h) = \min\pa{H-h, \max_a \algQ_{h+1}^k(x, a)} + \sum_{h'=h+1}^H \bias(k, h) \geq \trueV_{k,h+1}^*(x).
	\end{align*}

	From the definition of the algorithm, we have
	\begin{align*}
		\algQ_h^k(x, a) = \min_{s\in[k-1]} \sqrbrackets{ \tQ_h^k(x_h^s, a_h^s) + \lipQ \dist{(x, a), (x_h^s, a_h^s)}  }
	\end{align*}
	where $\tQ_h^k(x, a) = \estR_h^k(x, a) + \estP_h^k\algV_{h+1}^k(x, a) + \bonus_h^k(x, a)$. Hence,
	\begin{align*}
		& \tQ_h^k(x, a) - \trueQ_{k,h}^*(x, a)   \\
		& = \underbrace{
			    \estR_h^k(x, a) - \trueR_h^k(x, a) + \rbonus_h^k(x, a)
		     }_{\termA}
		    + \underbrace{
		    	\estP_h^k \algV_{h+1}^k(x, a) - \trueP_h^k \trueV_{k, h+1}^*(x, a) + \pbonus_h^k(x, a)
	    	 }_{\termB}.
	\end{align*}
	The term $\termA$ is lower bounded as follows
	\begin{align*}
		\termA & = \estR_h^k(x, a) - \avR_h^k(x, a) + \rbonus_h^k(x, a) + \avR_h^k(x, a) - \trueR_h^k(x, a) \geq -\biasR(k, h)
	\end{align*}
	by Corollary \ref{corollary:bias-between-avmdp-and-true-mdp} and the fact that $\estR_h^k(x, a) - \avR_h^k(x, a) + \rbonus_h^k(x, a) \geq 0$ on $\goodevent$.

	Similarly, for the term $\termB$, we have
	\begin{align*}
		\termB & =
		\estP_h^k \pa{\algV_{h+1}^k - \trueV_{k,h+1}^*}(x, a)
		+ \pa{\estP_h^k - \avP_h^k} \trueV_{k,h+1}^*(x, a)
		+ \pa{\avP_h^k - \trueP_h^k} \trueV_{k,h+1}^*(x, a) + \pbonus_h^k(x, a) \\
		& \geq \estP_h^k \pa{\algV_{h+1}^k - \trueV_{k,h+1}^*}(x, a) - \biasP(k, h)
	\end{align*}
	which gives us
	\begin{align*}
		& \tQ_h^k(x, a) - \trueQ_{k,h}^*(x, a) \\
		& \geq \estP_h^k \pa{\algV_{h+1}^k - \trueV_{k,h+1}^*}(x, a) -\pa{ \biasR(k, h) + \biasP(k, h)} \\
		& = \estP_h^k \pa{\algV_{h+1}^k - \trueV_{k,h+1}^*}(x, a) + \sum_{h'=h+1}^H \bias(k, h) - \sum_{h'=h}^H \bias(k, h) \\
		& \geq \estP_h^k \pa{\algV_{h+1}^k + \sum_{h'=h+1}^H \bias(k, h) - \trueV_{k,h+1}^*}(x, a)  - \sum_{h'=h}^H \bias(k, h) \\
		& \geq  - \sum_{h'=h}^H \bias(k, h) \quad \text{by the induction hypothesis.}
	\end{align*}
	Consequently, for all $(x, a)$ and all $s \in [k-1]$, we have
	\begin{align*}
		\trueQ_{k,h}^*(x, a) - \sum_{h'=h}^H \bias(k, h)
		& \leq \trueQ_{k,h}^*(x_h^s, a_h^s) + \lipQ \dist{(x, a), (x_h^s, a_h^s)} - \sum_{h'=h}^H \bias(k, h)   \\
		& \leq \tQ_h^k(x_h^s, a_h^s) + \lipQ \dist{(x, a), (x_h^s, a_h^s)}
	\end{align*}
	since $\trueQ_{k,h}^*$ is $\lipQ$-\lipschitz. Which implies the result
	\begin{align*}
		\trueQ_{k,h}^*(x, a) - \sum_{h'=h}^H \bias(k, h)  \leq \min_{s\in[k-1]}\sqrbrackets{\tQ_h^k(x_h^s, a_h^s) + \lipQ \dist{(x, a), (x_h^s, a_h^s)}} = \algQ_h^k(x, a).
	\end{align*}
\end{proof}

\begin{fcorollary}
	\label{corollary:error-wrt-true-upper-bound}
	Let $\upperQ_{k,h}$ and $\upperV_{k,h}$ be defined as as
	\begin{align*}
	& \upperQ_{k,h}(x, a) \eqdef \algQ_h^k(x, a) + \sum_{h'=h}^H \bias(k, h')
	,\quad
	\upperV_{k,h}(x) \eqdef  \min\pa{H, \max_a \upperQ_{k,h}(x, a)}
	\end{align*}
	Then,
	\begin{align*}
	\sup_{x\in\statespace} \abs{\algV_h^k(x) - \upperV_{k,h}(x)} \leq  \sum_{h'=h}^H \bias(k, h')
	\end{align*}
	and, by Lemma \ref{lemma:upper-bound-on-q-functions}, we have $\upperV_{k,h} \geq \trueV_{k,h}^*$ on the event $\goodevent$.
\end{fcorollary}
\begin{proof}
	For any $x \in \statespace$,
	\begin{align*}
	\abs{\algV_h^k(x) - \upperV_{k,h}(x)}
	& = \abs{\min\pa{H, \max_a \algQ_h^k(x, a)} -  \min\pa{H, \max_a \upperQ_{k,h}(x, a)}}  \\
	& \leq \abs{\max_a \algQ_h^k(x, a) - \max_a \upperQ_{k,h}(x, a)} \\
	& \leq \max_a \abs{\algQ_h^k(x, a) - \upperQ_{k,h}(x, a)} \\
	& =  \sum_{h'=h}^H \bias(k, h').
	\end{align*}
	where we used the fact that, for any $a, b, c \in \RR$, we have $\abs{\min(a, b) -\min(a, c)} \leq \abs{b-c}$.
\end{proof}

\section{Regret bounds}
\label{app:regret_bounds}

Using the results proved in the previous sections, we are now ready to prove our regret bounds. We first start by proving that the regret is bounded by sums involving
$\sqrt{1/\gencount_h^k(x,a)}$, $1/\gencount_h^k(x,a)$ and bias terms. Then, we provide upper bounds for these sums, which result in the final regret bounds.

In Theorem \ref{theorem:regret_main_text}, we prove two regret bounds, $\regret_1$ and $\regret_2$. Here, we refer to $\regret_1$ as a UCRL-type regret bound and to $\regret_2$ as a UCBVI-type bound, due to the technique used to bound the difference between $\estP_h^k$ and $\trueP_h^k$. Making an analogy with finite MDPs, in UCRL \citep{jaksch2010near}, a term analogous to $\normm{\estP_h-\trueP_h^k}_1$ is bounded (as in Lemma \ref{lemma:transitions-uniform-hoeffding}), whereas in UCBVI \citep{Azar2017}, the term $|(\estP_h-\trueP_h^k)\trueV_{k, h+1}^*|$ is bounded (as in Lemma \ref{lemma:transition-hoeffding}).

	\begin{fcorollary}
		\label{corollary:regret-is-bounded-by-delta}
		Let $\delta_h^k \eqdef \algV_h^k(x_h^k) - \trueV_{k,h}^{\pi_k}(x_h^k)$. Then, on $\goodevent$
		\begin{align*}
			\regret(K) \leq \sum_{k=1}^K\delta_1 ^k + \sum_{k=1}^K\sum_{h=1}^H \bias(k, h).
		\end{align*}
	\end{fcorollary}
	\begin{proof}
		It follows directly from Lemma \ref{lemma:upper-bound-on-q-functions}:
		\begin{align*}
			\regret(K) = \sum_{k=1}^{K} \trueV_{k,1}^*(x_1^k) - \trueV_{k,h}^{\pi_k}(x_1^k) \leq \sum_{k=1}^{K}\pa{ \algV_{1}^k(x_1^k)  + \sum_{h=1}^H \bias(k, h)  - \trueV_{k,h}^{\pi_k}(x_1^k)}.
		\end{align*}
	\end{proof}

	\begin{definition}
		For any $(k, h)$, let  $(\tx_h^k, \ta_h^k)$ be defined as
		\begin{align*}
			(\tx_h^k, \ta_h^k)  \eqdef \argmin_{(x_h^s, a_h^s): s<k} \dist{(x_h^k, a_h^k), (x_h^s, a_h^s)}
		\end{align*}
		that is, the state-action pair in the history that is the closest to $(x_h^k, a_h^k)$.
	\end{definition}

	\subsection{Regret bound in terms of the sum of exploration bonuses (UCRL-type)}

	\begin{flemma}[UCRL-type bound with sum of bonuses]
		\label{lemma:regret-in-terms-of-sum-of-bonus-ucrl-type}
		On the event $\goodevent$, the regret of \ouralgo is bounded by
		\begin{align*}
		\regret(K) \lesssim &
		\sum_{k=1}^K\sum_{h=1}^H \pa{\frac{H\sqrt{\Xsigmacov}}{\sqrt{\gencount_h^k(\tx_h^k, \ta_h^k)}}
			+  \frac{\kbeta H}{\gencount_h^k(\tx_h^k, \ta_h^k)} }\indic{\dist{(x_h^k, a_h^k), (\tx_h^k, \ta_h^k)} \leq  2 \ksigma}
		+ H^2 \sigmacov \\
		& +  \sum_{k=1}^K\sum_{h=1}^H \widetilde{\xi}_{h+1}^k
		+  H \sum_{k=1}^K\sum_{h=1}^H \bias(k, h)
		+ \lipQ K H\ksigma
		\end{align*}
		where $\sigmacov$ is the $\ksigma$-covering number of $(\stateactionspace, \distfunc)$, $\Xsigmacov$ is the $\ksigma$-covering number of $(\statespace, \Sdistfunc)$ and $(\widetilde{\xi}_{h+1}^k)_{k, h}$ is a martingale difference sequence with respect to $(\cF_h^k)_{k, h}$ bounded by $4H$.
	\end{flemma}
	\begin{proof}
		\textbf{Regret decomposition} \ \ On $\goodevent$, we upper bound $\delta_h^k$ using the following decomposition:
		\begin{align*}
		\delta_h^k
		& =  \algV_h^k(x_h^k) - \trueV_{k,h}^{\pi_k}(x_h^k) \\
		& \leq \algQ_h^k(x_h^k, a_h^k) - \trueQ_{k,h}^{\pi_k}(x_h^k, a_h^k) \\
		& \leq \algQ_h^k(\tx_h^k, \ta_h^k) - \trueQ_{k,h}^{\pi_k}(x_h^k, a_h^k)
		+ \lipQ\dist{(\tx_h^k, \ta_h^k), (x_h^k, a_h^k)}
		,\quad\text{since $\algQ_h^k$ is $\lipQ$-\lipschitz} \\
		& \leq  \tQ_h^k(\tx_h^k, \ta_h^k) - \trueQ_{k,h}^{\pi_k}(x_h^k, a_h^k)
		+ \lipQ\dist{(\tx_h^k, \ta_h^k), (x_h^k, a_h^k)}
		,\quad\text{since $\algQ_h^k(\tx_h^k, \ta_h^k) \leq  \tQ_h^k(\tx_h^k, \ta_h^k)$} \\
		& = \estR_h^k(\tx_h^k, \ta_h^k) - \trueR_h^k(x_h^k, a_h^k)
		+  \estP_h^k\algV_{h+1}^k(\tx_h^k, \ta_h^k) - \trueP_h^k\trueV_{k,h+1}^{\pi_k}(x_h^k, a_h^k)
		+\bonus_h^k(\tx_h^k, \ta_h^k)
		+ \lipQ\dist{(\tx_h^k, \ta_h^k), (x_h^k, a_h^k)} \\
		& =
		    \underbrace{
		    	\estR_h^k(\tx_h^k, \ta_h^k) - \trueR_h^k(x_h^k, a_h^k)
		    }_{\termA}
		    +
		     \underbrace{
		     	\sqrbrackets{\estP_h^k -\trueP_h^k}\trueV_{k,h+1}^* (\tx_h^k, \ta_h^k)
		     }_{\termB}
		  +  \underbrace{
		 \sqrbrackets{\estP_h^k -\trueP_h^k}\pa{ \algV_{h+1}^k- \trueV_{k,h+1}^*}(\tx_h^k, \ta_h^k)
		       }_{\termC}  \\
		&  +
			 \underbrace{
			 	 \trueP_h^k \algV_{h+1}^k(\tx_h^k, \ta_h^k) - \trueP_h^k\trueV_{k,h+1}^{\pi_k}(x_h^k, a_h^k)
			 }_{\termD}
		   +\bonus_h^k(\tx_h^k, \ta_h^k)
	    	+ 2\lipQ\dist{(\tx_h^k, \ta_h^k), (x_h^k, a_h^k)}.
		\end{align*}
		Now, we bound each term $\termA$-$\termD$ separately.

	\underline{Term $\termA$}:
	\begin{align*}
	\termA & = \estR_h^k(\tx_h^k, \ta_h^k) - \trueR_h^k(\tx_h^k, \ta_h^k) + \trueR_h^k(\tx_h^k, \ta_h^k) - \trueR_h^k(x_h^k, a_h^k) \\
	& \leq \estR_h^k(\tx_h^k, \ta_h^k) - \trueR_h^k(\tx_h^k, \ta_h^k) +\lipR \dist{(\tx_h^k, \ta_h^k), (x_h^k, a_h^k)} \\
	& = \estR_h^k(\tx_h^k, \ta_h^k) - \avR_h^k(\tx_h^k, \ta_h^k) + \avR_h^k(\tx_h^k, \ta_h^k) -  \trueR_h^k(\tx_h^k, \ta_h^k) + \lipR \dist{(\tx_h^k, \ta_h^k), (x_h^k, a_h^k)}  \\
	& \leq \rbonus_h^k(\tx_h^k, \ta_h^k) + \biasR(k, h) +\lipR \dist{(\tx_h^k, \ta_h^k), (x_h^k, a_h^k)}
	\end{align*}
	by the definition of $\goodevent$ and Corollary \ref{corollary:bias-between-avmdp-and-true-mdp}.

	\underline{Term $\termB$}:
	\begin{align*}
	\termB & = \sqrbrackets{\estP_h^k -\trueP_h^k}\trueV_{k,h+1}^* (\tx_h^k, \ta_h^k)
	= \sqrbrackets{\estP_h^k -\avP_h^k}\trueV_{k,h+1}^* (\tx_h^k, \ta_h^k) +
	\sqrbrackets{\avP_h^k -\trueP_h^k}\trueV_{k,h+1}^* (\tx_h^k, \ta_h^k) \\
	& \leq \pbonus_h^k(\tx_h^k, \ta_h^k) + \biasP(k, h)\,.
	\end{align*}

	\underline{Term $\termC$}: Using Corollary \ref{corollary:bias-between-avmdp-and-true-mdp}, we obtain
	\begin{align*}
		\termC &= \sqrbrackets{\estP_h^k -\trueP_h^k}\pa{ \algV_{h+1}^k- \trueV_{k,h+1}^*}(\tx_h^k, \ta_h^k) \\
		& \leq \sqrbrackets{\estP_h^k -\avP_h^k}\pa{ \algV_{h+1}^k- \trueV_{k,h+1}^*}(\tx_h^k, \ta_h^k) + 2\biasP(k, h) \\
		& \leq
		 \sqrt{  \frac{8H^2 \unifloghoeffdingP(k,\cdelta/8)}{\gencount_h^k(\tx_h^k, \ta_h^k)} }
		 + \frac{2 \kbeta H}{\gencount_h^k(\tx_h^k, \ta_h^k)}
		 + \unifbiashoeffdingONE(k, \cdelta/8)\ksigma^{1+\Xcovdim/2}
		 +  \unifbiashoeffdingTWO(k, \cdelta/8)\ksigma
		 + 2\biasP(k, h) \\
		& \lesssim \sqrt{  \frac{ H^2 \Xsigmacov }{\gencount_h^k(\tx_h^k, \ta_h^k)} }
		 + \frac{\kbeta H}{\gencount_h^k(\tx_h^k, \ta_h^k)}  + \lipQ \ksigma + 2\biasP(k, h)
	\end{align*}
	by the definition of $\goodevent$.

	\underline{Term $\termD$}: From Assumption \ref{assumption:lipschitz-rewards-and-transitions}, for any $\lipQ$-\lipschitz function, the mapping $(x, a) \mapsto \trueP_h^k f(x, a)$ is $\lipP\lipQ$-\lipschitz. Consequently,
	\begin{align*}
	\termD & =  \trueP_h^k \algV_{h+1}^k(\tx_h^k, \ta_h^k) - \trueP_h^k\trueV_{k,h+1}^{\pi_k}(x_h^k, a_h^k)\\
	& \leq \trueP_h^k \algV_{h+1}^k(x_h^k, a_h^k) - \trueP_h^k\trueV_{k,h+1}^{\pi_k}(x_h^k, a_h^k) + \lipP\lipQ\dist{(x_h^k, a_h^k), (\tx_h^k, \ta_h^k)}\\
	& = \trueP_h^k\pa{ \algV_{h+1}^k - \trueV_{k,h+1}^{\pi_k}}(x_h^k, a_h^k) + \lipP\lipQ\dist{(x_h^k, a_h^k), (\tx_h^k, \ta_h^k)} \\
	& = \delta_{h+1}^k + \xi_{h+1}^k + \lipP\lipQ\dist{(x_h^k, a_h^k), (\tx_h^k, \ta_h^k)}
	\end{align*}
	where
	\begin{align*}
	\xi_{h+1}^k \eqdef  \trueP_h^k\pa{ \algV_{h+1}^k - \trueV_{k,h+1}^{\pi_k}}(x_h^k, a_h^k) - \delta_{h+1}^k
	\end{align*}
	is a martingale difference sequence with respect to $(\cF_h^k)_{k, h}$ bounded by $4H$.

	Putting together the bounds for $\termA$-$\termD$ and using the definition of the bonuses $\bonus_h^k$, we obtain
	\begin{align*}
	& \delta_h^k \lesssim
	\delta_{h+1}^k + \xi_{h+1}^k
	+ \lipQ\dist{(x_h^k, a_h^k), (\tx_h^k, \ta_h^k)}
	+ \sqrt{  \frac{ H^2 \Xsigmacov }{\gencount_h^k(\tx_h^k, \ta_h^k)} }
	+ \frac{\kbeta H}{\gencount_h^k(\tx_h^k, \ta_h^k)}
	+ \lipQ\ksigma
	+ \bias(k, h)
	\end{align*}
	where the constant in front of $\delta_{h+1}^k$ is \emph{exact} (\ie not omitted by $\lesssim$).

	Let $E_h^k \eqdef \braces{\dist{(x_h^k, a_h^k), (\tx_h^k, \ta_h^k)} \leq  2 \ksigma}$. The inequality above implies
	\begin{align}
	\label{eq:aux-indicator-times-delta}
	 \indic{E_h^k}\delta_h^k
	\lesssim
	\indic{E_h^k} \delta_{h+1}^k
	+
	\indic{E_h^k}\pa{
		\xi_{h+1}^k
		+ \sqrt{  \frac{ H^2 \Xsigmacov }{\gencount_h^k(\tx_h^k, \ta_h^k)} }
		+ \frac{\kbeta H}{\gencount_h^k(\tx_h^k, \ta_h^k)}
	}
	+ 3 \lipQ\ksigma + \bias(k, h).
	\end{align}

	Now, we bound $	\indic{E_h^k} \delta_{h+1}^k$ in terms of $\delta_{h+1}^k$, which will be later used to bound $\delta_{h}^k$ in terms of $\delta_{h+1}^k$.  On $\goodevent$, we have
	\begin{align*}
		\indic{E_h^k} \delta_{h+1}^k
		& =  \indic{E_h^k} \pa{  \algV_{h+1}^k(x_{h+1}^k) - \trueV_{k,h+1}^{\pi_k}(x_{h+1}^k) } \\
		& = \indic{E_h^k} \pa{
				\underbrace{
			 \algV_{h+1}^k(x_{h+1}^k) +  \sum_{h'=h+1}^H \bias(k, h') - \trueV_{k,h+1}^{\pi_k}(x_{h+1}^k)
			}_{ \geq 0 \text{ by Lemma \ref{lemma:upper-bound-on-q-functions}}}
			 - \sum_{h'=h+1}^H \bias(k, h') } \\
		& \leq \algV_{h+1}^k(x_{h+1}^k) +  \sum_{h'=h+1}^H \bias(k, h') - \trueV_{k,h+1}^{\pi_k}(x_{h+1}^k) - \indic{E_h^k}\sum_{h'=h+1}^H \bias(k, h') \\
		& = \delta_{h+1}^k + \sum_{h'=h+1}^H \bias(k, h') - \indic{E_h^k}\sum_{h'=h+1}^H \bias(k, h') \\
		& \leq \delta_{h+1}^k + \sum_{h'=h+1}^H \bias(k, h')\,.
	\end{align*}

	The inequality above, combined with \eqref{eq:aux-indicator-times-delta} yields
	\begin{align*}
		\indic{E_h^k}\delta_h^k
		\lesssim
		 \delta_{h+1}^k
		+ \sum_{h'=h}^H \bias(k, h')
		+ \indic{E_h^k}\pa{
			\xi_{h+1}^k
			+ \sqrt{  \frac{ H^2 \Xsigmacov }{\gencount_h^k(\tx_h^k, \ta_h^k)} }
			+ \frac{\kbeta H}{\gencount_h^k(\tx_h^k, \ta_h^k)}
		}
		+ 3 \lipQ\ksigma .
	\end{align*}

	Let $\overline{E}_h^k$ be the complement of $E_h^k$. Since $\delta_h^k \leq H$, we have
	\begin{align*}
		\delta_h^k
		& = \indic{\overline{E}_h^k}\delta_h^k  + \indic{E_h^k}\delta_h^k  \\
		& \leq H  \indic{\overline{E}_h^k}  + \indic{E_h^k}\delta_h^k  \\
		& \lesssim H  \indic{\overline{E}_h^k} +  \delta_{h+1}^k
		+ \sum_{h'=h}^H \bias(k, h')
		+ \indic{E_h^k}\pa{
			\xi_{h+1}^k
			+ \sqrt{  \frac{ H^2 \Xsigmacov }{\gencount_h^k(\tx_h^k, \ta_h^k)} }
			+ \frac{\kbeta H}{\gencount_h^k(\tx_h^k, \ta_h^k)}
		}
		+  \lipQ\ksigma\,.
	\end{align*}

	 This yields
	\begin{align*}
		\delta_1^k \lesssim
		& \sum_{h=1}^H  \indic{E_h^k}\sqrt{  \frac{ H^2 \Xsigmacov }{\gencount_h^k(\tx_h^k, \ta_h^k)} }
		+ \sum_{h=1}^H \indic{E_h^k}\frac{\kbeta H}{\gencount_h^k(\tx_h^k, \ta_h^k)} \\
		& + \sum_{h=1}^H \indic{E_h^k} \xi_{h+1}^k
		+ H \sum_{h=1}^H \bias(k, h)
		+ H  \sum_{h=1}^H \indic{\overline{E}_h^k}
		+ H\lipQ \ksigma
	\end{align*}

	Using Corollary \ref{corollary:regret-is-bounded-by-delta}, we obtain
	\begin{align*}
	 \regret(K)
	& \leq \sum_{k=1}^K\delta_1 ^k + \sum_{k=1}^K\sum_{h=1}^H \bias(k, h) \\
	& \lesssim \sum_{k=1}^K
	   \sum_{h=1}^H  \indic{E_h^k}\sqrt{  \frac{ H^2 \Xsigmacov }{\gencount_h^k(\tx_h^k, \ta_h^k)} }
	  + \sum_{k=1}^K \sum_{h=1}^H \indic{E_h^k}\frac{\kbeta H}{\gencount_h^k(\tx_h^k, \ta_h^k)} \\
	& + \sum_{k=1}^K \sum_{h=1}^H \indic{E_h^k} \xi_{h+1}^k
      + H \sum_{k=1}^K \sum_{h=1}^H \bias(k, h)
	  + H  \sum_{k=1}^K\sum_{h=1}^H \indic{\overline{E}_h^k}
	  + K H\lipQ \ksigma \,.
	\end{align*}

	For each $h$, the number of episodes $k$ where the event $\braces{\dist{(x_h^k, a_h^k), (\tx_h^k, \ta_h^k)} >  2 \ksigma}$ occurs is bounded by $\sigmacov$. Hence, we can bound the sum
	\begin{align*}
		H \sum_{k=1}^K\sum_{h=1}^H \indic{\overline{E}_h^k}
		= H \sum_{h=1}^H \sum_{k=1}^K \indic{\dist{(x_h^k, a_h^k), (\tx_h^k, \ta_h^k)} >  2 \ksigma}
		\leq H^2\sigmacov.
	\end{align*}

	We conclude the proof by recalling the definition $E_h^k \eqdef \braces{\dist{(x_h^k, a_h^k), (\tx_h^k, \ta_h^k)} \leq  2 \ksigma}$ and using the fact that $\widetilde{\xi}_{h+1}^k\eqdef \indic{E_h^k} \xi_{h+1}^k$ is a martingale difference sequence with respect to $(\cF_h^k)_{k, h}$ bounded by $4H$.
	\end{proof}
	\subsection{Regret bound in terms of the sum of exploration bonuses (UCBVI-type)}
	\begin{flemma}[UCBVI-type bound with sum of bonuses]
		\label{lemma:regret-in-terms-of-sum-of-bonus}
		In the event $\goodevent$, the regret of \ouralgo is bounded by
		\begin{align*}
			\regret(K) \lesssim &
			\sum_{k=1}^K\sum_{h=1}^H \pa{\frac{H}{\sqrt{\gencount_h^k(\tx_h^k, \ta_h^k)}}
			+  \frac{H^2\Xsigmacov}{\gencount_h^k(\tx_h^k, \ta_h^k)} }\indic{\dist{(x_h^k, a_h^k), (\tx_h^k, \ta_h^k)} \leq  2 \ksigma}
			+ H^2 \sigmacov \\
			& +  \sum_{k=1}^K\sum_{h=1}^H \pa{1+\frac{1}{H}}^{h}\widetilde{\xi}_{h+1}^k
		     +  H \sum_{k=1}^K\sum_{h=1}^H \bias(k, h)
			 + \lipQ K H\ksigma
		\end{align*}
		where $\sigmacov$ is the $\ksigma$-covering number of $(\stateactionspace, \distfunc)$, $\Xsigmacov$ is the $\ksigma$-covering number of $(\statespace, \Sdistfunc)$ and $(\widetilde{\xi}_{h+1}^k)_{k, h}$ is a martingale difference sequence with respect to $(\cF_h^k)_{k, h}$ bounded by $4H$.
	\end{flemma}
	\begin{proof}
		The proof follows the one of Proposition 5 of \cite{domingues2020}. The key difference is that we need to handle the temporal bias. In particular, $\algV_h^k$ is not an upper bound on $\trueV_{k,h+1}^*$ due to the temporal bias, which makes our proof slightly more technical by introducing $\upperV_{k,h}$ (see Cor. \ref{corollary:error-wrt-true-upper-bound}) when applying the Bernstein-type concentration of Lemma \ref{lemma:transitions-bernstein}.

		\paragraph{Regret decomposition} We use the same regret decomposition as in the proof of Lemma \ref{lemma:regret-in-terms-of-sum-of-bonus-ucrl-type}. The terms $\termA, \termB$ and $\termD$ are bounded in the same way, but we handle the term $\termC$ differently.

		\underline{Term $\termC$}: To bound this term, we use corollaries \ref{corollary:bias-between-avmdp-and-true-mdp} and \ref{corollary:error-wrt-true-upper-bound}:
		\begin{align*}
			\termC
			& = \sqrbrackets{\estP_h^k -\trueP_h^k}\pa{ \algV_{h+1}^k- \trueV_{k,h+1}^*}(\tx_h^k, \ta_h^k) \\
			& = \sqrbrackets{\estP_h^k -\trueP_h^k}\pa{ \upperV_{k,h+1}- \trueV_{k,h+1}^*}(\tx_h^k, \ta_h^k)
			    + \sqrbrackets{\estP_h^k -\trueP_h^k}\pa{ \algV_{h+1}^k- \upperV_{k,h+1}}(\tx_h^k, \ta_h^k) \\
			& \leq \sqrbrackets{\estP_h^k -\trueP_h^k}\pa{ \upperV_{k,h+1}- \trueV_{k,h+1}^*}(\tx_h^k, \ta_h^k)
			        + 2 \sum_{h'=h}^H \bias(k, h')
			   , \quad \text{by Cor. \ref{corollary:error-wrt-true-upper-bound}}\\
			& =  \sqrbrackets{\estP_h^k -\avP_h^k}\pa{ \upperV_{k,h+1}- \trueV_{k,h+1}^*}(\tx_h^k, \ta_h^k)
			     + \sqrbrackets{\avP_h^k -\trueP_h^k}\pa{ \upperV_{k,h+1}- \trueV_{k,h+1}^*}(\tx_h^k, \ta_h^k)
			     + 2 \sum_{h'=h}^H \bias(k, h')	 \\
			& \leq \sqrbrackets{\estP_h^k -\avP_h^k}\pa{ \upperV_{k,h+1}- \trueV_{k,h+1}^*}(\tx_h^k, \ta_h^k)
			     + 2\biasP(k, h)
			     + 2 \sum_{h'=h}^H \bias(k, h')
			     , \quad \text{by Cor. \ref{corollary:bias-between-avmdp-and-true-mdp}}\\
			& \leq  \frac{1}{H} \trueP_h^k\pa{ \upperV_{k,h+1}- \trueV_{k,h+1}^*}(\tx_h^k, \ta_h^k)  +  \frac{14 H^2\spacekernelconstB\logbernstein(k,\cdelta/8) + 2\kbeta H  }{\gencount_h^k(\tx_h^k, \ta_h^k)} \\
			   & \quad + \bernbiasone(k, \cdelta/8) \ksigma^{1+\Xcovdim}
			           + \bernbiastwo(k, \cdelta/8) \ksigma
			           + \frac{2}{H}\biasP(k, h)
			           + 2\biasP(k, h)
			           + 2 \sum_{h'=h}^H \bias(k, h')	 \\
			& \leq  \frac{1}{H} \trueP_h^k\pa{ \upperV_{k,h+1}- \trueV_{k,h+1}^{*}}(x_h^k, a_h^k)  +  \frac{14 H^2\spacekernelconstB\logbernstein(k,\cdelta/8) + 2\kbeta H  }{\gencount_h^k(\tx_h^k, \ta_h^k)} +  \frac{2 \lipP\lipQ}{H}\dist{(x_h^k, a_h^k), (\tx_h^k, \ta_h^k)} \\
			& \quad + \bernbiasone(k, \cdelta/8) \ksigma^{1+\Xcovdim}
			+ \bernbiastwo(k, \cdelta/8) \ksigma
			+ \frac{2}{H}\biasP(k, h)
			+ 2\biasP(k, h)
			+ 2 \sum_{h'=h}^H \bias(k, h'),
		\end{align*}
		where we also used the definition of $\goodevent$ and the fact that the function $(x, a)\mapsto \trueP_h^k\pa{ \upperV_{k,h+1}- \trueV_{k,h+1}^*}(x, a)$ is $2\lipP\lipQ$-\lipschitz, from Assumption \ref{assumption:lipschitz-rewards-and-transitions}. Now, since
		\begin{align*}
			\logbernstein(k,\cdelta) = \BigOtilde{\Xsigmacov + \XAcovdim\Xcovdim}
			, \quad
			\bernbiasone(k, \cdelta)  = \BigOtilde{\Xsigmacov + \XAcovdim\Xcovdim + \lipQ \ksigma}
			, \quad
			\bernbiastwo(k, \cdelta)= \BigOtilde{\lipQ}
		\end{align*}
		and $\Xsigmacov = \BigO{1/\ksigma^\Xcovdim}$, we have
		\begin{align*}
			\termC \lesssim &  \frac{1}{H} \trueP_h^k\pa{ \upperV_{k,h+1}- \trueV_{k,h+1}^*}(x_h^k, a_h^k) + \frac{H^2\Xsigmacov}{\gencount_h^k(\tx_h^k, \ta_h^k)} + \lipQ\ksigma \\
			&  +  \pa{2 + \frac{2}{H}}\biasP(k, h)
			+ 2 \sum_{h'=h}^H \bias(k, h') + \frac{2 \lipP\lipQ}{H}\dist{(x_h^k, a_h^k), (\tx_h^k, \ta_h^k)}\,.
		\end{align*}
		Using again Corollary \ref{corollary:error-wrt-true-upper-bound}, we have
		\begin{align*}
			\frac{1}{H} \trueP_h^k\pa{ \upperV_{k,h+1}- \trueV_{k,h+1}^*}(x_h^k, a_h^k) \leq \frac{1}{H} \trueP_h^k\pa{ \algV_{h+1}^k- \trueV_{k,h+1}^*}(x_h^k, a_h^k) + \frac{1}{H}\sum_{h'=h}^H \bias(k, h')
		\end{align*}
		which gives us, since $\trueV_{k,h+1}^{\pi_k} \leq \trueV_{k,h+1}^*$,
		\begin{align*}
			\termC & \lesssim \frac{1}{H} \trueP_h^k\pa{ \algV_{h+1}^k- \trueV_{k,h+1}^*}(x_h^k, a_h^k) + \frac{H^2\Xsigmacov}{\gencount_h^k(\tx_h^k, \ta_h^k)} + \lipQ\ksigma +  \sum_{h'=h}^H \bias(k, h') + \lipQ\dist{(x_h^k, a_h^k), (\tx_h^k, \ta_h^k)}\\
			& \lesssim \frac{1}{H} \trueP_h^k\pa{ \algV_{h+1}^k- \trueV_{k,h+1}^{\pi_k}}(x_h^k, a_h^k) + \frac{H^2\Xsigmacov}{\gencount_h^k(\tx_h^k, \ta_h^k)} + \lipQ\ksigma
			+ \sum_{h'=h}^H \bias(k, h') + \lipQ\dist{(x_h^k, a_h^k), (\tx_h^k, \ta_h^k)}
		\end{align*}
		where we omit constants. Notice, however, that there are no constants omitted in the term $ \frac{1}{H} \trueP_h^k\pa{ \algV_{h+1}^k- \trueV_{k,h+1}^{\pi_k}}(x_h^k, a_h^k)$.

		Putting together the bounds for $\termA$-$\termD$, we obtain
		\begin{align*}
			& \delta_h^k \lesssim
			    \pa{1+\frac{1}{H}}\pa{\delta_{h+1}^k + \xi_{h+1}^k }
			    + \lipQ\dist{(x_h^k, a_h^k), (\tx_h^k, \ta_h^k)}
			    + 2 \bonus_h^k(\tx_h^k, \ta_h^k)
			    + \frac{H^2\Xsigmacov}{\gencount_h^k(\tx_h^k, \ta_h^k)}
			    + \lipQ\ksigma
			    + \sum_{h'=h}^H \bias(k, h')
		\end{align*}
		where the constant in front of $\delta_{h+1}^k$ is \emph{exact} (\ie not omitted by $\lesssim$).

		Let $E_h^k \eqdef \braces{\dist{(x_h^k, a_h^k), (\tx_h^k, \ta_h^k)} \leq  2 \ksigma}$. 	Using the definition of the bonus
		\begin{align*}
			\bonus_h^k(\tx_h^k, \ta_h^k) \lesssim \frac{H}{\sqrt{\gencount_h^k(\tx_h^k, \ta_h^k)}} + \frac{\kbeta H}{\gencount_h(\tx_h^k, \ta_h^k)} + \lipQ \ksigma \,,
		\end{align*}
		and the same argument as in the proof of Lemma \ref{lemma:regret-in-terms-of-sum-of-bonus-ucrl-type}, we obtain
		\begin{align*}
			\regret(K)
			& \leq \sum_{k=1}^K\delta_1 ^k + \sum_{k=1}^K\sum_{h=1}^H \bias(k, h) \\
			& \lesssim \sum_{k=1}^K
			\sum_{h=1}^H  \indic{E_h^k}\sqrt{  \frac{ H^2 }{\gencount_h^k(\tx_h^k, \ta_h^k)} }
			+ \sum_{k=1}^K \sum_{h=1}^H \indic{E_h^k}\frac{H^2\Xsigmacov}{\gencount_h^k(\tx_h^k, \ta_h^k)} \\
			& + \sum_{k=1}^K \sum_{h=1}^H \indic{E_h^k} \xi_{h+1}^k
			+ H \sum_{k=1}^K \sum_{h=1}^H \bias(k, h)
			+ H^2 \sigmacov
			+ K H\lipQ \ksigma.
		\end{align*}
		As in Lemma \ref{lemma:regret-in-terms-of-sum-of-bonus-ucrl-type}, we conclude the proof by recalling the definition $E_h^k \eqdef \braces{\dist{(x_h^k, a_h^k), (\tx_h^k, \ta_h^k)} \leq  2 \ksigma}$ and using the fact that $\widetilde{\xi}_{h+1}^k\eqdef \indic{E_h^k} \xi_{h+1}^k$ is a martingale difference sequence with respect to $(\cF_h^k)_{k, h}$ bounded by $4H$.
	\end{proof}

\subsection{Bounding the sum of bonuses and bias}

\begin{flemma}
	\label{lemma:useful-to-bound-bias}
	Let $(\mu_i)_{i\geq}$ be a sequence of non-negative numbers. Then,
	\begin{align*}
	\sum_{k=1}^K \sum_{i= 1 \vee (k-W)}^{k-1} \mu_i \leq 2 W \sum_{i=1}^K \mu_i.
	\end{align*}
\end{flemma}
\begin{proof}
	We have
	\begin{align*}
	\sum_{k=1}^K \sum_{i= 1 \vee (k-W)}^{k-1} \mu_i
	& = \sum_{k=1}^{W} \sum_{i=1}^{k-1} \mu_i +  \sum_{k=W+1}^{K}\sum_{i=k-W}^{k-1} \mu_i
	 \leq W\sum_{i=1}^K \mu_i + \sum_{i=1}^{K-1} \sum_{k=i+1}^{i + W}\mu_i \\
	& \leq W \sum_{i=1}^K \mu_i + W \sum_{i=1}^K \mu_i
	 =  2 W \sum_{i=1}^K \mu_i.
	\end{align*}
\end{proof}

\begin{fcorollary}[bound on the temporal bias]
	\label{corollary:bound-on-temporal-bias}
	Let $\variationR$ and $\variationP$ be the variation of the MDP over $KH$ time steps,
	\begin{align*}
		& \variationR \eqdef \sum_{i=1}^{K}\sum_{h=1}^H \sup_{x, a}\abs{\trueR_h^i(x, a) - \trueR_h^{i+1}(x, a)}
		,\quad
		 \variationP \eqdef \sum_{i=1}^{K}\sum_{h=1}^H \sup_{x, a}\Wassdist{\trueP_h^i(\cdot|x, a), \trueP_h^{i+1}(\cdot|x, a)}.
	\end{align*}
	Then,
	\begin{align*}
		\sum_{k=1}^K\sum_{h=1}^H \bias(h, k) \leq 2\kW \pa{\variationR +  \lipQ \variationP} + \frac{2\timekernelconstA (H+1)KH}{\kbeta}\frac{\keta^\kW}{1-\keta}\,.
	\end{align*}
\end{fcorollary}
\begin{proof}
	From Lemma \ref{lemma:useful-to-bound-bias}, we have
	\begin{align*}
		\sum_{k=1}^K\sum_{h=1}^H \biasP(h, k)
		& =  \sum_{k=1}^K\sum_{h=1}^H \lipQ\sum_{i=1\vee(k-\kW)}^{k-1} \sup_{x, a}\Wassdist{\trueP_h^{i}(\cdot|x, a), \trueP_h^{i+1}(\cdot|x, a)}
		 +  \sum_{k=1}^K\sum_{h=1}^H  \frac{2\timekernelconstA H}{\kbeta} \frac{\keta^\kW}{1-\keta} \\
		& \leq 2 \kW \lipQ \sum_{h=1}^H \sum_{i=1}^{K} \sup_{x, a}\Wassdist{\trueP_h^{i}(\cdot|x, a), \trueP_h^{i+1}(\cdot|x, a)} +   \frac{2\timekernelconstA K H^2}{\kbeta} \frac{\keta^\kW}{1-\keta}.
	\end{align*}
	The sum $\sum_{k=1}^K\sum_{h=1}^H \biasR(h, k) $ is bounded in the same way, which concludes the proof.
\end{proof}

\begin{flemma}[bounding sum on sliding window]
	\label{lemma:bound-sum-sliding-window}
	Let $\braces{a_n}_{n\geq 1}$ be a sequence of real numbers such that $0 \leq a \leq c$ for some constant $c > 0$. Let $A_{t} = \sum_{n= 1\vee(t-W) }^{t-1} a_n$. Then, for any $p, b > 0$,
	\begin{align*}
	\sum_{t=1}^T \frac{a_t}{(1+b A_{t})^p} \leq \sum_{n=1}^{\ceil{T/W}}\pa{ c + \int_0^{A_{nW+1}-c} \frac{1}{(1+bz)^p} \mathrm{d}z }\,.
	\end{align*}
\end{flemma}
\begin{proof}
	We have
	\begin{align*}
	\sum_{t=1}^T \frac{a_t}{(1+b A_{t})^p}
	=  \underbrace{\sum_{t=1}^W \frac{a_t}{(1+b A_{t})^p}}_{\text{\ding{192}}}
	+ \underbrace{\sum_{t=W+1}^T \frac{a_t}{(1+b A_{t})^p}}_{\text{\ding{193}}}.
	\end{align*}
	By Lemma 9 of \cite{domingues2020}, we have
	\begin{align*}
	\text{\ding{192}  } = \sum_{t=1}^W \frac{a_t}{(1+b A_{t})^p} \leq c + \int_{0}^{A_{W+1}-c} \frac{1}{(1+bz)^p} \mathrm{d}z.
	\end{align*}
	Now, we handle \ding{193}:
	\begin{align*}
	\text{\ding{193}  } & \leq \sum_{n=1}^{\ceil{T/W}-1} \sum_{t=nW+1}^{(n+1)W}  \frac{a_t}{(1+b A_{t})^p}
	  =  \sum_{n=1}^{\ceil{T/W}-1} \sum_{l=1}^{W} \frac{a_{l+nW}}{(1+b A_{l+nW})^p} \\
	& \leq \sum_{n=1}^{\ceil{T/W}-1} \pa{ c + \int_0^{A_{(n+1)W+1}-c} \frac{1}{(1+bz)^p} \mathrm{d}z }\,.
	\end{align*}
\end{proof}

\begin{definition}
	Consider a $\ksigma$-covering of $(\stateactionspace, \distfunc)$, $\sigmacovset = \braces{(x_j, a_j) \in \stateactionspace,\; j = 1, \ldots, \sigmacov}$. We define a partition $\braces{B_j}_{j\in[\sigmacov]}$ such that
	\begin{align*}
	B_j = \braces{(x, a) \in \stateactionspace: (x_j, a_j) = \argmin_{(x_i, a_i)\in\sigmacovset}\dist{(x, a), (x_i, a_i)}}
	\end{align*}
	with ties broken arbitrarily.
\end{definition}

\begin{flemma}
	\label{lemma:lower-bound-generalized-counts}
	Let $U_\keta \eqdef \ceil{1/\log(1/\keta)}$ and
	\begin{align*}
		\partitioncount_h^k(B_j, U_\keta) \eqdef \sum_{s=1\vee(k-U_\keta)}^{k-1}  \indic{(x_h^s, a_h^s) \in B_j}.
	\end{align*}
	If  $U_\keta \leq \kW $, $\dist{(x_h^k, a_h^k), (\tx_h^k, \ta_h^k)} \leq  2 \ksigma$ and $(x_h^k, a_h^k) \in B_j$ then
	\begin{align*}
		\gencount_h^k(\tx_h^k, \ta_h^k)
		&\geq \kbeta + \timekernelconstB(4)  e^{-1} \partitioncount_h^k(B_j, U_\keta).
	\end{align*}
\end{flemma}
\begin{proof}
	This result is based on the proof of Proposition 6 of \cite{domingues2020}, which we generalize to the case where the kernel is time-dependent.
		From Assumption \ref{assumption:kernel-properties}, we have $\kernel_{(\ksigma, \keta, \kW)}(k-1-s, z) \geq \timekernelconstB(z)\keta^{k-1-s}$ for all $s \geq k-W$. Consequently, if $U_\keta \leq W$, for all $s \geq k-U_\keta$:
		\begin{align}
			\label{eq:aux_time_kernel_bound}
			\kernel_{(\ksigma, \keta, \kW)}(k-1-s, z) \geq \timekernelconstB(z)\keta^{k-1-s} \geq \timekernelconstB(z)\keta^{U_\keta} \geq \timekernelconstB(z) \exp(-1).
		\end{align}

	Also, if  $\dist{(x_h^k, a_h^k), (\tx_h^k, \ta_h^k)} \leq  2 \ksigma$, $(x_h^k, a_h^k) \in B_j$, and $(x_h^s, a_h^s) \in B_j$, we have
	\begin{align}
		\dist{(\tx_h^k, \ta_h^k), (x_h^s, a_h^s)} & \leq \dist{(\tx_h^k, \ta_h^k), (x_h^k, a_h^k)} + \dist{(x_h^k, a_h^k), (x_h^s, a_h^s)} \nonumber \\
		& \leq 2\ksigma + \dist{(x_h^k, a_h^k), (x_j, a_j)} + \dist{(x_j, a_j), (x_h^s, a_h^s)}
		\leq 4\ksigma.		\label{eq:aux_bounding_dist_for_counts}
	\end{align}
	By Assumption \ref{assumption:kernel-properties}, the function $z \mapsto \kernel_{(\keta, \kW)}(t, z)$ is non-increasing. Together with \eqref{eq:aux_time_kernel_bound} and \eqref{eq:aux_bounding_dist_for_counts}, this yields:
	\begin{align*}
		\gencount_h^k(\tx_h^k, \ta_h^k)
		& = \kbeta + \sum_{s=1}^{k-1}\kernel_{(\ksigma, \keta, \kW)}\pa{ k-1-s, \frac{\dist{(\tx_h^k, \ta_h^k), (x_h^s, a_h^s)}}{\ksigma}} \\
		& \geq \kbeta + \sum_{s=1}^{k-1}\kernel_{(\ksigma, \keta, \kW)}\pa{ k-1-s, \frac{\dist{(\tx_h^k, \ta_h^k), (x_h^s, a_h^s)}}{\ksigma}} \indic{(x_h^s, a_h^s) \in B_j} \\
		& \geq \kbeta + \sum_{s=1\vee(k-U_\keta)}^{k-1}\kernel_{(\ksigma, \keta, \kW)}\pa{ k-1-s, 4}\indic{(x_h^s, a_h^s) \in B_j} \\
		& \geq \kbeta + \timekernelconstB(4) e^{-1} \sum_{s=k-U_\keta}^{k-1} \indic{(x_h^s, a_h^s) \in B_j},
	\end{align*}
	which concludes the proof.
\end{proof}

\begin{flemma}
	\label{lemma:sum-of-bonus}
	Let $U_\keta = \ceil{1/\log(1/\keta)}$. If $U_\keta \leq \kW $, we have
	\begin{align*}
		& \sum_{k=1}^K\sum_{h=1}^H  \frac{1}{\sqrt{\gencount_h^k(\tx_h^k, \ta_h^k)}}\indic{\dist{(x_h^k, a_h^k), (\tx_h^k, \ta_h^k)} \leq  2 \ksigma} \lesssim H\ceil{\frac{K}{U_\keta}}\pa{ \sigmacov + \sqrt{\sigmacov U_\keta}} \\
		& \sum_{k=1}^K\sum_{h=1}^H \frac{1}{\gencount_h^k(\tx_h^k, \ta_h^k)}\indic{\dist{(x_h^k, a_h^k), (\tx_h^k, \ta_h^k)} \leq  2 \ksigma} \lesssim H \sigmacov \ceil{\frac{K}{U_\keta}}\,.
	\end{align*}
\end{flemma}
\begin{proof}

The proof relies on Lemmas \ref{lemma:bound-sum-sliding-window} and \ref{lemma:lower-bound-generalized-counts}.

Here, we define the constant $c$ as $ c = \timekernelconstB(4) \kbeta^{-1}e^{-1} > 0$, since $\timekernelconstB(4) > 0$ by Assumption \ref{assumption:kernel-properties}.
\paragraph{Bounding the sum $\sum_k 1/\sqrt{\gencount_h^k}$}

\begin{align*}
	& \sum_{k=1}^K  \frac{1}{\sqrt{\gencount_h^k(\tx_h^k, \ta_h^k)}}\indic{\dist{(x_h^k, a_h^k), (\tx_h^k, \ta_h^k)} \leq  2 \ksigma} \\
	& = \sum_{j=1}^{\sigmacov}\sum_{k=1}^K  \frac{1}{\sqrt{\gencount_h^k(\tx_h^k, \ta_h^k)}}\indic{\dist{(x_h^k, a_h^k), (\tx_h^k, \ta_h^k)} \leq  2 \ksigma}\indic{(x_h^k, a_h^k)\in B_j} \\
	& \leq \kbeta^{-1/2}\sum_{j=1}^{\sigmacov}\sum_{k=1}^K \frac{\indic{(x_h^k, a_h^k)\in B_j}}{\sqrt{1 +    c \partitioncount_h^k(B_j, U_\keta)}}
	  ,\quad \text{by Lemma \ref{lemma:lower-bound-generalized-counts}} \\
	& \leq \kbeta^{-1/2}\sum_{j=1}^{\sigmacov} \sum_{n=1}^{\ceil{K/U_\keta}}\pa{  1 + \int_{0}^{\partitioncount_h^{nU_\keta+1}(B_j, U_\keta)}  \frac{1}{\sqrt{1+ cz}}\rmd z  }
	,\quad \text{by Lemma \ref{lemma:bound-sum-sliding-window}} \\
	& = \kbeta^{-1/2}\sigmacov\ceil{\frac{K}{U_\keta}}
	    + \frac{2\kbeta^{-1/2}}{c} \sum_{n=1}^{\ceil{K/U_\keta}} \sum_{j=1}^{\sigmacov} \sqrt{1 + c \partitioncount_h^{nU_\keta+1}(B_j, U_\keta) } \\
    & \leq \kbeta^{-1/2}\sigmacov\ceil{\frac{K}{U_\keta}}
           + \frac{2\kbeta^{-1/2}}{c} \sum_{n=1}^{\ceil{K/U_\keta}} \sqrt{\sigmacov}\sqrt{\sigmacov + c\sum_{j=1}^{\sigmacov} \partitioncount_h^{nU_\keta+1}(B_j, U_\keta)}
           ,\quad \text{by Cauchy-Schwarz inequality} \\
    & \leq \kbeta^{-1/2}\sigmacov\ceil{\frac{K}{U_\keta}}
	    + \frac{2\kbeta^{-1/2}}{c} \sum_{n=1}^{\ceil{K/U_\keta}} \sqrt{\sigmacov}\sqrt{\sigmacov + c U_\keta}
	   \lesssim \sigmacov\ceil{\frac{K}{U_\keta}} + \sqrt{\sigmacov U_\keta}\ceil{\frac{K}{U_\keta}}\,.
\end{align*}

\paragraph{Bounding the sum $\sum_k 1/\gencount_h^k$}

	\begin{align*}
		& \sum_{k=1}^K  \frac{1}{\gencount_h^k(\tx_h^k, \ta_h^k)}\indic{\dist{(x_h^k, a_h^k), (\tx_h^k, \ta_h^k)} \leq  2 \ksigma} \\
		& = \sum_{j=1}^{\sigmacov}\sum_{k=1}^K  \frac{1}{\gencount_h^k(\tx_h^k, \ta_h^k)}\indic{\dist{(x_h^k, a_h^k), (\tx_h^k, \ta_h^k)} \leq  2 \ksigma}\indic{(x_h^k, a_h^k)\in B_j} \\
		& \leq \kbeta^{-1}\sum_{j=1}^{\sigmacov}\sum_{k=1}^K  \frac{\indic{(x_h^k, a_h^k)\in B_j}}{1 +    c \partitioncount_h^k(B_j, U_\keta)}
		,\quad \text{by Lemma \ref{lemma:lower-bound-generalized-counts}} \\
		& \leq \kbeta^{-1}\sum_{j=1}^{\sigmacov} \sum_{n=1}^{\ceil{K/U_\keta}}\pa{  1 + \int_{0}^{\partitioncount_h^{nU_\keta+1}(B_j, U_\keta)}  \frac{1}{1+ cz}\rmd z  }
		,\quad \text{by Lemma \ref{lemma:bound-sum-sliding-window}} \\
		& = \kbeta^{-1}\sigmacov\ceil{\frac{K}{U_\keta}}
		    + \frac{\kbeta^{-1}\sigmacov}{c}  \sum_{n=1}^{\ceil{K/U_\keta}}\sum_{j=1}^{\sigmacov} \frac{1}{\sigmacov} \log\pa{1 + c \partitioncount_h^{nU_\keta+1}(B_j, U_\keta)} \\
		& \leq \kbeta^{-1}\sigmacov\ceil{\frac{K}{U_\keta}}
		+ \frac{\kbeta^{-1}\sigmacov}{c}  \sum_{n=1}^{\ceil{K/U_\keta}} \log\pa{1 + \frac{c}{\sigmacov}\sum_{j=1}^{\sigmacov}\partitioncount_h^{nU_\keta+1}(B_j, U_\keta)}
		,\quad\text{by Jensen's inequality} \\
		& \leq  \kbeta^{-1}\sigmacov\ceil{\frac{K}{U_\keta}}
		+  \frac{\kbeta^{-1}\sigmacov}{c} \ceil{\frac{K}{U_\keta}}\log\pa{1 + c\frac{U_\keta}{\sigmacov}}
		\lesssim \sigmacov \ceil{\frac{K}{U_\keta}}\,.
	\end{align*}
\end{proof}

\subsection{Final regret bounds}

\subsubsection{UCRL-type regret bounds}
\begin{ftheorem}[UCRL-type regret bound]
	\label{theorem:regret-bound-ucrl-type}
	If $U_\keta = \ceil{1/\log(1/\keta)} \leq \kW$, the regret of \ouralgo is bounded by
	\begin{align*}
	\regret(K) \lesssim &
	H^2\ceil{\frac{K}{U_\keta}}\sqrt{\Xsigmacov} \pa{ \sigmacov + \sqrt{\sigmacov U_\keta}} +
	H^2 \sigmacov \ceil{\frac{K}{U_\keta}} +  H^{3/2}\sqrt{K}  \\
	& +   \kW \pa{\variationR +  \lipQ \variationP}H + \frac{\keta^\kW}{1-\keta}KH^3 \\
	& + H^2 \sigmacov + \lipQ K H\ksigma
	\end{align*}
	with probability at least $1-\cdelta$, where $\sigmacov$ and  $\Xsigmacov$ are the $\ksigma$-covering numbers of $(\stateactionspace, \distfunc)$ and $(\statespace, \Sdistfunc)$, respectively, and
	\begin{align*}
	& \variationR \eqdef \sum_{i=1}^{K}\sum_{h=1}^H \sup_{x, a}\abs{\trueR_h^i(x, a) - \trueR_h^{i+1}(x, a)}
	,\quad
	\variationP \eqdef \sum_{i=1}^{K}\sum_{h=1}^H \sup_{x, a}\Wassdist{\trueP_h^i(\cdot|x, a), \trueP_h^{i+1}(\cdot|x, a)}
	\end{align*}
	represent the variation of the rewards and transitions, respectively.
\end{ftheorem}
\begin{proof}
	We apply Lemma \ref{lemma:regret-in-terms-of-sum-of-bonus-ucrl-type}, Lemma \ref{lemma:sum-of-bonus} and Corollary \ref{corollary:bound-on-temporal-bias} and the fact that $\prob{\goodevent} \geq 1-\cdelta/2$ by Lemma \ref{lemma:good-event}. To conclude, notice that $\sum_{k=1}^K\sum_{h=1}^H \widetilde{\xi}_{h+1}^k \lesssim H \sqrt{KH}$ with probability at least $1-\cdelta/2$ by Hoeffding-Azuma's inequality.
\end{proof}

\begin{fcorollary}
	\label{corollary:regret-bound-ucrl-type-with-optimized-constants}
	Let $\XAcovdim$ and $\Xcovdim$ be the covering dimensions of $(\stateactionspace, \distfunc)$ and $(\statespace,\Sdistfunc)$, respectively. Let
	$\alpha = \frac{1}{\XAcovdim+\Xcovdim+ 3}$,  $\variationMDPtotal = \variationR +  \lipQ \variationP$ and
	\begin{align*}
		\ksigma = K^{-\alpha}
		,\quad
		\log\pa{\frac{1}{\keta}} =\pa{\frac{\variationMDPtotal}{K^{1+\alpha(\XAcovdim+\Xcovdim)/2}}}^{2/3}
		,\quad
		\kW = \frac{ \log\pa{K/(1-\keta)}  }{ \log\pa{1/\keta}  }
	\end{align*}
	Since $\kW \geq U_\keta = \ceil{1/\log(1/\keta)}$, we have, with probability at least $1-\cdelta$,
	\begin{align*}
		\regret(K)
		\lesssim
		& H^{ 2 } \variationMDPtotal^{\frac{2}{3}}K^{\frac{\XAcovdim+\Xcovdim/6+1}{\XAcovdim+\Xcovdim+3}}
		+ H^{ 2 } \variationMDPtotal^{\frac{1}{3}}K^{\frac{\XAcovdim+\Xcovdim+2}{\XAcovdim+\Xcovdim+3}}  \\
		& + H^{ 2 } \variationMDPtotal^{\frac{2}{3}}K^{\frac{\XAcovdim+\Xcovdim/6+1}{\XAcovdim+\Xcovdim+3}}
		+ H^{\frac{3}{2}}\sqrt{K}
		+ H \variationMDPtotal^{\frac{1}{3}} K^{\frac{\XAcovdim+\Xcovdim+2}{\XAcovdim+\Xcovdim+3}} \log\pa{ \frac{HK}{1-\keta}}
		+ H^3 \\
		& + H^2 K^{\frac{\XAcovdim}{\XAcovdim+\Xcovdim+3}}
		+ \lipQ H K^{\frac{\XAcovdim+\Xcovdim+2}{\XAcovdim+\Xcovdim+3}}
	\end{align*}
	that is,
	\begin{align*}
		\regret(K) = \BigOtilde{   \variationMDPtotal^{\frac{1}{3}} K^{\frac{\XAcovdim+\Xcovdim+2}{\XAcovdim+\Xcovdim+3}} \pa{H^2 + H \log\pa{ \frac{HK}{1-\keta}} } }.
	\end{align*}
	Furthermore, if $\limsup_{K\to\infty} \variationMDPtotal/K = 0$, we have $$\keta = \exp\pa{ - \variationMDPtotal^{2/3} / K^{\frac{\XAcovdim+\Xcovdim+2}{\XAcovdim+\Xcovdim+3}} }
	\underset{K\to\infty}{\sim} 1 - \variationMDPtotal^{2/3} / K^{\frac{\XAcovdim+\Xcovdim+2}{\XAcovdim+\Xcovdim+3}} $$
	which implies
	\begin{align*}
		\regret(K) = \BigOtilde{  H^2 \variationMDPtotal^{\frac{1}{3}} K^{\frac{\XAcovdim+\Xcovdim+2}{\XAcovdim+\Xcovdim+3}} }.
	\end{align*}
\end{fcorollary}
\begin{proof}
	Immediate consequence of Theorem \ref{theorem:regret-bound-ucrl-type} and the fact that $\sigmacov = \BigO{\ksigma^{-\XAcovdim}}$ and $\Xsigmacov = \BigO{\ksigma^{-\Xcovdim}}$.
\end{proof}

\begin{fcorollary}[UCRL-type regret bound in discrete case]
	\label{corollary:regret-bound-ucrl-type-discrete-case}
	If $\stateactionspace$ is finite, we can take $\ksigma = 0$ and $\sigmacov = \Nstates\Nactions$, $\Xsigmacov = \Nstates$, where $\Nstates=\abs{\statespace}$ and $\Nactions = \abs{\actionspace}$. In this case, Theorem \ref{theorem:regret-bound-ucrl-type} and Corollary \ref{corollary:regret-bound-ucrl-type-with-optimized-constants} give us
	\begin{align*}
	\regret(K) = \BigOtilde{ H^2 \Nstates \sqrt{\Nactions} \variationMDPtotal^{\frac{1}{3}} K^{\frac{2}{3}}}.
	\end{align*}
\end{fcorollary}

\subsubsection{UCBVI-type regret bounds}
\begin{ftheorem}[UCBVI-type regret bound]
	\label{theorem:regret-bound-ucbvi-type}
	If $U_\keta = \ceil{1/\log(1/\keta)} \leq \kW$, the regret of \ouralgo is bounded by
	\begin{align*}
	\regret(K) \lesssim &
	H^2\ceil{\frac{K}{U_\keta}}\pa{ \sigmacov + \sqrt{\sigmacov U_\keta}} +
	H^3 \sigmacov\Xsigmacov \ceil{\frac{K}{U_\keta}} +  H^{3/2}\sqrt{K}  \\
	& +   \kW \pa{\variationR +  \lipQ \variationP}H + \frac{\keta^\kW}{1-\keta}KH^3 \\
	& + H \sigmacov + \lipQ K H\ksigma
	\end{align*}
	with probability at least $1-\cdelta$, where $\sigmacov$ and  $\Xsigmacov$ are the $\ksigma$-covering numbers of $(\stateactionspace, \distfunc)$ and $(\statespace, \Sdistfunc)$, respectively, and
	\begin{align*}
		& \variationR \eqdef \sum_{i=1}^{K}\sum_{h=1}^H \sup_{x, a}\abs{\trueR_h^i(x, a) - \trueR_h^{i+1}(x, a)}
		,\quad
		\variationP \eqdef \sum_{i=1}^{K}\sum_{h=1}^H \sup_{x, a}\Wassdist{\trueP_h^i(\cdot|x, a), \trueP_h^{i+1}(\cdot|x, a)}
	\end{align*}
	represent the variation of the rewards and transitions, respectively.
\end{ftheorem}
\begin{proof}
	We apply Lemma \ref{lemma:regret-in-terms-of-sum-of-bonus}, Lemma \ref{lemma:sum-of-bonus} and Corollary \ref{corollary:bound-on-temporal-bias} and the fact that $\prob{\goodevent} \geq 1-\cdelta/2$ by Lemma \ref{lemma:good-event}. To conclude, notice that $\sum_{k=1}^K\sum_{h=1}^H  \pa{1+\frac{1}{H}}^{h} \widetilde{\xi}_{h+1}^k \lesssim H \sqrt{KH}$ with probability at least $1-\cdelta/2$ by Hoeffding-Azuma's inequality.
\end{proof}

\begin{fcorollary}[UCBVI-type regret bound in discrete case]
	\label{corollary:regret-bound-ucbvi-type-discrete-case}
	If $\stateactionspace$ is finite, we can take $\ksigma = 0$ and $\sigmacov = \Nstates\Nactions$, $\Xsigmacov = \Nstates$, where $\Nstates=\abs{\statespace}$ and $\Nactions = \abs{\actionspace}$. In this case, Theorem \ref{theorem:regret-bound-ucbvi-type} gives us
	\begin{align*}
		\regret(K) \lesssim &
		H^2\ceil{\frac{K}{U_\keta}}\pa{ \Nstates\Nactions + \sqrt{\Nstates\Nactions U_\keta}} +
		H^3 \Nstates^2\Nactions \ceil{\frac{K}{U_\keta}} +  H^{3/2}\sqrt{K}  \\
		& +   \kW \pa{\variationR +  \lipQ \variationP}H + \frac{\keta^\kW}{1-\keta}KH^3
		  + H^2 \Nstates\Nactions
	\end{align*}
	Let $\variationMDPtotal = \variationR +  \lipQ \variationP$. By choosing
	\begin{align*}
		\log\pa{\frac{1}{\keta}} =\pa{\frac{\variationMDPtotal}{K}}^{2/3}
		,\quad
		\kW = \frac{ \log\pa{K/(1-\keta)}  }{ \log\pa{1/\keta}  }
	\end{align*}
	we obtain
	\begin{align*}
		\regret(K)  \lesssim
		& H^2 \Nstates\Nactions  \variationMDPtotal^{\frac{2}{3}}K^{\frac{1}{3}}
		+ H^2 \sqrt{\Nstates\Nactions } \variationMDPtotal^{\frac{1}{3}}K^{\frac{2}{3}}
		+ H^3 \Nstates^2\Nactions\variationMDPtotal^{\frac{2}{3}}K^{\frac{1}{3}} \\
		& +  H^{3/2}\sqrt{K}
		  +  \log\pa{\frac{K}{1-\keta}} H \variationMDPtotal^{\frac{1}{3}}K^{\frac{2}{3}}
		  + H^3 + H \Nstates\Nactions.
	\end{align*}
	since $\kW \geq U_\keta = \ceil{1/\log(1/\keta)}$.	Furthermore, if $\limsup_{K\to\infty} \variationMDPtotal/K = 0$, we have $$\keta = \exp\pa{ - \variationMDPtotal^{\frac{2}{3}} / K^{\frac{2}{3}} }
	\underset{K\to\infty}{\sim} 1 - \variationMDPtotal^{\frac{2}{3}} / K^{\frac{2}{3}} $$
	which implies
	\begin{align*}
	\regret(K) = \BigOtilde{ H^2 \sqrt{\Nstates\Nactions } \variationMDPtotal^{\frac{1}{3}}K^{\frac{2}{3}}    + H^3 \Nstates^2\Nactions\variationMDPtotal^{\frac{2}{3}}K^{\frac{1}{3}}}.
	\end{align*}
\end{fcorollary}

\begin{fcorollary}
	\label{corollary:regret-bound-ucbvi-type-with-optimized-constants}
	Let $\XAcovdim$ and $\Xcovdim$ be the covering dimensions of $(\stateactionspace, \distfunc)$ and $(\statespace,\Sdistfunc)$, respectively. Let $\alpha = \frac{1}{\XAcovdim+\Xcovdim+2}$, $\variationMDPtotal = \variationR +  \lipQ \variationP$ and
	\begin{align*}
	& \ksigma = K^{\alpha}
	,\quad
	\log\pa{\frac{1}{\keta}} = \pa{\frac{\variationMDPtotal}{H K^{1+ \alpha(\XAcovdim+\Xcovdim)}}}^{1/2}
	,\quad
	\kW = \ceil{\frac{\log\pa{ K/(1-\keta)}}{\log(1/\keta)}}
	\end{align*}
	Since $\kW \geq U_\keta = \ceil{1/\log(1/\keta)}$, we have, with probability at least $1-\cdelta$,
	\begin{align*}
	\regret(K)  \lesssim
	    H^2\pa{ 1 +
		\log\pa{ \frac{K}{1-\keta}}}\variationMDPtotal^\frac{1}{2} K^{\frac{\XAcovdim+\Xcovdim+1}{\XAcovdim+\Xcovdim+2}}
		+ H^{\frac{3}{2}} \variationMDPtotal^{\frac{1}{4}}K^{\frac{3}{4}}
		+ \lipQ H K^{\frac{\XAcovdim+\Xcovdim+1}{\XAcovdim+\Xcovdim+2}} + H^2 .
	\end{align*}
		Furthermore, if $\limsup_{K\to\infty} \variationMDPtotal/K = 0$, we have $$\keta = \exp\pa{ - \variationMDPtotal^{1/2} / K^{\frac{\XAcovdim+\Xcovdim+1}{\XAcovdim+\Xcovdim+2}} }
	\underset{K\to\infty}{\sim} 1 - \variationMDPtotal^{1/2} / K^{\frac{\XAcovdim+\Xcovdim+1}{\XAcovdim+\Xcovdim+2}} $$
	which implies
	\begin{align*}
	\regret(K) = \BigOtilde{  H^2 \variationMDPtotal^{\frac{1}{2}} K^{\frac{\XAcovdim+\Xcovdim+1}{\XAcovdim+\Xcovdim+2}} + H^{\frac{3}{2}} \variationMDPtotal^{\frac{1}{4}}K^{\frac{3}{4}} }.
	\end{align*}
\end{fcorollary}
\begin{proof}
	Immediate consequence of Theorem \ref{theorem:regret-bound-ucbvi-type} and the fact that $\sigmacov = \BigO{\ksigma^{-\XAcovdim}}$ and $\Xsigmacov = \BigO{\ksigma^{-\Xcovdim}}$.
\end{proof}

\newpage
\section{\RSkerns: An efficient version of \kerns using representative states}
\label{app:rs_kerns_new}

\RSkerns is described in Algorithm~\ref{alg:RS-kerns-ONLINE}, which uses a backward induction on representative states (Algorithm~\ref{alg:RS-kernel_backward_induction-ONLINE}) and updates the model online (algorithms \ref{alg:update-representative-states-appendix} and  \ref{alg:online-kerns-updates}). In this section, we introduce the main definitions used by \RSkerns, and we analyze its runtime and regret.

\subsection{Definitions}

In each episode $k$ and for each $h$, \RSkerns keeps and updates sets of representative states $\ReprStates_h^k$, actions $\ReprActions_h^k$, and next-states $\ReprNextStates_h^k$, with cardinalities $\NreprStates_h^k, \NreprActions_h^k$ and $\NreprNextStates_h^k$, respectively. 
These sets are built using the data observed up to episode $k-1$. We define the following projections:
\begin{align*}
\maptoReprSA_h^{k+1}(x, a) \eqdef \argmin_{(\bx,\ba) \in \ReprStates_h^k\times\ReprActions_h^k} \dist{(x,a),  (\bx,\ba)}
,\quad 
\maptoReprS_h^{k+1}(y) \eqdef \argmin_{\by \in \ReprNextStates_h^k} \Sdist{y, \by}.
\end{align*}
where we also assume to have access to the metric $\Sdistfunc$. The definitions below introduce the kernel function and the estimated MDP used by \RSkerns.

\begin{fdefinition}[kernel function for \RSkerns]
	\label{def:rs-kerns-kernel-choice_new}
	Let $\keta \in ]0, 1]$. 
	\RSkerns uses a kernel of the form
	$
	\fullkernel(t, u, v) =  \timekernelshort(t) \spacekernel\pa{u, v}
	$, 
	where 
	\begin{align*}
	\timekernelshort(t) \eqdef \keta^t
	,\quad 
	\spacekernel\pa{u, v} \eqdef \exp\pa{-\dist{u, v}^2/(2\ksigma^2)}.
	\end{align*}
\end{fdefinition}

\begin{fdefinition}[empirical MDP for \RSkerns]
	\label{def:empirical-mdp-rs-kerns}
	Let 
	\begin{align*}
		\kernsW_{h}^{k+1}(x, a) & = \sum_{s=1}^{k}\timekernelshort(k-s)\spacekernel\pa{\maptoReprSA_h^{k+1}(x, a), \maptoReprSA_h^{s+1}(x_h^s, a_h^s)}.
	\end{align*}
	In episode $k+1$, \RSkerns uses the following estimate of the reward function
	\begin{align*}
	\kernsR_{h}^{k+1}(x, a) 
	& = \frac{1}{ \kbeta + \kernsW_{h}^{k+1}(x, a)} 
	 \sum_{s=1}^{k} \timekernelshort(k-s)\spacekernel\pa{\maptoReprSA_h^{k+1}(x, a), \maptoReprSA_h^{s+1}(x_h^s, a_h^s)}  \obsreward_h^s
	\end{align*}
	and the follow estimate of the transitions
	\begin{align*}
	\kernsP_{h}^{k+1}(y|x,a) & = 
	\frac{1}{ \kbeta + \kernsW_{h}^{k+1}(x, a)}
	\sum_{s=1}^{k} \timekernelshort(k-s)\spacekernel\pa{\maptoReprSA_h^{k+1}(x, a), \maptoReprSA_h^{s+1}(x_h^s, a_h^s)} \dirac{\maptoReprS_h^{s+1}(x_{h+1}^s)}(y)   .
	\end{align*}
	Also, its exploration bonuses are computed as
	\begin{align*}
		& \kernsBonus_{h}^{k+1}(x, a)
	\eqdef
	\BigOtilde{
		\frac{H}{ \sqrt{\kbeta+\kernsW_{h}^{k+1}(x, a)}}
		+ \frac{\kbeta H}{ \kbeta + \kernsW_{h}^{k+1}(x, a)}
		+ \lipQ \ksigma 
	}
	\end{align*}
	where the factors hidden by $\BigOtilde{\cdot}$ are the same as in Definition~\ref{def:exploration-bonuses}.
\end{fdefinition}

At step $h$, \RSkerns needs to store the quantities in Def. \ref{def:empirical-mdp-rs-kerns} \emph{only for the representatives $(x, a)$ in $\ReprStates_h^{k+1}\times\ReprActions_h^{k+1}$ and $y\in\ReprNextStates_h^{k+1}$}. We will show that, using the auxiliary quantities defined below, the values of $\kernsW_{h}^{k}$, $\kernsR_{h}^{k}$ and $\kernsP_{h}^{k}$ can be updated online in $\BigO{\sum_h \NreprStates_h^{k}\NreprActions_h^{k}\NreprNextStates_h^{k}}$ time per episode $k$. 

\begin{fdefinition}[auxiliary quantities for online updates]
	\label{def:rs-kerns-auxiliary}
	For any $(h, x, a)$, we define 
	\begin{align*}
	& \kernsN_{h}^{k+1}(x, a, y) \eqdef \sum_{s=1}^k \timekernelshort(k-s) \indic{ \maptoReprSA_h^{s+1}(x_h^s, a_h^s) = (x, a)} \dirac{\maptoReprS_h^{s+1}(x_{h+1}^s)}(y) \\
	& \kernsN_{h}^{k+1}(x, a) \eqdef \sum_{s=1}^k \timekernelshort(k-s) \indic{ \maptoReprSA_h^{s+1}(x_h^s, a_h^s) = (x, a) } \\
	& \kernsSumR_{h}^{k+1}(x, a) \eqdef \sum_{s=1}^k \timekernelshort(k-s) \indic{ \maptoReprSA_h^{s+1}(x_h^s, a_h^s) = (x, a) } \obsreward_h^s.
	\end{align*}
	Notice that, if $(x, a) \notin \ReprStates_h^{k+1}\times\ReprActions_h^{k+1}$, the quantities above are equal to zero. 
\end{fdefinition}

The following Lemma will be necessary in order to derive online updates.

\begin{flemma}
	\label{lemma:aux-lemma-for-online-updates}
	The empirical MDP used by \RSkerns can be computed as
	\begin{align}
	& \kernsR_{h}^{k+1}(x, a)
	= \frac{ \sum_{(\bx,\ba)} \spacekernel\pa{ \maptoReprSA_h^{k+1}(x,a), (\bx, \ba) }\kernsSumR_{h}^{k+1}(\bx, \ba)  }{ \kbeta + \sum_{(\bx,\ba)} \spacekernel\pa{ \maptoReprSA_h^{k+1}(x,a), (\bx, \ba) }\kernsN_{h}^{k+1}(\bx, \ba) }
	\label{eq:online-R-kerns} \\
	& \kernsP_{h}^{k+1}(y|x, a)
	=  \frac{ \sum_{(\bx,\ba)} \spacekernel\pa{ \maptoReprSA_h^{k+1}(x,a), (\bx, \ba) }\kernsN_{h}^{k+1}(\bx, \ba, y) }{ \kbeta + \sum_{(\bx,\ba)} \spacekernel\pa{ \maptoReprSA_h^{k+1}(x,a), (\bx, \ba) }\kernsN_{h}^{k+1}(\bx, \ba) }
	\label{eq:online-P-kerns} \\
	& \kernsW_{h}^{k+1}(x, a)
	=
	\sum_{(\bx,\ba)} \spacekernel\pa{ \maptoReprSA_h^{k+1}(x,a), (\bx, \ba) }\kernsN_{h}^{k+1}(\bx, \ba)
	\label{eq:online-Counts-kerns}
	\end{align}
	where the sums are over $(\bx,\ba)\in \ReprStates_h^{k+1}\times\ReprActions_h^{k+1}$.
\end{flemma}
\begin{proof}
	It is an immediate consequence of the definitions. For instance,
	\begin{align*}
		& \kernsW_{h}^{k+1}(x, a) = \sum_{s=1}^{k}\timekernelshort(k-s)\spacekernel\pa{\maptoReprSA_h^{k+1}(x, a), \maptoReprSA_h^{s+1}(x_h^s, a_h^s)}
		 \\
		 &
		 = 	\sum_{s=1}^{k}\timekernelshort(k-s)\spacekernel\pa{\maptoReprSA_h^{k+1}(x, a), \maptoReprSA_h^{s+1}(x_h^s, a_h^s)} 
		 \sum_{(\bx,\ba)\in \ReprStates_h^{k+1}\times\ReprActions_h^{k+1}}	
		 \indic{ \maptoReprSA_h^{s+1}(x_h^s, a_h^s) = (\bx, \ba) }
		 \\
		&
		= 	\sum_{(\bx,\ba)}	\sum_{s=1}^{k}\timekernelshort(k-s)\spacekernel\pa{\maptoReprSA_h^{k+1}(x, a), (\bx, \ba)} 
		\indic{ \maptoReprSA_h^{s+1}(x_h^s, a_h^s) = (\bx, \ba) }
		 \\
		&
		= 	\sum_{(\bx,\ba)}	\spacekernel\pa{\maptoReprSA_h^{k+1}(x, a), (\bx, \ba)} 
		\sum_{s=1}^{k}\timekernelshort(k-s)
		\indic{ \maptoReprSA_h^{s+1}(x_h^s, a_h^s) = (\bx, \ba) }
		 \\
		&
		= 	\sum_{(\bx,\ba)}	\spacekernel\pa{\maptoReprSA_h^{k+1}(x, a), (\bx, \ba)} \kernsN_{h}^{k+1}(\bx, \ba).
	\end{align*}
\end{proof}

\begin{algorithm}[H]
	\centering
	\caption{\RSkerns}\label{alg:RS-kerns-ONLINE}
	\begin{algorithmic}[1]
		\State {\bfseries Input:} global parameters $K$, $H$, $\lipQ$, $\lipR$ , $\lipP$, $\kbeta$, $\cdelta$, $\totalcovdim$, $\ksigma$, $\keta$, $\kW$, $\Smaxdist$, $\maxdist$.
		\State Initialize representative states, actions and next states: $\ReprStates_h = \emptyset$, $\ReprActions_h = \emptyset$, $\ReprNextStates_h=\emptyset$, for $h \in [H]$.
		\For{episode $k=1, \ldots, K$}
		\State get initial state $x_1^k$
		\State compute $(\kernsQ_{h}^k)_h$ using kernel backward induction on the representative sets (Alg. \ref{alg:RS-kernel_backward_induction-ONLINE}).
		\For{$h = 1, \ldots, H$}
		\State execute $a_h^k = \argmax_a \kernsQ_{h}^k(x_h^k, a)$, observe reward $\obsreward_h^k$ and next state $x_{h+1}^k$
		\State update representatives $\ReprStates_h$, $\ReprActions_h, \ReprNextStates_h$ using $\braces{ x_h^k, a_h^k, x_{h+1}^k}$ with Alg. \ref{alg:update-representative-states-appendix}
		\State update model using $x_h^k, a_h^k, x_{h+1}^k, \obsreward_h^k$ with Alg. \ref{alg:online-kerns-updates} 
		\EndFor
		\EndFor 
	\end{algorithmic}
\end{algorithm}

\begin{algorithm}[H]
	\centering
	\caption{Kernel Backward Induction on Representative States}\label{alg:RS-kernel_backward_induction-ONLINE}
	\begin{algorithmic}[1]
		\State {\bfseries Input:} $\kernsR_{h}^k(\bx, \ba)$, $\kernsP_h^k(\by|\bx, \ba)$, $\kernsBonus_{h}^k(\bx, \ba)$ for all $ (\bx,\ba,\by)\in\ReprStates_h^{k}\times\ReprActions_h^{k}\times\ReprNextStates_h^{k}$ and all $h \in [H]$. 
		\State {\bfseries Initialization: } $\kernsV_{H+1}(x) = 0$ for all $x\in\statespace$
		\For{$h = H,\ldots, 1$}
		\For{$(\bx, \ba) \in \ReprStates_h^{k}\times\ReprActions_h^{k}$}
		\State $\tQ_{h,\mapzeta}^k(\bx, \ba) = \kernsR_{h}^k(\bx, \ba) + \kernsP_{h}^k \kernsV_{h+1}(\bx, \ba) + \kernsBonus_{h}^k(\bx, \ba)  $ 
		\EndFor
		\State{\color{darkgreen2} // Interpolated $Q$-function. Defined, but not computed for all $(x, a)$}
		\State $\kernsQ_{h}^k(x, a) = \underset{(\bx,\ba) \in\ReprStates_h^{k}\times\ReprActions_h^{k} }{\min}\pa{\tQ_{h,\mapzeta}^k(\bx,\ba)+\lipQ \dist{(x,a), (\bx,\ba)}}$ 
		\If{$h>1$}
		\State{\color{darkgreen2} // Compute $V$-function at the next states for the stage $h-1$}
		\For{$\by \in \ReprNextStates_{h-1}^{k}$}
		\State $\kernsV_{h}^k(\by) = \min\pa{H-h+1, \max_{a}\kernsQ_{h}^k(\by, a)}$ 
		\EndFor
		\EndIf
		\EndFor 
		\State {\bfseries Return:} $(\kernsQ_{h}^k)_{h\in[H]}$
	\end{algorithmic}
\end{algorithm} 

\begin{algorithm}[H]
	\centering
	\caption{Update Representative Sets}\label{alg:update-representative-states-appendix}
	\begin{algorithmic}[1]
		\State {\bfseries Input: $\ReprStates_h^k$, $\ReprActions_h^k$, $\ReprNextStates_h^k$, $\braces{x_h^k, a_h^k, x_{h+1}^k}$}, $\maxdist$, $\Smaxdist$.
		\If{$\min_{(\bx,\ba)\in\ReprStates_h^k\times\ReprActions_h^k}\dist{(\bx,\ba), (x_h^k, a_h^k)} > \maxdist$}
		\State $\ReprStates_h^{k+1} = \ReprStates_h^k \cup \braces{x_h^k}$, \ \  $\ReprActions_h^{k+1} = \ReprActions_h^k \cup \braces{a_h^k}$
		\Else
		\State $\ReprStates_h^{k+1} = \ReprStates_h^k$, \ \ $\ReprActions_h^{k+1} = \ReprActions_h^k$
		\EndIf
		\If{$\min_{\by\in\ReprNextStates_h^k}\Sdist{\bx, x_{h+1}^k} > \Smaxdist$}
		\State $\ReprNextStates_h^{k+1} = \ReprNextStates_h^k \cup \braces{x_{h+1}^k}$
		\Else
		\State  $\ReprNextStates_h^{k+1} = \ReprNextStates_h^k$
		\EndIf
	\end{algorithmic}
\end{algorithm}

\begin{algorithm}[H]
	\centering
	\caption{Online update of \RSkerns Model}\label{alg:online-kerns-updates}
	\begin{algorithmic}[1]
		\State {\bfseries Input: $k, h, x_h^k, a_h^k, x_{h+1}^k, \obsreward_h^k$}.
		\State {\color{darkgreen2} // Map to representatives}
		\State Map $(\tx, \ta) = \maptoReprSA_h^{k+1}(x_h^k, a_h^k)$ and $\ty = \maptoReprS_h^{k+1}(x_{h+1}^k)$
		\State {\color{darkgreen2} // Update auxiliary quantities}
		\State $ \kernsN_{h}^{k+1}(\tx, \ta, \ty) 
		=  1
		+
		\keta \kernsN_{h}^{k}(\tx, \ta, \ty)  $
		\State $ \kernsN_{h}^{k+1}(\tx, \ta) 
		= 1
		+
		\keta \kernsN_{h}^{k}(\tx, \ta)  $
		\State $ \kernsSumR_{h}^{k+1}(\tx, \ta) 
		=  \obsreward_h^k
		+
		\keta \kernsSumR_{h}^{k}(\tx, \ta)  $
		\State {\color{darkgreen2} // Update empirical MDP}
		\For{$(\bx, \ba) \in \ReprStates_h^{k+1}\times\ReprActions_h^{k+1}$}
		\If{$(\bx, \ba) \in \ReprStates_h^{k}\times\ReprActions_h^{k}$}
		\State {\color{darkgreen2} // $(\bx,\ba)$ was added before episode $k$}
		\State $ \kernsW_{h}^{k+1}(\bx, \ba) = \spacekernel\pa{(\bx, \ba), (\tx, \ta)} + \keta\kernsW_{h}^{k}(\bx, \ba)$
		\State $ \kernsR_{h}^{k+1}(\bx, \ba) = \frac{\spacekernel\pa{(\bx, \ba), (\tx, \ta)}}{\kbeta +\kernsW_{h}^{k+1}(\bx, \ba)} \obsreward_h^k
		+ \keta\cdot\pa{\frac{\kbeta +\kernsW_{h}^{k}(\bx, \ba)}{\kbeta +\kernsW_{h}^{k+1}(\bx, \ba)}} \kernsR_{h}^{k}(\bx, \ba) $
		\For{$y \in \ReprNextStates_h^{k+1}$}
		\State $\kernsP_{h}^{k+1}(y | \bx, \ba) =  \frac{\spacekernel\pa{(\bx, \ba), (\tx, \ta)}}{\kbeta +\kernsW_{h}^{k+1}(\bx, \ba)} \dirac{\ty}(y)
		+ \keta\cdot\pa{\frac{\kbeta +\kernsW_{h}^{k}(\bx, \ba)}{\kbeta +\kernsW_{h}^{k+1}(\bx, \ba)}} \kernsP_{h}^{k}(y | \bx, \ba)$
		\EndFor
		\Else 
		\State {\color{darkgreen2} // $(\bx,\ba)$ was added in episode $k$}
		\State Initialize $ \kernsR_{h}^{k+1}(\bx, \ba), \kernsP_{h}^{k+1}(\cdot | \bx, \ba), \kernsW_{h}^{k+1}(\bx, \ba)$ using equations \eqref{eq:online-R-kerns}, \eqref{eq:online-P-kerns} and \eqref{eq:online-Counts-kerns}
		\EndIf
		\EndFor
	\end{algorithmic}
\end{algorithm}

\subsection{Online updates \& runtime}
\label{app:online-updates}

Assume that we observed a transition $\braces{x_h^k, a_h^k, x_{h+1}^k, \obsreward_h^k}$ at time $(k,h)$, updated the representative sets, and mapped the transition to the representatives $(\tx, \ta, \ty) \in \ReprStates_h^{k+1}\times\ReprActions_h^{k+1} \times \ReprNextStates_h^{k+1}$. We wish to update the estimated MDP given in Def.~\ref{def:empirical-mdp-rs-kerns}, which, at step $h$, are only stored for $(x, a)$ in $\ReprStates_h^{k+1}\times\ReprActions_h^{k+1}$ and $y\in\ReprNextStates_h^{k+1}$.

The auxiliary quantities (Def.~\ref{def:rs-kerns-auxiliary}) are updated as:
\begin{align*}
	& \kernsN_{h}^{k+1}(\tx, \ta, \ty) 
	=  1
	+
	\keta \kernsN_{h}^{k}(\tx, \ta, \ty)
	\\
	&
	 \kernsN_{h}^{k+1}(\tx, \ta) 
	= 1
	+
	\keta \kernsN_{h}^{k}(\tx, \ta) 
	\\
	& 
	\kernsSumR_{h}^{k+1}(\tx, \ta) 
	=  \obsreward_h^k
	+
	\keta \kernsSumR_{h}^{k}(\tx, \ta).
\end{align*}

We need to update $\kernsW_{h}^{k}$, $\kernsR_{h}^{k}$ and $\kernsP_{h}^{k}$ for all $(\bx,\ba, \by) \in \ReprStates_h^{k+1}\times\ReprActions_h^{k+1}\times\ReprNextStates_h^{k+1}$. The update rule will depend on whether the $(\bx, \ba)$ is a \emph{new} representative state-action pair (included in episode $k$) or it was \emph{visited before episode $k$}. These two cases are studied below.

\paragraph{Case 1: $(\bx, \ba) \in \ReprStates_h^{k+1}\times\ReprActions_h^{k+1}$ and $(\bx, \ba) \notin \ReprStates_h^{k}\times\ReprActions_h^{k}$} This means that the representative state-action pair $(\bx, \ba)$ was added at time $(k, h)$. In this case, for all $y \in \ReprNextStates_h^{k+1}$, the quantities $\kernsR_{h}^{k+1}(\bx, \ba)$, $\kernsP_{h}^{k+1}(y|\bx, \ba)$ and $\kernsW_{h}^{k+1}(\bx, \ba)$ can be initialized using equations \eqref{eq:online-R-kerns}, \eqref{eq:online-P-kerns} and \eqref{eq:online-Counts-kerns}. This is done in $\BigO{\NreprStates_h^{k+1} \NreprActions_h^{k+1} \NreprNextStates_h^{k+1}}$ time and can happen, at most, for one pair $(\bx,\ba)$: the one that was newly added. Therefore, we have a total per-episode runtime  of  $\BigO{ \sum_{h=1}^H\NreprStates_h^{k+1} \NreprActions_h^{k+1} \NreprNextStates_h^{k+1}}$ taking this case into account.

\paragraph{Case 2: $(\bx, \ba) \in \ReprStates_h^{k}\times\ReprActions_h^{k}$} This means that the representative state-action pair $(\bx, \ba)$ was added \emph{before} episode $k$, which implies that $\maptoReprSA_h^{k+1}(\bx, \ba) = \maptoReprSA_h^k(\bx, \ba) = (\bx, \ba)$. Hence,
\begin{align*}
\kernsW_{h}^{k+1}(\bx, \ba) & = \sum_{s=1}^{k}\timekernelshort(k-s)\spacekernel\pa{\maptoReprSA_h^{k+1}(\bx, \ba), \maptoReprSA_h^{s+1}(x_h^s, a_h^s)} \\
& = \spacekernel\pa{\maptoReprSA_h^{k+1}(\bx, \ba), \maptoReprSA_h^{k+1}(x_h^k, a_h^k)} + \sum_{s=1}^{k-1}\keta^{k-s}\spacekernel\pa{\maptoReprSA_h^{k}(\bx, \ba), \maptoReprSA_h^{s+1}(x_h^s, a_h^s)} \\
& = \spacekernel\pa{(\bx, \ba), \maptoReprSA_h^{k+1}(x_h^k, a_h^k)} + \keta\sum_{s=1}^{k-1}\keta^{k-s-1}\spacekernel\pa{\maptoReprSA_h^{k}(\bx, \ba), \maptoReprSA_h^{s+1}(x_h^s, a_h^s)} \\
& = \spacekernel\pa{(\bx, \ba), \maptoReprSA_h^{k+1}(x_h^k, a_h^k)} + \keta\kernsW_{h}^{k}(\bx, \ba).
\end{align*} 
This implies that, for a fixed $(\bx,\ba)$, the quantity $\kernsW_{h}^{k+1}(\bx, \ba)$ can be updated in $\BigO{1}$ time, assuming that the mapping $\maptoReprSA_h^{k+1}(x_h^k, a_h^k)$ was previously computed (this mapping is only computed \emph{once} for all the updates, and takes $\BigO{\NreprStates_h^{k+1}\times\NreprActions_h^{k+1}}$ time).

Now, notice that 
\begin{align*}
\kernsR_{h}^{k+1}(\bx, \ba) 
& =  \frac{ \sum_{s=1}^{k} \timekernelshort(k-s)\spacekernel\pa{\maptoReprSA_h^{k+1}(\bx, \ba), \maptoReprSA_h^{s+1}(x_h^s, a_h^s)}  \obsreward_h^s}{ \kbeta +\kernsW_{h}^{k+1}(\bx, \ba)} \\
&  = \frac{\spacekernel\pa{(\bx, \ba), \maptoReprSA_h^{k+1}(x_h^k, a_h^k)}}{\kbeta +\kernsW_{h}^{k+1}(\bx, \ba)} \obsreward_h^k
+ \keta \cdot \pa{\frac{\kbeta +\kernsW_{h}^{k}(\bx, \ba)}{\kbeta +\kernsW_{h}^{k+1}(\bx, \ba)}} \kernsR_{h}^{k}(\bx, \ba) 
\end{align*}
where we used again the fact that, in this case, $\maptoReprSA_h^{k+1}(\bx, \ba) = \maptoReprSA_h^k(\bx, \ba)$. Hence, similarly to $\kernsW_{h}^{k+1}(\bx, \ba)$, the quantity $\kernsR_{h}^{k+1}(\bx, \ba)$ can be updated in $\BigO{1}$ time. A similar reasoning shows that $\kernsP_{h}^{k+1}(y, \bx, \ba)$ can be updated, for all $y \in \ReprNextStates_h^{k+1}$, in $\BigO{\NreprNextStates_h^{k+1}}$ time:
\begin{align*}
	\kernsP_{h}^{k+1}(y | \bx, \ba) =  \frac{\spacekernel\pa{(\bx, \ba), \maptoReprSA_h^{k+1}(x_h^k, a_h^k)}}{\kbeta +\kernsW_{h}^{k+1}(\bx, \ba)} \dirac{\maptoReprS_h^{k+1}(x_{h+1}^k)}(y)
	+ \keta \cdot \pa{\frac{\kbeta +\kernsW_{h}^{k}(\bx, \ba)}{\kbeta +\kernsW_{h}^{k+1}(\bx, \ba)}} \kernsP_{h}^{k}(y | \bx, \ba).
\end{align*}

\paragraph{Summary}  Every time a new transition is observed at time $(k, h)$, the estimators for all $(x, a, y) \in \ReprStates_h^{k+1}\times\ReprActions_h^{k+1}\times\ReprNextStates_h^{k+1}$ must be updated. For a given representative $(x, a)$, the updates can be done in $\BigO{\NreprNextStates_h^{k+1}}$ time if it has been observed before episode $k$ (case 2). This results in a total runtime, per episode, of $\BigO{ \sum_h \NreprStates_h^{k+1} \NreprActions_h^{k+1} \NreprNextStates_h^{k+1} }$ for all the representatives observed before episode $k$. If the representative $(x, a)$ has not been observed before episode $k$ (case 1), the updates require  $\BigO{\NreprStates_h^{k+1} \NreprActions_h^{k+1} \NreprNextStates_h^{k+1} }$ time, and this can happen, at most, for one state-action pair at each time $(k, h)$.  Hence, the total runtime required for the updates is $\BigO{\sum_h \NreprStates_h^{k+1} \NreprActions_h^{k+1} \NreprNextStates_h^{k+1} }$ per episode.


\subsection{Regret analysis}

The regret analysis of \RSkerns is based on the following result, which is a corollary of Lemma \ref{lemma:gaussian-case-weighted-average-and-bonuses-are-lipschitz}, and is used to bound the bias introduced by using representative states. 

\begin{fcorollary}
	\label{corollary:lipschitznes_of_kerns_with_gaussian_kernel}
	Let $\timekernel:\NN \to [0, 1]$, consider the following kernel
	\begin{align*}
	\fullkernel(t, u, v) = \timekernel(t) \exp\pa{-\dist{u, v}^2/(2\ksigma^2)}
	\end{align*}
	where $u, v\in\stateactionspace$. For $s < k$, let 
	\begin{align*}
	\weight_h^{k, s}(x, a) = \fullkernel(k-s-1, (x, a), (x_h^s,a_h^s)), \; \weight_{h,\mapzeta}^{k, s}(x, a) \eqdef \fullkernel\pa{k-s-1, \maptoReprSA_h^{k}(x, a), \maptoReprSA_h^k(x_h^s, a_h^s)} 
	\end{align*}
	and consider the functions
	\begin{align*}
	& g_1(x, a) = \frac{\sum_{s=1}^{k-1} \weight_h^{k, s}(x, a) Y_s}{\kbeta+\sum_{s=1}^{k-1} \weight_h^{k, s}(x, a)}
	,\quad
	g_1^{\mapzeta}(x, a) = \frac{\sum_{s=1}^{k-1} \weight_{h,\mapzeta}^{k, s}(x, a) Y_s}{\kbeta+\sum_{s=1}^{k-1} \weight_{h,\mapzeta}^{k, s}(x, a)}
	, \\ 
	& g_2(x, a) = \sqrt{\frac{1}{\kbeta+\sum_{s=1}^{k-1} \weight_h^{k, s}(x, a)}}
	, \quad
	g_2^{\mapzeta}(x, a) = \sqrt{\frac{1}{\kbeta+\sum_{s=1}^{k-1} \weight_{h,\mapzeta}^{k, s}(x, a)}}
	,\\
	& g_3(x, a) = \frac{1}{\kbeta+\sum_{s=1}^{k-1} \weight_h^{k, s}(x, a)}
	,\quad
	g_3^{\mapzeta}(x, a) = \frac{1}{\kbeta+\sum_{s=1}^{k-1} \weight_{h,\mapzeta}^{k, s}(x, a)}
	.
	\end{align*}
	where $(Y_s)_{s=1}^{k-1}$ is an arbitrary sequence. 
	Then, $g_1$, $g_2$ and $g_3$ are \lipschitz continuous, whose \lipschitz constants are bounded by $L_1$, $L_2$ and $L_3$, respectively, with
	\begin{align*}
	L_1 = \frac{ 4 \max_s\abs{Y_s}}{\ksigma} & \pa{1+\sqrt{\logplus(k/\kbeta)}}
	,\quad 
	L_2 = \frac{1+\sqrt{\logplus(k/\kbeta)}}{2 \kbeta^{1/2}\ksigma}
	,\quad 
	L_3 = \frac{1+\sqrt{\logplus(k/\kbeta)}}{\kbeta \ksigma}
	\end{align*}
	Furthermore, for any $(x, a)$ and for $i \in \braces{1, 2, 3}$,
	\begin{align*}
	\abs{g_i^{\mapzeta}(x, a)-g_i(\maptoReprSA_h^k(x, a))} \leq L_i \max_s \dist{(x_h^s, a_h^s), \maptoReprSA_h^k(x_h^s, a_h^s)}.
	\end{align*}
\end{fcorollary}
\begin{proof} First, let's prove that $g_1$, $g_2$ and $g_3$ are Lipschitz continuous.
	From Lemma \ref{lemma:gaussian-case-weighted-average-and-bonuses-are-lipschitz}, taking $z= (\dist{(x,a),(x_h^s,a_h^s)})_{s=1}^{k-1}$ and $y = (\dist{(x',a'),(x_h^s,a_h^s)})_{s=1}^{k-1}$ we have
	\begin{align*}
	\abs{g_1(x, a)-g_1(x', a')} \leq \frac{ 4 \max_s\abs{Y_s}}{\ksigma} \pa{1+\sqrt{\logplus(k/\kbeta)}} \max_{s\in[k-1]} \abs{z_s - y_s}.
	\end{align*}
	By the triangle inequality, for all $s$,
	\begin{align*}
	\abs{z_s - y_s} = \abs{ \dist{(x,a),(x_h^s,a_h^s)} - \dist{(x',a'),(x_h^s,a_h^s)}} \leq \dist{(x,a),(x',a')}
	\end{align*}
	which implies
	\begin{align*}
	\abs{g_1(x, a)-g_1(x', a')} \leq \frac{ 4 \max_s\abs{Y_s}}{\ksigma} \pa{1+\sqrt{\logplus(k/\kbeta)}} \dist{(x,a),(x',a')}.
	\end{align*}
	giving the \lipschitz constant of $g_1$. The \lipschitz constants of $g_2$ and $g_3$ follow similarly from Lemma \ref{lemma:gaussian-case-weighted-average-and-bonuses-are-lipschitz}.
	
	Now, let's bound the differences $\abs{g_i^{\mapzeta}(x, a)-g_i(\maptoReprSA_h^k(x, a))}$. Let $\chi(s) = \timekernel(k-s-1)$. For $i=1$,  and applying again Lemma  \ref{lemma:gaussian-case-weighted-average-and-bonuses-are-lipschitz}, we obtain
	\begin{align*}
	& \abs{g_1^{\mapzeta}(x, a)-g_1(\maptoReprSA_h^k(x, a))}  \\
	& = \abs{
		\frac{\sum_{s=1}^{k-1} \chi(s) \exp\pa{-\frac{\dist{\maptoReprSA_h^k(x,a), \maptoReprSA_h^k(x_h^s,a_h^s)}^2}{2\ksigma^2}} Y_s}{\kbeta+\sum_{s=1}^{k-1} \chi(s) \exp\pa{-\frac{\dist{\maptoReprSA_h^k(x,a), \maptoReprSA_h^k(x_h^s,a_h^s)}^2}{2\ksigma^2}}}
		-
		\frac{\sum_{s=1}^{k-1} \chi(s) \exp\pa{-\frac{\dist{\maptoReprSA_h^k(x,a), (x_h^s,a_h^s)}^2}{2\ksigma^2}} Y_s}{\kbeta+\sum_{s=1}^{k-1} \chi(s) \exp\pa{-\frac{\dist{\maptoReprSA_h^k(x,a), (x_h^s,a_h^s)}^2}{2\ksigma^2}}}
	} \\
	& \leq \frac{ 4 \max_s\abs{Y_s}}{\ksigma} \pa{1+\sqrt{\logplus(k/\kbeta)}} \max_{s\in[k-1]} \abs{\dist{\maptoReprSA_h^k(x,a), \maptoReprSA_h^k(x_h^s,a_h^s)} - \dist{\maptoReprSA_h^k(x,a), (x_h^s,a_h^s)}} \\
	& \leq \frac{ 4 \max_s\abs{Y_s}}{\ksigma} \pa{1+\sqrt{\logplus(k/\kbeta)}} \max_{s\in[k-1]} \dist{\maptoReprSA_h^k(x_h^s,a_h^s), (x_h^s,a_h^s)},
	\end{align*}
	where, in the last line, we used the triangle inequality.  The proof for $i \in \braces{2, 3}$ also follow from Lemma \ref{lemma:gaussian-case-weighted-average-and-bonuses-are-lipschitz}.
\end{proof}

We use Corollary \ref{corollary:lipschitznes_of_kerns_with_gaussian_kernel} to bound the difference between the estimators and the bonuses of \kerns,  $(\estR_h^k, \estP_h^k,\bonus_h^k)$, and the ones of \RSkerns, $(\kernsR_h^k, \kernsP_h^k,\kernsBonus_h^k)$.
\begin{flemma}
	\label{lemma:error-in-P-when-usin-representative-states_NEW}
	Let $V$ be an arbitrary $\lipQ$-\lipschitz function bounded by $H$. Then, for any $(x, a)$,
	\begin{align*}
	\abs{\pa{\estP_h^k - \kernsP_h^k}V(x, a)} 
	\leq &  
	  \frac{4H}{\ksigma}\pa{1+\sqrt{\logplus(k/\kbeta)}}\pa{ \dist{(x,a),\maptoReprSA_h^k(x,a)} +\maxdist  } \\
	& + 3 \lipQ \Smaxdist  + 8H\pa{1+\sqrt{\logplus(k/\kbeta)}} \frac{\maxdist}{\ksigma}
	\end{align*}
\end{flemma}
\begin{proof}
	To simplify the notations, let $f(s) \eqdef \timekernelshort(k-s-1)$ for $s\in[k-1]$. We have
	\begin{align*}
		& 
		\abs{\pa{\estP_h^k - \kernsP_h^k}V(x, a)}  
		\\
		& 
		= \Bigg|
		  \frac{\sum_{s=1}^{k-1} f(s)\spacekernel\pa{(x, a), (x_h^s, a_h^s)}  V(x_{h+1}^s) }
		 { \kbeta + \sum_{s=1}^{k-1}f(s)\spacekernel\pa{(x, a), (x_h^s, a_h^s)}}
		 - 
		 \frac{\sum_{s=1}^{k-1} f(s)\spacekernel\pa{\maptoReprSA_h^{k}(x, a), \maptoReprSA_h^{s+1}(x_h^s, a_h^s)}  V(\maptoReprS_h^{s+1}(x_{h+1}^s)) }
		 { \kbeta + \sum_{s=1}^{k-1}f(s)\spacekernel\pa{\maptoReprSA_h^{k}(x, a), \maptoReprSA_h^{s+1}(x_h^s, a_h^s)}}
	    \Bigg| 
	    \\
	    & \leq \dingone + \dingtwo 
	\end{align*}
	where $\dingone$ and $\dingtwo$ are defined and bounded below. First, 
	\begin{align*}
		\dingone 
		& =
		\Bigg|
		\frac{\sum_{s=1}^{k-1} f(s)\spacekernel\pa{(x, a), (x_h^s, a_h^s)}  V(x_{h+1}^s) }
		{ \kbeta + \sum_{s=1}^{k-1}f(s)\spacekernel\pa{(x, a), (x_h^s, a_h^s)}}
		-
		\frac{\sum_{s=1}^{k-1}f(s)\spacekernel\pa{\maptoReprSA_h^{k}(x, a), \maptoReprSA_h^{k}(x_h^s, a_h^s)}  V(\maptoReprS_h^{k}(x_{h+1}^s)) }
		{ \kbeta + \sum_{s=1}^{k-1}f(s)\spacekernel\pa{\maptoReprSA_h^{k}(x, a), \maptoReprSA_h^{k}(x_h^s, a_h^s)}}
		\Bigg|
		\\
		&
		\leq 
		\abs{
			 \frac{\sum_{s=1}^{k-1} f(s)\spacekernel\pa{(x, a), (x_h^s, a_h^s)}  V(x_{h+1}^s) }
			 { \kbeta + \sum_{s=1}^{k-1}f(s)\spacekernel\pa{(x, a), (x_h^s, a_h^s)}}
			 -
			 \frac{\sum_{s=1}^{k-1} f(s)\spacekernel\pa{(x, a), (x_h^s, a_h^s)}  V(\maptoReprS_h^{k}(x_{h+1}^s)) }
			 { \kbeta + \sum_{s=1}^{k-1}f(s)\spacekernel\pa{(x, a), (x_h^s, a_h^s)}}
		 }
	 	\\
	 	&
	 	+ \abs{
	 		\frac{\sum_{s=1}^{k-1}f(s)\spacekernel\pa{(x, a), (x_h^s, a_h^s)}  V(\maptoReprS_h^{k}(x_{h+1}^s)) }
	 		{ \kbeta + \sum_{s=1}^{k-1}f(s)\spacekernel\pa{(x, a), (x_h^s, a_h^s)}}
	 		-
	 		\frac{\sum_{s=1}^{k-1} f(s)\spacekernel\pa{\maptoReprSA_h^{k}(x, a), (x_h^s, a_h^s)}  V(\maptoReprS_h^{k}(x_{h+1}^s)) }
	 		{ \kbeta + \sum_{s=1}^{k-1}f(s)\spacekernel\pa{\maptoReprSA_h^{k}(x, a), (x_h^s, a_h^s)}}
	 	}
 		\\
 		& 
 		+ \abs{
 			\frac{\sum_{s=1}^{k-1} f(s)\spacekernel\pa{\maptoReprSA_h^{k}(x, a), (x_h^s, a_h^s)}  V(\maptoReprS_h^{k}(x_{h+1}^s)) }
 			{ \kbeta + \sum_{s=1}^{k-1}f(s)\spacekernel\pa{\maptoReprSA_h^{k}(x, a), (x_h^s, a_h^s)}}
 			-
 			\frac{\sum_{s=1}^{k-1} f(s)\spacekernel\pa{\maptoReprSA_h^{k}(x, a), \maptoReprSA_h^{k}(x_h^s, a_h^s)}  V(\maptoReprS_h^{k}(x_{h+1}^s)) }
 			{ \kbeta + \sum_{s=1}^{k-1}f(s)\spacekernel\pa{\maptoReprSA_h^{k}(x, a), \maptoReprSA_h^{k}(x_h^s, a_h^s)}}
 		}
 		\\
 		& 
 		\leq 
 		\lipQ \max_{s\in[k-1]}\Sdist{x_{h+1}^s, \maptoReprS_h^k(x_{h+1}^s)} 
 		\\
 		&
 		+  \frac{ 4 H}{\ksigma} \pa{1+\sqrt{\logplus(k/\kbeta)}}\dist{(x,a),\maptoReprSA_h^k(x,a)}
 		\\
 		& 
 		+ \frac{ 4 H}{\ksigma} \pa{1+\sqrt{\logplus(k/\kbeta)}} \max_s \dist{(x_h^s, a_h^s), \maptoReprSA_h^k(x_h^s, a_h^s)}
	\end{align*}
	by using the fact that $V$ is $\lipQ$-\lipschitz and Corollary \ref{corollary:lipschitznes_of_kerns_with_gaussian_kernel}. 
	
	To bound $\dingtwo$, let $z_s = \dist{ \maptoReprSA^{k}(x, a), \maptoReprSA^{s+1}(x_h^s, a_h^s) }$ and $y_s = \dist{\maptoReprSA^{k}(x, a), \maptoReprSA^{k}(x_h^s, a_h^s)}$, for $s\in[k-1]$. Using again the fact that $V$ is $\lipQ$-\lipschitz and Corollary \ref{corollary:lipschitznes_of_kerns_with_gaussian_kernel}, we obtain
	\begin{align*}
		\dingtwo
		& =
		\abs{
			\frac{\sum_{s=1}^{k-1}f(s)\spacekernel\pa{\maptoReprSA_h^{k}(x, a), \maptoReprSA_h^{k}(x_h^s, a_h^s)}  V(\maptoReprS_h^{k}(x_{h+1}^s)) }
			{ \kbeta + \sum_{s=1}^{k-1}f(s)\spacekernel\pa{\maptoReprSA_h^{k}(x, a), \maptoReprSA_h^{k}(x_h^s, a_h^s)}}
			-
			\frac{\sum_{s=1}^{k-1} f(s)\spacekernel\pa{\maptoReprSA_h^{k}(x, a), \maptoReprSA_h^{s+1}(x_h^s, a_h^s)}  V(\maptoReprS_h^{s+1}(x_{h+1}^s)) }
			{ \kbeta + \sum_{s=1}^{k-1}f(s)\spacekernel\pa{\maptoReprSA_h^{k}(x, a), \maptoReprSA_h^{s+1}(x_h^s, a_h^s)}}
		}
		\\
		& 
		\leq \lipQ \max_s \Sdist{\maptoReprS^{s+1}(x_{h+1}^s), \maptoReprS^{k}(x_{h+1}^s)}
		+  \frac{4H}{\ksigma}\pa{1+\sqrt{\logplus(k/\kbeta)}} \max_s \abs{z_s-y_s}\\
		& \leq 2\lipQ\Smaxdist + \frac{4H}{\ksigma}\pa{1+\sqrt{\logplus(k/\kbeta)}} \max_s \abs{\dist{ \maptoReprSA^{k}(x_h^s, a_h^s), \maptoReprSA^{s+1}(x_h^s, a_h^s) }} \\
		& \leq  2 \lipQ \Smaxdist  + 8H\pa{1+\sqrt{\logplus(k/\kbeta)}} \frac{\maxdist}{\ksigma}.
	\end{align*}
	By the construction of the algorithm, $\Sdist{x_{h+1}^s, \maptoReprS_h^k(x_{h+1}^s)}\leq\Smaxdist$, $\dist{(x_h^s,a_h^s),\maptoReprSA_h^k(x_h^s,a_h^s)}\leq\maxdist$ ,  $\Sdist{\maptoReprS^{s+1}(x_{h+1}^s), \maptoReprS^{k}(x_{h+1}^s)} \leq 2\Smaxdist$ and $\dist{ \maptoReprSA^{k}(x_h^s, a_h^s), \maptoReprSA^{s+1}(x_h^s, a_h^s) }\leq 2\maxdist$, which concludes the proof.
\end{proof}

\begin{flemma}
	\label{lemma:error-in-R-when-usin-representative-states_NEW}
	For any $(x, a)$,
	\begin{align*}
	\abs{\estR_h^k(x,a) - \kernsR_h^k(x, a)} 
	\leq 
	& \frac{4}{\ksigma}\pa{1+\sqrt{\logplus(k/\kbeta)}}\pa{ \dist{(x,a),\maptoReprSA_h^k(x,a)} + \maxdist } \\
	&   + 8\pa{1+\sqrt{\logplus(k/\kbeta)}} \frac{\maxdist}{\ksigma}
	\end{align*}
\end{flemma}
\begin{proof}
	It follows from a similar decomposition as in the proof of Lemma \ref{lemma:error-in-P-when-usin-representative-states_NEW}, and from Corollary \ref{corollary:lipschitznes_of_kerns_with_gaussian_kernel}.
\end{proof}

\begin{flemma}
	\label{lemma:error-in-bonus-when-usin-representative-states_NEW}
	For any $(x, a)$, we have,
	\begin{align*}
	\abs{ \bonus_h^k(x, a) - \kernsBonus_{h}^k(x, a)} \lesssim  
	& \frac{H}{\ksigma}\pa{1+\frac{1}{2 \kbeta^{1/2}}}\pa{ \dist{(x,a),\maptoReprSA_h^k(x,a)} + \maxdist}
	+ H\pa{1+\frac{1}{\kbeta^{1/2}}}\frac{\maxdist}{\ksigma}
	\end{align*}
\end{flemma}
\begin{proof}
	To simplify the notations, let $f(s) \eqdef \timekernelshort(k-s-1)$ for $s\in[k-1]$. Using definitions \ref{def:exploration-bonuses}, \ref{def:rs-kerns-kernel-choice_new}, and \ref{def:empirical-mdp-rs-kerns}, we have
	\begin{align*}
	& \abs{ \bonus_h^k(x, a) - \kernsBonus_{h}^k(x, a)}  \\
	& \lesssim
	H 
	\underbrace{
	\abs{ 
		\sqrt{  \frac{1}{\kbeta + \sum_{s=1}^{k-1}f(s)\spacekernel\pa{(x, a), (x_h^s, a_h^s)}}} 
	- 
	\sqrt{\frac{1}{	\kbeta + \sum_{s=1}^{k-1}f(s)\spacekernel\pa{\maptoReprSA_h^{k}(x, a), \maptoReprSA_h^{s+1}(x_h^s, a_h^s)}}}
	} 
	}_{\dingone}
	\\
	& 
	+ \kbeta H
	\underbrace{
	\abs{ 
		 \frac{1}{\kbeta + \sum_{s=1}^{k-1}f(s)\spacekernel\pa{(x, a), (x_h^s, a_h^s)}}
		- 
		\frac{1}{	\kbeta + \sum_{s=1}^{k-1}f(s)\spacekernel\pa{\maptoReprSA_h^{k}(x, a), \maptoReprSA_h^{s+1}(x_h^s, a_h^s)}}
	}
	}_{\dingtwo}.
	\end{align*}
	Using Corollary \ref{corollary:lipschitznes_of_kerns_with_gaussian_kernel}, we obtain
	\begin{align*}
		\dingone & \leq
			\abs{ 
			\sqrt{  \frac{1}{\kbeta + \sum_{s=1}^{k-1}f(s)\spacekernel\pa{(x, a), (x_h^s, a_h^s)}}} 
			- 
			\sqrt{\frac{1}{	\kbeta + \sum_{s=1}^{k-1}f(s)\spacekernel\pa{\maptoReprSA_h^{k}(x, a), (x_h^s, a_h^s)}}}
		} 
		\\
		& +
			\abs{ 
			\sqrt{  \frac{1}{\kbeta + \sum_{s=1}^{k-1}f(s)\spacekernel\pa{\maptoReprSA_h^{k}(x, a), (x_h^s, a_h^s)}}} 
			- 
			\sqrt{\frac{1}{	\kbeta + \sum_{s=1}^{k-1}f(s)\spacekernel\pa{\maptoReprSA_h^{k}(x, a), \maptoReprSA_h^{k}(x_h^s, a_h^s)}}}
		}  
		\\
		& +
		\abs{ 
			\sqrt{  \frac{1}{\kbeta + \sum_{s=1}^{k-1}f(s)\spacekernel\pa{\maptoReprSA_h^{k}(x, a), \maptoReprSA_h^{k}(x_h^s, a_h^s)}}} 
			- 
			\sqrt{\frac{1}{	\kbeta + \sum_{s=1}^{k-1}f(s)\spacekernel\pa{\maptoReprSA_h^{k}(x, a), \maptoReprSA_h^{s+1}(x_h^s, a_h^s)}}}
		}  \\
	    & 
	    \leq  \pa{\frac{1+\sqrt{\logplus(k/\kbeta)}}{2 \kbeta^{1/2}\ksigma}}
	    \Bigg( 
	    	\dist{(x,a),\maptoReprSA_h^k(x, a)} 
	    	+ \max_{s\in[k-1]} \dist{(x_h^s, a_h^s), \maptoReprSA_h^k(x, a)}
	   \\
	    & 
	    \quad\quad\quad\quad\quad\quad\quad\quad\quad\quad\quad\quad
	    + \max_{s\in[k-1]} \dist{ \maptoReprSA_h^k(x_h^s, a_h^s), \maptoReprSA_h^{s+1}(x, a)}
	     \Bigg).
	\end{align*}
	
	Similarly, Corollary \ref{corollary:lipschitznes_of_kerns_with_gaussian_kernel} yields
	\begin{align*}
		\dingtwo & 
		\leq  \pa{\frac{1+\sqrt{\logplus(k/\kbeta)}}{\kbeta \ksigma}}
		\Bigg( 
		\dist{(x,a),\maptoReprSA_h^k(x, a)} 
		+ \max_{s\in[k-1]} \dist{(x_h^s, a_h^s), \maptoReprSA_h^k(x, a)}
		\\
		& 
		\quad\quad\quad\quad\quad\quad\quad\quad\quad\quad\quad\quad
		+ \max_{s\in[k-1]} \dist{ \maptoReprSA_h^k(x_h^s, a_h^s), \maptoReprSA_h^{s+1}(x, a)}.
	\end{align*}
	 By the construction of the algorithm, $\dist{(x_h^s, a_h^s), \maptoReprSA_h^k(x, a)} \leq \maxdist$ and $\dist{ \maptoReprSA^{k}(x_h^s, a_h^s), \maptoReprSA^{s+1}(x_h^s, a_h^s) }\leq 2\maxdist$, which concludes the proof.
\end{proof}

\begin{flemma}
	\label{lemma:technical-inequality-for-KeRNS-with-minimum}
	Let $\kernsQ_{h}^k$ and $\tQ_{h, \mapzeta}^k$ be the $Q$-functions defined in Algorithm~\ref{alg:RS-kernel_backward_induction-ONLINE}. Then,
	\begin{align*}
	\kernsQ_{h}^k(x, a) & \eqdef \underset{(\bx,\ba) \in\ReprStates_h^{k}\times\ReprActions_h^{k} }{\min}\pa{\tQ_{h, \mapzeta}^k(\bx,\ba)+\lipQ \dist{(x,a), (\bx,\ba)}} \\
	& = \underset{s\in[k-1]}{\min}\pa{\tQ_{h, \mapzeta}^k(x_h^s, a_h^s)+\lipQ \dist{(x,a), \maptoReprSA_h^{k}(x_h^s, a_h^s)}}.
	\end{align*}
\end{flemma}
\begin{proof}
	Notice that, although $\tQ_{h, \mapzeta}^k$ is only computed for the representative state-action pairs, it is defined for any $(x, a)$ as $$\tQ_{h,\mapzeta}^k(x, a) = \kernsR_{h}^k(x, a) + \kernsP_{h}^k \kernsV_{h+1}(x, a) + \kernsBonus_{h}^k(x, a). $$
		
	We claim that 
	\begin{align*}
	\underbrace{\braces{\maptoReprSA_h^{k}(x_h^s,a_h^s): s\in[k-1]}}_{A} = \underbrace{\ReprStates_h^{k}\times\ReprActions_h^{k}}_{B}.
	\end{align*}
	First, $A\subset B$, since $\forall (s,h)$, we have $\maptoReprSA_h^{k}(x_h^s,a_h^s) \in \ReprStates_h^{k}\times\ReprActions_h^{k}$. Second, for any $(\bx, \ba) \in \ReprStates_h^{k}\times\ReprActions_h^{k}$, there exists $(s,h)$ such that $(\bx, \ba) = (x_h^s,a_h^s) = \maptoReprSA_h^{s+1}(x_h^s,a_h^s) = \maptoReprSA_h^{k}(x_h^s,a_h^s) \in A$, which implies that $B \subset A$.

	Together with the fact that $\tQ_{h, \mapzeta}^k(x_h^s, a_h^s) = \tQ_{h, \mapzeta}^k(\maptoReprSA_h^k(x_h^s, a_h^s))$, which holds by Definition~\ref{def:empirical-mdp-rs-kerns} and Algorithm~\ref{alg:RS-kernel_backward_induction-ONLINE}, this concludes the proof. 
\end{proof}

\begin{flemma}
	\label{lemma:diff-between-Q-in-kerns-and-RS-kerns}
	The difference between the $Q$-values computed by \kerns and \RSkerns are bounded as follows
	\begin{align*}
	\sup_{x, a}\abs{ \algQ_h^k(x, a) - \kernsQ_{h}^k(x, a)} \lesssim \pa{\lipQ (\maxdist + \Smaxdist) + \frac{\maxdist}{\ksigma}H}(H-h+1).
	\end{align*}
\end{flemma}
\begin{proof}
	We proceed by induction on $h$. Let
	\begin{align*}
	\epsilon_h \eqdef \pa{\lipQ (\maxdist + \Smaxdist) + \frac{\maxdist}{\ksigma}H}(H-h+1).
	\end{align*}
	For $h=H+1$, $\epsilon_{H+1} = 0$, $\algQ_{H+1}^k = \kernsQ_{H+1}^k=0$ and the claim holds. 
	
	Now, assume that the claim is true for $h+1$. In this case, we have, for any $x$, 
	\begin{align*}
	& 
	\abs{\algV_{h+1}^{k}(x) - \kernsV_{h+1}^{k}(x)}
	\\
	& 
	= \abs{\min \pa{H-h, \max_a\algQ_{h+1}^{k}(x, a)} - \min \pa{H-h, \max_a\kernsQ_{h+1}^{k}(x, a)}}
	\\
	& 
	\leq
	\abs{\max_a\algQ_{h+1}^{k}(x, a) - \max_a\kernsQ_{h+1}^{k}(x, a)}
	\\
	& 
	\leq \max_{a}\abs{\algQ_{h+1}^{k}(x, a) - \kernsQ_{h+1}^{k}(x, a)} \leq \epsilon_{h+1},
	\end{align*}
	where we used induction hypothesis and the fact that, for any $a, b, c \in \RR$, we have $\abs{\min(a, b) -\min(a, c)} \leq \abs{b-c}$ (Fact~\ref{fact:small-useful-result-with-min}).
	
	For any $(x_h^s, a_h^s)$ with $s\in[k-1]$, we have
	\begin{align*}
	& \abs{\tQ_{h}^k(x_h^s, a_h^s) -\tQ_{h, \mapzeta}^k(x_h^s, a_h^s)} \\
	& \leq  \abs{\estR_h^k(x_h^s, a_h^s) - \kernsR_{h}^k(x_h^s, a_h^s)}
	+  \abs{\bonus_h^k(x_h^s, a_h^s) - \kernsBonus_{h}^k(x_h^s, a_h^s)} \\
	&\quad + \abs{\pa{\estP_h^k - \kernsP_{h}^k}\algV_{h+1}^{k}(x_h^s, a_h^s) 
		+ \kernsP_{h}^k\pa{ \algV_{h+1}^{k} - \kernsV_{h+1}^{k}}(x_h^s, a_h^s)} \\
	& \lesssim \epsilon_{h+1} + \frac{H}{\ksigma}\pa{ \dist{(x_h^s,a_h^s),\maptoReprSA_h^k(x_h^s,a_h^s)} +\maxdist  } 
	 +  \lipQ \Smaxdist  +  \frac{\maxdist}{\ksigma} H
	\end{align*}
	where, in the last line, we used the induction hypothesis and lemmas \ref{lemma:error-in-P-when-usin-representative-states_NEW}, \ref{lemma:error-in-R-when-usin-representative-states_NEW} and \ref{lemma:error-in-bonus-when-usin-representative-states_NEW}.

	By the construction of \RSkerns, we have
	$\max_{s'\in[k-1]} \dist{(x_h^{s'},a_h^{s'}),\maptoReprSA_h^k(x_h^{s'},a_h^{s'})} \leq \maxdist$.
	Consequently,
	\begin{align*}
	& \abs{\tQ_{h}^k(x_h^s, a_h^s) -\tQ_{h, \mapzeta}^k(x_h^s, a_h^s)} \lesssim  \epsilon_{h+1} +  \lipQ\Smaxdist +  \frac{\maxdist}{\ksigma}H.
	\end{align*}
	Now, take an arbitrary $(x, a)$. By Lemma \ref{lemma:technical-inequality-for-KeRNS-with-minimum},
	\begin{align*}
	\kernsQ_{h}^k(x, a) & = \underset{(\bx,\ba) \in\ReprStates_h^k\times\ReprActions_h^k }{\min}\pa{\tQ_{h, \mapzeta}^k(\bx,\ba)+\lipQ \dist{(x,a), (\bx,\ba)}} \\
	& = \underset{s\in[k-1]}{\min}\pa{\tQ_{h, \mapzeta}^k(x_h^s, a_h^s)+\lipQ \dist{(x,a), \maptoReprSA_h^k(x_h^s, a_h^s)}} 
	\end{align*}
	and, by definition,
	\begin{align*}
	\algQ_{h}^k(x, a)= \underset{s\in[k-1]}{\min}\pa{\tQ_{h}^k(x_h^s, a_h^s)+\lipQ \dist{(x,a), (x_h^s, a_h^s)}},
	\end{align*}
	we obtain, for any $(x, a)$,
	\begin{align*}
	& \abs{\algQ_{h}^k(x, a) - \kernsQ_{h}^k(x, a)} \\
	&  \lesssim 
	\min_{s\in[k-1]}\abs{ \tQ_{h}^k(x_h^s, a_h^s) -  \tQ_{h, \mapzeta}^k(x_h^s, a_h^s) } 
	+ \lipQ \min_{s\in[k-1]}\abs{ \dist{(x,a), (x_h^s, a_h^s)} - \dist{(x,a), \maptoReprSA_h^k(x_h^s, a_h^s)}} \\
	& \leq \epsilon_{h+1} +  \lipQ\Smaxdist +  \frac{\maxdist}{\ksigma}H + \lipQ \max_{s\in[k-1]}\dist{ (x_h^s, a_h^s), \maptoReprSA_h^k(x_h^s, a_h^s) } \\
	& \leq \epsilon_{h+1} +  \lipQ(\maxdist + \Smaxdist) + \frac{\maxdist}{\ksigma}H = \epsilon_h.
	\end{align*}
	
	which concludes the proof.
\end{proof}


\begin{ftheorem}[UCRL-type regret bound for \RSkerns]
	\label{theorem:regret-RS-kerns-ucrl-type}
	With probability at least $1-\cdelta$, the regret of \RSkerns is bounded as follows
	\begin{align*}
	\regret^{\RSkerns}(K) \lesssim   \regret_1^{\kerns}(K) + \lipQ(\maxdist + \Smaxdist)KH^2 + \frac{\maxdist}{\ksigma}KH^3.
	\end{align*}
	where $\regret_1^{\kerns}(K)$ is the UCRL-type regret bound given in Theorem \ref{theorem:regret-bound-ucrl-type} for \kerns.
\end{ftheorem}
\begin{proof} 
	Let $\kernsPolicy_k$ be the policy followed by \RSkerns in episode $k$ and let  $\kernsDelta_h^k \eqdef \kernsV_{h}^k(x_h^k) - \trueV_{k,h}^{\kernsPolicy_k}(x_h^k) $

	\paragraph{Regret decomposition} Consider the following decomposition, also used in the proof of Lemma \ref{lemma:regret-in-terms-of-sum-of-bonus-ucrl-type}:
	\begin{align*}
	\kernsDelta_{h}^k 
	& =  \kernsV_{h}^k(x_h^k) - \trueV_{k,h}^{\kernsPolicy_k}(x_h^k) \\
	& \leq \kernsQ_{h}^k(x_h^k, a_h^k) - \trueQ_{k,h}^{\kernsPolicy_k}(x_h^k, a_h^k) \\
	& \leq \kernsQ_{h}^k(\tx_h^k, \ta_h^k) - \trueQ_{k,h}^{\kernsPolicy_k}(x_h^k, a_h^k)
	+ \lipQ\dist{(\tx_h^k, \ta_h^k), (x_h^k, a_h^k)}
	,\quad\text{since $\kernsQ_{h}^k$ is $\lipQ$-\lipschitz} \\
	& \leq  \tQ_{h,\mapzeta}^k(\tx_h^k, \ta_h^k) - \trueQ_{k,h}^{\kernsPolicy_k}(x_h^k, a_h^k)
	+ \lipQ\dist{(\tx_h^k, \ta_h^k), (x_h^k, a_h^k)}
	,\quad\text{since $\kernsQ_{h}^k(\tx_h^k, \ta_h^k) \leq  \tQ_{h,\mapzeta}^k(\tx_h^k, \ta_h^k)$} \\
	& = \kernsR_{h}^k(\tx_h^k, \ta_h^k) - \trueR_h^k(x_h^k, a_h^k) 
	+  \kernsP_{h}^k\algV_{h+1}^k(\tx_h^k, \ta_h^k) - \trueP_h^k\trueV_{k,h+1}^{\kernsPolicy_k}(x_h^k, a_h^k)
	+\kernsBonus_{h}^k(\tx_h^k, \ta_h^k) 
	+ \lipQ\dist{(\tx_h^k, \ta_h^k), (x_h^k, a_h^k)} \\
	& = \quad 
	\underbrace { 
		\kernsR_{h}^k(\tx_h^k, \ta_h^k) - \trueR_h^k(x_h^k, a_h^k) 
	}_{\termA} 
	+ 
	\underbrace{
		\sqrbrackets{\kernsP_{h}^k -\trueP_h^k}\trueV_{k,h+1}^* (\tx_h^k, \ta_h^k)
	}_{\termB}
	+ 
	\underbrace{
		\sqrbrackets{\kernsP_{h}^k -\trueP_h^k}\pa{ \kernsV_{h+1}^k- \trueV_{k,h+1}^*}(\tx_h^k, \ta_h^k) 
	}_{\termC} \\
	&  + \underbrace{ 
		\trueP_h^k \kernsV_{h+1}^k(\tx_h^k, \ta_h^k) - \trueP_h^k\trueV_{k,h+1}^{\kernsPolicy_k}(x_h^k, a_h^k)
	}_{\termD}
	+\kernsBonus_{h}^k(\tx_h^k, \ta_h^k) 
	+ 2\lipQ\dist{(\tx_h^k, \ta_h^k), (x_h^k, a_h^k)}.
	\end{align*}
	
	We will use the following results, which are a consequence of Lemmas  \ref{lemma:error-in-P-when-usin-representative-states_NEW}, \ref{lemma:error-in-R-when-usin-representative-states_NEW} and \ref{lemma:error-in-bonus-when-usin-representative-states_NEW} and the fact that $ \dist{(\tx_h^k, \ta_h^k),\maptoReprSA(\tx_h^k, \ta_h^k)} \leq \maxdist$:
	\begin{align*}
	\abs{\estR_h^k(\tx_h^k, \ta_h^k) - \kernsR_{h}^k(\tx_h^k, \ta_h^k)} 
	& \lesssim  \frac{4}{\ksigma}\pa{ \dist{(\tx_h^k, \ta_h^k),\maptoReprSA_h^k(\tx_h^k, \ta_h^k)} + \maxdist } + \frac{8\maxdist}{\ksigma}\lesssim \frac{\maxdist}{\ksigma}.
	\end{align*}
	\begin{align*}
	\kernsBonus_{h}^k(\tx_h^k, \ta_h^k) & = \bonus_h^k(\tx_h^k, \ta_h^k) + \kernsBonus_{h}^k(\tx_h^k, \ta_h^k)  - \bonus_h^k(\tx_h^k, \ta_h^k)  \\
	& \lesssim \bonus_h^k(\tx_h^k, \ta_h^k) + \frac{H}{\ksigma}\pa{ \dist{(\tx_h^k, \ta_h^k),\maptoReprSA_h^k(\tx_h^k, \ta_h^k)} + \maxdist} + \frac{H\maxdist}{\ksigma} \\
	& \lesssim \bonus_h^k(\tx_h^k, \ta_h^k) +\frac{\maxdist}{\ksigma}H.
	\end{align*}
	and, for any function $f$ that is $\lipQ$-\lipschitz and bounded by $H$,
	\begin{align*}
	& \abs{\pa{\estP_h^k - \kernsP_{h}^k}f(\tx_h^k, \ta_h^k)} 
	 \lesssim \lipQ \Smaxdist 
	+ \frac{4H}{\ksigma}\pa{ \dist{(\tx_h^k, \ta_h^k),\maptoReprSA(\tx_h^k, \ta_h^k)} + \maxdist } + \frac{H\maxdist}{\ksigma}
	\lesssim \lipQ\Smaxdist +\frac{\maxdist}{\ksigma}H.
	\end{align*}
	
	Now, we bound each term $\termA-\termD$.
	
	\underline{Term $\termA$}: by Lemma \ref{lemma:error-in-R-when-usin-representative-states_NEW}, the definition of $\goodevent$ and Corollary \ref{corollary:bias-between-avmdp-and-true-mdp}, we have
	\begin{align*}
	\termA & = \kernsR_{h}^k(\tx_h^k, \ta_h^k) - \estR_h^k(\tx_h^k, \ta_h^k) + \estR_h^k(\tx_h^k, \ta_h^k) - \trueR_h^k(\tx_h^k, \ta_h^k) + \trueR_h^k(\tx_h^k, \ta_h^k) - \trueR_h^k(x_h^k, a_h^k) \\
	& \lesssim \frac{\maxdist}{\ksigma} + \estR_h^k(\tx_h^k, \ta_h^k) - \trueR_h^k(\tx_h^k, \ta_h^k) + \trueR_h^k(\tx_h^k, \ta_h^k) - \trueR_h^k(x_h^k, a_h^k) \\
	& \leq \frac{\maxdist}{\ksigma} +  \rbonus_h^k(\tx_h^k, \ta_h^k) + \biasR(k, h) +\lipR \dist{(\tx_h^k, \ta_h^k), (x_h^k, a_h^k)}.
	\end{align*}

	\underline{Term $\termB$}:
	\begin{align*}
	\termB & = \sqrbrackets{\kernsP_{h}^k -\estP_h^k}\trueV_{k,h+1}^* (\tx_h^k, \ta_h^k) + \sqrbrackets{\estP_h^k -\trueP_h^k}\trueV_{k,h+1}^* (\tx_h^k, \ta_h^k)\\ 
	& \lesssim \lipQ\Smaxdist +\frac{\maxdist}{\ksigma}H + \sqrbrackets{\estP_h^k -\avP_h^k}\trueV_{k,h+1}^* (\tx_h^k, \ta_h^k) + 
	\sqrbrackets{\avP_h^k -\trueP_h^k}\trueV_{k,h+1}^* (\tx_h^k, \ta_h^k) \\
	& \leq \lipQ\Smaxdist +\frac{\maxdist}{\ksigma}H + \pbonus_h^k(\tx_h^k, \ta_h^k) + \biasP(k, h)
	\end{align*}

	\underline{Term $\termC$}: Using Corollary \ref{corollary:bias-between-avmdp-and-true-mdp}, we obtain
	\begin{align*}
	\termC &= \sqrbrackets{\kernsP_{h}^k -\trueP_h^k}\pa{ \kernsV_{h+1}^k- \trueV_{k,h+1}^*}(\tx_h^k, \ta_h^k) \\
	&= \sqrbrackets{\kernsP_{h}^k -\estP_h^k}\pa{ \kernsV_{h+1}^k- \trueV_{k,h+1}^*}(\tx_h^k, \ta_h^k) + \sqrbrackets{\estP_h^k -\trueP_h^k}\pa{ \kernsV_{h+1}^k- \trueV_{k,h+1}^*}(\tx_h^k, \ta_h^k)  \\
	& \lesssim  \lipQ\Smaxdist +\frac{\maxdist}{\ksigma}H +  \sqrbrackets{\estP_h^k -\avP_h^k}\pa{ \kernsV_{h+1}^k- \trueV_{k,h+1}^*}(\tx_h^k, \ta_h^k) + \biasP(k, h) \\
	& \lesssim \lipQ\Smaxdist +\frac{\maxdist}{\ksigma}H +  \sqrt{  \frac{ H^2 \Xsigmacov }{\gencount_h^k(\tx_h^k, \ta_h^k)} }
	+ \frac{\kbeta H}{\gencount_h^k(\tx_h^k, \ta_h^k)}  + \lipQ \ksigma + \biasP(k, h)
	\end{align*}
	by the definition of $\goodevent$.

	\underline{Term $\termD$}: From Assumption \ref{assumption:lipschitz-rewards-and-transitions}, for any $\lipQ$-\lipschitz function, the mapping $(x, a) \mapsto \trueP_h^k f(x, a)$ is $\lipP\lipQ$-\lipschitz. Consequently,
	\begin{align*}
	\termD & =  \trueP_h^k \kernsV_{h+1}^k(\tx_h^k, \ta_h^k) - \trueP_h^k\trueV_{k,h+1}^{\kernsPolicy_k}(x_h^k, a_h^k)\\
	& \leq \trueP_h^k \kernsV_{h+1}^k(x_h^k, a_h^k) - \trueP_h^k\trueV_{k,h+1}^{\kernsPolicy_k}(x_h^k, a_h^k) + \lipP\lipQ\dist{(x_h^k, a_h^k), (\tx_h^k, \ta_h^k)}\\
	& = \trueP_h^k\pa{ \kernsV_{h+1}^k - \trueV_{k,h+1}^{\kernsPolicy_k}}(x_h^k, a_h^k) + \lipP\lipQ\dist{(x_h^k, a_h^k), (\tx_h^k, \ta_h^k)} \\
	& = \kernsDelta_{h+1}^k + \xi_{h+1}^k + \lipP\lipQ\dist{(x_h^k, a_h^k), (\tx_h^k, \ta_h^k)}.
	\end{align*}
	where 
	\begin{align*}
	\xi_{h+1}^k \eqdef  \trueP_h^k\pa{ \kernsV_{h+1}^k - \trueV_{k,h+1}^{\kernsPolicy_k}}(x_h^k, a_h^k) - \kernsDelta_{h+1}^k
	\end{align*}
	is a martingale difference sequence with respect to $(\cF_h^k)_{k, h}$ bounded by $4H$. 
	
	Putting together the bounds for $\termA$-$\termD$ and using the definition of the bonuses $\bonus_h^k$, we obtain
	\begin{align*}
	\kernsDelta_{h}^k \lesssim
	& \kernsDelta_{h+1}^k + \xi_{h+1}^k 
	+ \lipQ\dist{(x_h^k, a_h^k), (\tx_h^k, \ta_h^k)}
	+ \sqrt{  \frac{ H^2 \Xsigmacov }{\gencount_h^k(\tx_h^k, \ta_h^k)} } \\
	& + \frac{\kbeta H}{\gencount_h^k(\tx_h^k, \ta_h^k)} 
	+ \lipQ\ksigma  
	+ \bias(k, h')
	+ \lipQ\Smaxdist +\frac{\maxdist}{\ksigma}H
	\end{align*}
	where the constant in front of $\kernsDelta_{h+1}^k$ is \emph{exact} (\ie not omitted by $\lesssim$). 
	
	Now, we follow the same arguments as in the proof of Lemma \ref{lemma:regret-in-terms-of-sum-of-bonus-ucrl-type}. Consider the event $E_h^k \eqdef \braces{\dist{(x_h^k, a_h^k), (\tx_h^k, \ta_h^k)} \leq  2 \ksigma}$. The inequality above gives us
	\begin{align}
	\indic{E_h^k}\kernsDelta_{h,}^k \lesssim
	&  \indic{E_h^k} \kernsDelta_{h+1}^k +  \indic{E_h^k}\xi_{h+1}^k 
	+  \indic{E_h^k}\sqrt{  \frac{ H^2 \Xsigmacov }{\gencount_h^k(\tx_h^k, \ta_h^k)} } \nonumber \\
	& +  \indic{E_h^k}\frac{\kbeta H}{\gencount_h^k(\tx_h^k, \ta_h^k)} 
	+ 3\lipQ\ksigma  
	+ \bias(k, h')
	+ \lipQ\Smaxdist +\frac{\maxdist}{\ksigma}H 		\label{eq:rs-kerns-delta-trick}.
	\end{align}
	
	Using Lemmas \ref{lemma:upper-bound-on-q-functions} and \ref{lemma:diff-between-Q-in-kerns-and-RS-kerns}, we upper bound $ \indic{E_h^k} \kernsDelta_{h+1}^k$ in terms of $\kernsDelta_{h+1}^k$:
	\begin{align}
	\indic{E_h^k} \kernsDelta_{h+1}^k
	&
	= \indic{E_h^k}\pa{ \kernsV_{h+1}^k(x_{h+1}^k) - \trueV_{k,h}^{\kernsPolicy_k}(x_{h+1}^k)} 
	\nonumber
	\\
	&
	\lesssim \indic{E_h^k}\pa{ 
		H\lipQ (\maxdist + \Smaxdist) + \frac{\maxdist}{\ksigma}H^2
		+ \algV_{h+1}^k(x_{h+1}^k) - \trueV_{k,h+1}^{\kernsPolicy_k}(x_{h+1}^k)
	} 
	\nonumber
	\\
	& 
	\leq H\lipQ (\maxdist + \Smaxdist) + \frac{\maxdist}{\ksigma}H^2
	+ \indic{E_h^k}
	\underbrace{
		\pa{ 
			\algV_{h+1}^k(x_{h+1}^k) + \sum_{h'=h+1}^H \bias(k, h)  - \trueV_{k,h+1}^{\kernsPolicy_k}(x_{h+1}^k)
		} 
	}_{\geq 0 \text{ by Lemma \ref{lemma:upper-bound-on-q-functions}}}
	\nonumber
	\\
	& 
	\leq 
	H\lipQ (\maxdist + \Smaxdist) + \frac{\maxdist}{\ksigma}H^2
	+ \algV_{h+1}^k(x_{h+1}^k) -  \trueV_{k,h+1}^{\kernsPolicy_k}(x_{h+1}^k)
	+ \sum_{h'=h+1}^H \bias(k, h)
	\nonumber
	\\
	& 
	\leq 
	2H\lipQ (\maxdist + \Smaxdist) + 2\frac{\maxdist}{\ksigma}H^2
	+ \kernsV_{h+1}^k(x_{h+1}^k) -  \trueV_{k,h+1}^{\kernsPolicy_k}(x_{h+1}^k)
	+ \sum_{h'=h+1}^H \bias(k, h) 
	\nonumber
	\\
	&
	= \kernsDelta_{h+1}^k        
	+ 2H\lipQ (\maxdist + \Smaxdist) + 2\frac{\maxdist}{\ksigma}H^2
	+ \sum_{h'=h+1}^H \bias(k, h)
	\label{eq:rs-kerns-delta-trick-two}
	\end{align}
	
	Let $\overline{E}_h^k$ be the complement of $E_h^k$. Using the inequality above combined with \eqref{eq:rs-kerns-delta-trick}, and the fact that $\kernsDelta_{h}^k \leq H$, we obtain
	\begin{align}
	\kernsDelta_{h}^k 
	& = \indic{\overline{E}_h^k}\kernsDelta_{h}^k  + \indic{E_h^k}\kernsDelta_{h}^k 
	\label{eq:rs-kerns-delta-trick-three}
	\\
	& 
	\leq  H \indic{\overline{E}_h^k} + \indic{E_h^k}\kernsDelta_{h}^k 
	\nonumber
	\\
	& 
	\lesssim  
	H \indic{\overline{E}_h^k}
	+ \kernsDelta_{h+1}^k        
	+ H\lipQ (\maxdist + \Smaxdist) + \frac{\maxdist}{\ksigma}H^2
	+ \sum_{h'=h}^H \bias(k, h)
	+  \indic{E_h^k}\xi_{h+1}^k
	\nonumber 
	\\
	& 
	+  \indic{E_h^k}\sqrt{  \frac{ H^2 \Xsigmacov }{\gencount_h^k(\tx_h^k, \ta_h^k)} } 
	+  \indic{E_h^k}\frac{\kbeta H}{\gencount_h^k(\tx_h^k, \ta_h^k)} 
	+ \lipQ\ksigma  
	\nonumber
	\end{align}

	Consequently, 
	\begin{align*}
	\sum_{k=1}^K \kernsDelta_{1}^k \lesssim & 
	\sum_{k=1}^K\sum_{h=1}^H \pa{\frac{H\sqrt{\Xsigmacov}}{\sqrt{\gencount_h^k(\tx_h^k, \ta_h^k)}} 
		+  \frac{\kbeta H}{\gencount_h^k(\tx_h^k, \ta_h^k)} }\indic{E_h^k}
	+ H \sum_{k=1}^K\sum_{h=1}^H \indic{\overline{E}_h^k}  \\
	& +  \sum_{k=1}^K\sum_{h=1}^H \indic{E_h^k} \xi_{h+1}^k 
	+  H \sum_{k=1}^K\sum_{h=1}^H \bias(k, h) 
	+ \lipQ (\maxdist + \Smaxdist)KH^2 + \frac{\maxdist}{\ksigma}KH^3
	\end{align*}

	Now, as in the proof of Lemma  \ref{lemma:regret-in-terms-of-sum-of-bonus-ucrl-type}, we show that 
	\begin{align*}
	H \sum_{k=1}^K\sum_{h=1}^H \indic{\overline{E}_h^k} \leq H^2\sigmacov
	\end{align*}
	
	and, from lemmas \ref{lemma:upper-bound-on-q-functions} and \ref{lemma:diff-between-Q-in-kerns-and-RS-kerns}, we have:
	\begin{align*}
	\regret^{\RSkerns}(K) & = \sum_{k=1}^{K} \pa{\trueV_{k,1}^*(x_1^k) - \trueV_{k,h}^{\kernsPolicy_k}(x_1^k)} \\
	& \leq \sum_{k=1}^{K}\pa{ \algV_{1}^k(x_1^k)  - \trueV_{k,h}^{\kernsPolicy_k}(x_1^k)} + \sum_{k=1}^{K}\sum_{h=1}^H \bias(k, h) \\
	& \lesssim  \sum_{k=1}^{K}\pa{ \kernsV_{1}^k(x_1^k)  - \trueV_{k,h}^{\kernsPolicy_k}(x_1^k)} + \sum_{k=1}^{K}\sum_{h=1}^H \bias(k, h) + \lipQ KH(\maxdist + \Smaxdist) + \frac{\maxdist}{\ksigma}KH^2 \\
	& =  \sum_{k=1}^K \kernsDelta_{1}^k + \sum_{k=1}^{K}\sum_{h=1}^H \bias(k, h) + \lipQ KH(\maxdist + \Smaxdist) + \frac{\maxdist}{\ksigma}KH^2 \\
	& \lesssim \sum_{k=1}^K\sum_{h=1}^H \pa{\frac{H\sqrt{\Xsigmacov}}{\sqrt{\gencount_h^k(\tx_h^k, \ta_h^k)}} 
		+  \frac{\kbeta H}{\gencount_h^k(\tx_h^k, \ta_h^k)} }\indic{E_h^k}
	+ H^2\sigmacov  \\
	& +  \sum_{k=1}^K\sum_{h=1}^H \indic{E_h^k} \xi_{h+1}^k 
	+  H \sum_{k=1}^K\sum_{h=1}^H \bias(k, h) 
	+ \lipQ (\maxdist + \Smaxdist)KH^2 + \frac{\maxdist}{\ksigma}KH^3.
	\end{align*}
	
	Recall the definition $E_h^k \eqdef \braces{\dist{(x_h^k, a_h^k), (\tx_h^k, \ta_h^k)} \leq  2 \ksigma}$ and the fact that $\widetilde{\xi}_{h+1}^k\eqdef \indic{E_h^k} \xi_{h+1}^k$ is a martingale difference sequence with respect to $(\cF_h^k)_{k, h}$ bounded by $4H$. 
	Then, as in the proof of Theorem \ref{theorem:regret-bound-ucrl-type}, we obtain
	
	\begin{align*}
	\regret^{\RSkerns}(K) \lesssim & \regret_1^{\kerns}(K) +  \lipQ KH^2 (\maxdist + \Smaxdist) +\frac{\maxdist}{\ksigma} KH^3
	\end{align*}
	with probability at least $1-\cdelta$.
\end{proof}


\begin{ftheorem}[UCBVI-type regret bound for \RSkerns]
	\label{theorem:regret-RS-kerns-ucbvi-type}
	With probability at least $1-\cdelta$, the regret of \RSkerns is bounded as follows
	\begin{align*}
	\regret^{\RSkerns}(K) \lesssim   \regret_2^{\kerns}(K) + \lipQ(\maxdist + \Smaxdist)KH^2 + \frac{\maxdist}{\ksigma}KH^3.
	\end{align*}
	where $\regret_2^{\kerns}(K)$ is the UCBVI-type regret bound given in Theorem \ref{theorem:regret-bound-ucbvi-type} for \kerns.
\end{ftheorem}

\begin{proof}
	We use the same regret decomposition as in the proof of Theorem \ref{theorem:regret-RS-kerns-ucrl-type}, but the term $\termC$ is bounded differently (as in Lemma \ref{lemma:regret-in-terms-of-sum-of-bonus}).
	
	Using Lemma \ref{lemma:error-in-P-when-usin-representative-states_NEW}, Lemma \ref{lemma:diff-between-Q-in-kerns-and-RS-kerns}, and the same arguments as in the proof of Lemma \ref{lemma:regret-in-terms-of-sum-of-bonus}, we have
	\begin{align*}
	\termC 
	& = \sqrbrackets{\kernsP_h^k -\trueP_h^k}\pa{ \kernsV_{h+1}^k- \trueV_{k,h+1}^*}(\tx_h^k, \ta_h^k) \\
	& = \sqrbrackets{\kernsP_h^k -\estP_h^k}\pa{ \kernsV_{h+1}^k- \trueV_{k, h+1}^*}(\tx_h^k, \ta_h^k) 
	+ \sqrbrackets{\estP_h^k -\trueP_h^k}\pa{ \kernsV_{h+1}^k- \trueV_{k,h+1}^*}(\tx_h^k, \ta_h^k) 
	\\
	& \lesssim  \lipQ (\maxdist + \Smaxdist)H + \frac{\maxdist}{\ksigma}H^2 + \sqrbrackets{\estP_h^k -\trueP_h^k}\pa{ \algV_{h+1}^k- \trueV_{k,h+1}^*}(\tx_h^k, \ta_h^k) 
	\quad \text{(by lemmas \ref{lemma:error-in-P-when-usin-representative-states_NEW} and \ref{lemma:diff-between-Q-in-kerns-and-RS-kerns})}
	\\
	& \lesssim   \lipQ (\maxdist + \Smaxdist)H + \frac{\maxdist}{\ksigma}H^2 +  \frac{1}{H} \trueP_h^k\pa{ \algV_{h+1}^k- \trueV_{k,h+1}^*}(x_h^k, a_h^k) + \frac{H^2\Xsigmacov}{\gencount_h^k(\tx_h^k, \ta_h^k)} \\
	& + \lipQ\ksigma +  \sum_{h'=h}^H \bias(k, h') + \lipQ\dist{(x_h^k, a_h^k), (\tx_h^k, \ta_h^k)}
	\quad \text{(following the proof of Lemma \ref{lemma:regret-in-terms-of-sum-of-bonus})}
	\\
	& \lesssim   \lipQ (\maxdist + \Smaxdist)H + \frac{\maxdist}{\ksigma}H^2 +  \frac{1}{H} \trueP_h^k\pa{ \kernsV_{h+1}^k- \trueV_{k,h+1}^{\kernsPolicy_k}}(x_h^k, a_h^k) + \frac{H^2\Xsigmacov}{\gencount_h^k(\tx_h^k, \ta_h^k)} \\
	& + \lipQ\ksigma +  \sum_{h'=h}^H \bias(k, h') + \lipQ\dist{(x_h^k, a_h^k), (\tx_h^k, \ta_h^k)}
	\end{align*}
	where, in the last line, we used the fact that $\trueV_{k,h+1}^{\kernsPolicy_k} \leq \trueV_{k,h+1}^*$ and that $\algV_{k+1}^k
 \leq \kernsV_{k+1}^k + \lipQ(\maxdist+\Smaxdist)H + (\maxdist/\ksigma)H^2$ by Lemma \ref{lemma:diff-between-Q-in-kerns-and-RS-kerns}.
	
	Putting together the bounds for $\termA-\termD$ and using the same arguments as in the proof of Theorem \ref{theorem:regret-RS-kerns-ucrl-type}, especially the inequalities \eqref{eq:rs-kerns-delta-trick}, \eqref{eq:rs-kerns-delta-trick-two} and \eqref{eq:rs-kerns-delta-trick-three}, we obtain
	
	\begin{align*}
	\sum_{k=1}^K \kernsDelta_{1}^k \lesssim & 
	\sum_{k=1}^K\sum_{h=1}^H \pa{\frac{H}{\sqrt{\gencount_h^k(\tx_h^k, \ta_h^k)}} 
		+  \frac{H^2\Xsigmacov}{\gencount_h^k(\tx_h^k, \ta_h^k)} }\indic{\dist{(x_h^k, a_h^k), (\tx_h^k, \ta_h^k)} \leq  2 \ksigma}
	+ H^2 \sigmacov \\
	& +  \sum_{k=1}^K\sum_{h=1}^H \pa{1+\frac{1}{H}}^{H-h+1}\widetilde{\xi}_{h+1}^k 
	+  H \sum_{k=1}^K\sum_{h=1}^H \bias(k, h) 
	+ \lipQ K H\ksigma 
	+ \lipQ (\maxdist + \Smaxdist)H^2 + \frac{\maxdist}{\ksigma}H^3
	\end{align*}
	
	where $\widetilde{\xi}_{h+1}^k$	is a martingale difference sequence with respect to $(\cF_h^k)_{k, h}$ bounded by $4H$.

	Now, from lemmas \ref{lemma:upper-bound-on-q-functions} and \ref{lemma:diff-between-Q-in-kerns-and-RS-kerns} , we have:
	\begin{align*}
	\regret^{\RSkerns}(K) & = \sum_{k=1}^{K} \pa{\trueV_{k,1}^*(x_1^k) - \trueV_{k,h}^{\kernsPolicy_k}(x_1^k)} 
	\leq \sum_{k=1}^{K}\pa{ \algV_{1}^k(x_1^k)  - \trueV_{k,h}^{\kernsPolicy_k}(x_1^k)} + \sum_{k=1}^{K}\sum_{h=1}^H \bias(k, h) \\
	& \lesssim  \sum_{k=1}^{K}\pa{ \kernsV_{1}^k(x_1^k)  - \trueV_{k,h}^{\kernsPolicy_k}(x_1^k)} + \sum_{k=1}^{K}\sum_{h=1}^H \bias(k, h) + \lipQ KH(\maxdist + \Smaxdist) + \frac{\maxdist}{\ksigma}KH^2 \\
	& =  \sum_{k=1}^K \kernsDelta_{1}^k + \sum_{k=1}^{K}\sum_{h=1}^H \bias(k, h) + \lipQ KH(\maxdist + \Smaxdist) + \frac{\maxdist}{\ksigma}KH^2 \\
	& \lesssim 		\sum_{k=1}^K\sum_{h=1}^H \pa{\frac{H}{\sqrt{\gencount_h^k(\tx_h^k, \ta_h^k)}} 
		+  \frac{H^2\Xsigmacov}{\gencount_h^k(\tx_h^k, \ta_h^k)} }\indic{\dist{(x_h^k, a_h^k), (\tx_h^k, \ta_h^k)} \leq  2 \ksigma}
	+ H^2 \sigmacov \\
	& +  \sum_{k=1}^K\sum_{h=1}^H \pa{1+\frac{1}{H}}^{h}\widetilde{\xi}_{h+1}^k 
	+  H \sum_{k=1}^K\sum_{h=1}^H \bias(k, h) 
	+ \lipQ K H\ksigma 
	+ \lipQ (\maxdist + \Smaxdist)KH^2 + \frac{\maxdist}{\ksigma}KH^3
	\end{align*}
	
	Then, as in the proof of Theorem \ref{theorem:regret-bound-ucbvi-type}, we obtain
	
	\begin{align*}
	\regret^{\RSkerns}(K) \lesssim & \regret_2^{\kerns}(K) +  \lipQ (\maxdist + \Smaxdist)KH^2  +\frac{\maxdist}{\ksigma} KH^3
	\end{align*}
	with probability at least $1-\cdelta$.
\end{proof}

\newpage
\section{Technical Lemmas}

\begin{lemma}[adapted from \cite{domingues2020}]
	\label{lemma:kernel-bias}
	Consider a sequence of non-negative real numbers $\braces{z_s}_{s=1}^t$ and let $\kernel_{(\keta, \kW)}: \RR_+ \to [0, 1]$ satisfy Assumption \ref{assumption:kernel-properties}.  For a given $t$, let
	\begin{align*}
	w_s \eqdef \kernel_{(\keta, \kW)}\pa{t-s-1, \frac{z_s}{\ksigma}}  \; \mbox{ and } \; \widetilde{w}_s \eqdef \frac{w_s}{\kbeta + \sum_{s'=1}^t w_{s'}}
	\end{align*}
	for $\kbeta > 0$. Then, we have
	\begin{align*}
	\sum_{s=1}^t \widetilde{w}_s z_s \leq 2\ksigma \pa{1 + \sqrt{\log (\spacekernelconstA t/\kbeta+e)}}.
	\end{align*}
\end{lemma}
\begin{proof}
	For completeness, we reproduce here the proof of Lemma 7 of \cite{domingues2020}, which also applies to our setting.
	We split the sum into two terms:
	\begin{align*}
	\sum_{s=1}^t \widetilde{w}_s z_s &= \sum_{s: z_s < c} \widetilde{w}_s z_s + \sum_{s: z_s \geq c} \widetilde{w}_s z_s  \leq c + \sum_{s: z_s \geq c} \widetilde{w}_s \,.
	\end{align*}
	From Assumption \ref{assumption:kernel-properties}, we have $w_s \leq \spacekernelconstA \exp\pa{-z_s^2/(2\ksigma^2)}$. Hence, $ \widetilde{w}_s \leq (\spacekernelconstA/\kbeta) \exp\pa{-z_s^2/(2\ksigma^2)}$, since $\kbeta + \sum_{s'=1}^t w_{s'} \geq \kbeta$.

	We want to find $c$ such that:
	\begin{align*}
	z_s \geq c \implies \frac{\spacekernelconstA}{\kbeta}\exp\pa{-\frac{z_s^2}{2\ksigma^2}} \leq \frac{1}{t}\frac{2\ksigma^2}{z_s^2}
	\end{align*}
	which implies, for $z_s \geq c$, that $\widetilde{w}_s \leq \frac{1}{t}\frac{2\ksigma^2}{z_s^2}$.

	Let $ x = z_s^2/2\ksigma^2 $. Reformulating, we want to find a value $c'$ such that $\spacekernelconstA\exp(-x) \leq \kbeta/(xt)$ for all $x \geq c'$. Let $c' = 2\log (\spacekernelconstA t/\kbeta+e)$. If $x \geq c'$, we have:
	\begin{align*}
	 \frac{x}{2} \geq \log \!\!\left(\frac{\spacekernelconstA t}{\kbeta}+e\right)  \implies x \geq \frac{x}{2} + \log\left( \frac{\spacekernelconstA t}{\kbeta}+e\right)
	&\implies x \geq \log x + \log (\spacekernelconstA t/\kbeta+e)\\
	&\implies (\spacekernelconstA/\kbeta)\exp(-x) \leq 1/(xt)
	\end{align*}
	as we wanted.

	Now, $x \geq c'$ is equivalent to $z_s \geq \sqrt{2\ksigma^2 c'} = 2\ksigma\sqrt{\log (\spacekernelconstA t/\kbeta+e)}$. Therefore, we take $c = 2\ksigma\sqrt{\log (\spacekernelconstA t/\kbeta+e)}$, which gives us
	\begin{align*}
	\sum_{s: z_s \geq c} \widetilde{w}_s z_s & \leq \sum_{s: z_s \geq c} \frac{1}{t}\frac{2\ksigma^2}{z_s^2}z_s \leq \frac{2\ksigma^2}{t}\sum_{s: z_s \geq c} \frac{1}{z_s} \leq \frac{2\ksigma^2}{c} \frac{\abs{\braces{s: z_s \geq c}}}{t} \leq \frac{2\ksigma^2}{c}\,.
	\end{align*}

	Finally, we obtain:
	\begin{align*}
	\sum_{s=1}^t \widetilde{w}_s z_s &\leq c + \sum_{s: z_s \geq c} \widetilde{w}_s z_s  \leq c + \frac{2\ksigma^2}{c} \\
	&=  2\ksigma\sqrt{\log (\spacekernelconstA t/\kbeta+e)} + \frac{\ksigma}{\sqrt{\log (\spacekernelconstA t/\kbeta+e)}} \leq 2\ksigma\pa{1 + \sqrt{\log (\spacekernelconstA t/\kbeta)}}\,.
	\end{align*}
\end{proof}

\begin{lemma}
	\label{lemma:weighted-average-and-bonuses-are-lipschitz}
	Let $\kernel_{(\keta, \kW)}: \RR_+ \to [0, 1]$ be a kernel that satisfies Assumption \ref{assumption:kernel-properties}. Let $a \in \RR_+^t$ and $f_1, f_2, f_3$ be functions from  $\RR_+^t $ to $\RR$ defined as
	\begin{align*}
	 & f_1(z) = \frac{\sum_{s=1}^t \kernel_{(\keta, \kW)}\pa{t-s-1, z_s/\ksigma} a_s}{\kbeta + \sum_{s=1}^t \kernel_{(\keta, \kW)}\pa{t-s-1, z_s/\ksigma}}, \\
	 & f_2(z) = \sqrt{\frac{1}{\kbeta + \sum_{s=1}^t \kernel_{(\keta, \kW)}\pa{t-s-1, z_s/\ksigma}}},\\
	 &  f_3(z) = \frac{1}{\kbeta + \sum_{s=1}^t \kernel_{(\keta, \kW)}\pa{t-s-1, z_s/\ksigma}}
	\end{align*}
	Then, for any $y, z \in \RR_+ $, we have
	\begin{align*}
	 & \abs{f_1(z)-f_1(y)} \leq \frac{2\spacekernelconstB \norm{a}_\infty t}{\kbeta\ksigma} \norm{z-y}_\infty \\
	 & \abs{f_2(z)-f_2(y)} \leq \frac{\spacekernelconstB t}{2 \kbeta^{3/2}\ksigma} \norm{z-y}_\infty \\
	 & \abs{f_3(z)-f_3(y)} \leq \frac{\spacekernelconstB  t}{\kbeta^2 \ksigma} \norm{z-y}_\infty
	\end{align*}
\end{lemma}
\begin{proof}
	From Assumption \ref{assumption:kernel-properties}, the function $z \mapsto \kernel_{(\keta, \kW)}(t-s-1, z)$ is $\spacekernelconstB$-\lipschitz, which yields
	\begin{align*}
		& \abs{f_1(z)-f_1(y)} \\
		& \leq \abs{
		    \frac{ \sum_{s=1}^t \pa{\kernel_{(\keta, \kW)}\pa{t-s-1, z_s/\ksigma}-\kernel_{(\keta, \kW)}\pa{t-s-1, y_s/\ksigma}} a_s  }{\kbeta + \sum_{s=1}^t \kernel_{(\keta, \kW)}\pa{t-s-1, z_s/\ksigma} }
			 } \\
		&  	+\abs{\frac{ \sum_{s=1}^t \kernel_{(\keta, \kW)}\pa{t-s-1, y_s/\ksigma} a_s  }{\kbeta + \sum_{s=1}^t \kernel_{(\keta, \kW)}\pa{t-s-1, y_s/\ksigma}}}
			 \abs{ \frac{ \sum_{s=1}^t \pa{\kernel_{(\keta, \kW)}\pa{t-s-1, z_s/\ksigma}-\kernel_{(\keta, \kW)}\pa{t-s-1, y_s/\ksigma}} }{ \kbeta + \sum_{s=1}^t \kernel_{(\keta, \kW)}\pa{t-s-1, z_s/\ksigma} } } \\
		& \leq
			\frac{ \spacekernelconstB \sum_{s=1}^t (1/\ksigma)\abs{z_s-y_s }a_s  }{\kbeta + \sum_{s=1}^t \kernel_{(\keta, \kW)}\pa{t-s-1, z_s/\ksigma} }
			+  \norm{a}_\infty \frac{ \spacekernelconstB \sum_{s=1}^t (1/\ksigma)\abs{z_s-y_s } }{ \kbeta + \sum_{s=1}^t \kernel_{(\keta, \kW)}\pa{t-s-1, z_s/\ksigma} } \\
		& \leq \frac{2\spacekernelconstB \norm{a}_\infty t}{\kbeta\ksigma} \norm{z-y}_\infty.
	\end{align*}
	The proofs for $f_2$ and $f_3$ are analogous. For $f_2$, we also use the fact that the function $x \mapsto (1/\sqrt{\kbeta+x})$ is $1/(2\beta^{3/2})$-\lipschitz.
\end{proof}

\begin{lemma}
	\label{lemma:gaussian-case-weighted-average-and-bonuses-are-lipschitz}
	 Let $a \in \RR_+^t$ and $f_1, f_2, f_3$ be functions from  $\RR_+^t $ to $\RR$ defined as
	\begin{align*}
	& f_1(z) = \frac{\sum_{s=1}^t g(s) \exp\pa{ -z_s^2/ (2\ksigma^2)}  a_s}{\kbeta + \sum_{s=1}^t g(s) \exp\pa{ -z_s^2/ (2\ksigma^2)}}, \\
	& f_2(z) = \sqrt{\frac{1}{\kbeta + \sum_{s=1}^t g(s) \exp\pa{ -z_s^2/ (2\ksigma^2)}}},\\
	&  f_3(z) = \frac{1}{\kbeta + \sum_{s=1}^t g(s) \exp\pa{ -z_s^2/ (2\ksigma^2)}}
	\end{align*}
	where $g:\NN^*\mapsto [0, 1]$ is an arbitrary function bounded by $1$.
	Then, for any $y, z \in \RR_+ $, we have
	\begin{align*}
	& \abs{f_1(z)-f_1(y)} \leq \frac{ 4 \norm{a}_\infty}{\ksigma} \pa{1+\sqrt{\logplus(t/\kbeta)}} \norm{z-y}_\infty \\
	& \abs{f_2(z)-f_2(y)} \leq \frac{1}{2 \kbeta^{1/2}\ksigma}\pa{1+\sqrt{\logplus(t/\kbeta)}} \norm{z-y}_\infty \\
	& \abs{f_3(z)-f_3(y)} \leq \frac{1}{\kbeta \ksigma}\pa{1+\sqrt{\logplus(t/\kbeta)}} \norm{z-y}_\infty
	\end{align*}
\end{lemma}
\begin{proof}
	We use the fact that, for any differentiable $f: \RR_+^t \to \RR$,
	$$\abs{f(x)-f(y)} \leq \sup_{z\in\RR_+^t} \norm{\nabla f(z)}_1 \norm{x-y}_\infty.$$
	We have
	\begin{align*}
		\abs{ \frac{\partial f_1(z)}{\partial z_i}  } \leq  \frac{2\norm{a}_\infty}{\ksigma^2} \frac{ g(i) z_i \exp( -z_i^2/(2\ksigma^2) ) }{ \kbeta + \sum_{s=1}^t g(s) z_s \exp( -z_s^2/(2\ksigma^2) }
	\end{align*}
	which implies, by Lemma \ref{lemma:kernel-bias},
	\begin{align*}
		\norm{\nabla f_1(z)}_1
		 & \leq  \frac{2\norm{a}_\infty}{\ksigma^2} \frac{ \sum_{i=1}^t g(i) z_i \exp( -z_i^2/(2\ksigma^2) ) }{ \kbeta + \sum_{s=1}^t g(s) z_s \exp( -z_s^2/(2\ksigma^2) } \\
		 & \leq \frac{2\norm{a}_\infty}{\ksigma^2} 2\ksigma \pa{1 + \sqrt{\logplus (t/\kbeta)}}
		 = \frac{ 4 \norm{a}_\infty}{\ksigma} \pa{1+\sqrt{\logplus(t/\kbeta)}}.
	\end{align*}
	The proofs for $f_2$ and $f_3$ are analogous.
\end{proof}

\begin{lemma}[value functions are \lipschitz continuous]
	\label{lemma:value-functions-are-lipschitz}
	Under Assumptions \ref{assumption:metric-state-space} and \ref{assumption:lipschitz-rewards-and-transitions}, for all $(k, h)$, the functions $\trueV_{k,h}^*$ and $\trueQ_{k,h}^*$ are $L_h$-\lipschitz, where $L_h \eqdef \sum_{h'=h}^H \lipR\lipP^{H-h'}$.
\end{lemma}
\begin{proof}
	This fact is proved in Lemma 4 of \cite{domingues2020} and also in Proposition 2.5 of \cite{Sinclair2019}. For completeness, we also present a proof here.

	We proceed by induction. For $h = H$, $\trueQ_{k, H}^*(x, a) = \reward_H^k(x, a)$ which is $\lipR$-\lipschitz by Assumption \ref{assumption:lipschitz-rewards-and-transitions}. Also,
	\begin{align}
		\label{eq:aux-V-is-lipschitz}
		\trueV_{k, H}^*(x) - \trueV_{k, H}^*(y)
		& = \max_a \trueQ_{k, H}^*(x, a) - \max_a\trueQ_{k, H}^*(y, a)
		 \leq \max_a \pa{ \trueQ_{k, H}^*(x, a) - \trueQ_{k, H}^*(y, a) } \\
		& \leq \max_a L_H \dist{(x, a), (y, a)}
		 \leq L_H \Sdist{x, y}, \quad \text{by Assumption \ref{assumption:metric-state-space}}
	\end{align}
	which verifies the induction hypothesis for $h=H$, since we can invert the roles of $x$ and $y$ to obtain $\abs{\trueV_{k, H}^*(x) - \trueV_{k, H}^*(y)} \leq L_H \Sdist{x, y}$.

	Now, assume that the hypothesis is true for $h+1$, \ie that $\trueV_{k,h+1}^*$ and $\trueQ_{k, h+1}^*$ are $L_{h+1}$-\lipschitz. We have
	\begin{align*}
	\trueQ_{k, h}^*(x, a) - \trueQ_{k, h}^*(x', a') &\leq \lipR\dist{(x, a), (x',a')} + \int_{\statespace}\trueV_{k, h+1}^*(y)(P_h(\mathrm{d}y|x,a) - P_h(\mathrm{d}y|x',a')) \\
	&\leq \lipR\dist{(x, a), (x',a')} + L_{h+1}\int_{\statespace}\frac{\trueV_{k, h+1}^*(y)}{L_{h+1}}(P_h(\mathrm{d}y|x,a) - P_h(\mathrm{d}y|x',a')) \\
	& \leq \left[ \lipR + \lipP \sum_{h'= h+1}^H \lipR \lipP^{H-h'} \right]\dist{(x, a), (x',a')} \\
	&  = \sum_{h'= h}^H \lipR \lipP^{H-h'}\dist{(x, a), (x',a')}
	\end{align*}
	where, in last inequality, we use fact that $\trueV_{k, h+1}^*/L_{h+1}$ is $1$-Lipschitz, the definition of the 1-Wasserstein distance and Assumption \ref{assumption:lipschitz-rewards-and-transitions}. The same argument used in Eq. \ref{eq:aux-V-is-lipschitz} shows that $\abs{\trueV_{k, h}^*(x) - \trueV_{k, h}^*(y)} \leq L_h \Sdist{x, y}$, which concludes the proof.
\end{proof}

\begin{fact}[small useful result]
	\label{fact:small-useful-result-with-min}
	\begin{align*}
	\abs{ \min(a, b) - \min(a, c)} \leq \abs{b - c}
	\end{align*}
\end{fact}
\begin{proof}
	Since $\min(x, y) = (x+y)/2 - \abs{x-y}/2$, we have
	\begin{align*}
	\min(a, b) - \min(a, c) 
	& = \frac{a+b}{2} - \frac{\abs{a-b}}{2}  - \frac{a+c}{2} + \frac{\abs{a-c}}{2}  \\
	& = \frac{b-c}{2}  + \frac{\abs{a-c} - \abs{a-b}}{2}  \\
	& \leq \frac{b-c}{2}  + \frac{\abs{b-c}}{2} \\
	& \leq \abs{b-c}
	\end{align*}
	where we used the fact that $\abs{a-c} \leq \abs{a-b} + \abs{b-c}$. By symmetry, $\min(a, c) - \min(a, b) \leq \abs{b-c}$.
	which gives us the result.
\end{proof}
\newpage

\section{Experiments}
\label{app:experiments}
\subsection{Setup}

We consider a continuous MDP whose state-space is the unit ball in $\RR^2$ with four actions, representing a move to the right, left, up or down. Each action results in a displacement of $0.1$ in the corresponding direction, plus a Gaussian noise, in both directions, of standard variation $0.01$. The agent starts at $(0, 0)$. Let $b_i^k \in \braces{0, 0.25, 0.5, 0.75, 1}$ and $x_i \in \braces{(0.8, 0.0), (0.0,0.8),(-0.8, 0.0), (0.0,-0.8)}$. We consider the following mean reward function:
\begin{align*}
\trueR_h^k(x,a) = \sum_{i=1}^4 b_i^k \max\pa{0, 1 - \frac{\norm{x-x_i}_2}{0.5}} 
\end{align*}
Every $N$ episodes, the coefficients $b_i^k$ are changed, which impact the optimal policy.

Taking $\keta = \exp(-(1/N)^{2/3})$, we tested the Gaussian kernel $\fullkernel(t, u, v) = \keta^t \exp\pa{-\dist{u, v}^2/(2\ksigma^2)}$ and a higher-order kernel $\fullkernel(t, u, v) = \keta^t \exp\pa{-(\dist{u, v}/\ksigma)^4/2}$. We set $\ksigma = 0.05$, $\maxdist=\Smaxdist=0.1$, $\kbeta=0.01$, $H=15$.

We ran the experiment with horizon $H=15$ for $2\times10^4$ episodes. Every $N$ episodes, the coefficients $b_i^k$ were changed, according to Table \ref{tab:experiment-reward-info}.

\begin{table}[ht!]
	\centering
	\caption{Value of $b_{i}^k$ according to $x_i$ and the episode $k$ }
	\label{tab:experiment-reward-info}
	\begin{tabular}{@{}l|llll@{}}
		\toprule
		episode / $x_i$            & $(0.8, 0.0)$ & $(0.0, 0.8)$ & $(-0.8, 0.0)$ & $(0.0, -0.8)$ \\ \midrule
		$\lfloor k / N\rfloor \mod 4 = 0$     &     $1/4$    &     $0$      &     $0$     &  $0$       \\
		$\lfloor k / N\rfloor \mod 4 = 1$  &     $1/4$    &     $1/2$    &     $0$     &  $0$       \\
		$\lfloor k / N\rfloor \mod 4 = 2$  &     $1/4$    &     $1/2$    &     $3/4$   &  $0$       \\
		$\lfloor k / N\rfloor \mod 4 = 3$ &     $1/4$    &     $1/2$    &     $3/4$   &  $1$       \\ \bottomrule
	\end{tabular}
\end{table}

We took $\kbeta = 0.01$ and used the following simplified exploration bonuses:
\begin{align}
\bonus_h^k(x, a) =  \frac{0.1}{\sqrt{\gencount_h^k(x, a)}} + \frac{\kbeta H}{\gencount_h^k(x, a)}
\end{align}
where the factor $0.1$ was chosen in order to ensure that the baseline is able to learn a good policy in less than $1000$ episodes, \ie before there is a change in the environment.

Additionally, to take into account the fact that the \lipschitz constant is rarely known in practical problems, we replaced the interpolation step (line 8 of Alg. \ref{alg:RS-kernel_backward_induction-ONLINE}) by a nearest-neighbor search in the representative states:
\begin{align*}
\kernsQ_{h}^k(x, a) = \tQ_{h, \mapzeta}^k(x', a'),
\quad\text{where}\quad
(x', a') = \argmin_{(\bx, \ba)\in\ReprStates_{h}^k\times\ReprActions_{h}^k}\dist{(x, a), (\bx, \ba)}.
\end{align*}

\subsection{Results}

Figures~\ref{fig:app-experiments-gaussian} and \ref{fig:app-experiments-ord4} show the total reward and the regret of \RSkerns compared to baselines for the two choices of kernel function (Gaussian and 4-th order kernel), for 3 different values of $\variationMDPtotal$, which is determined by the period $N$ of changes in the MDP (the reward changes every $N$ episodes).  

In all experiments we observe that \kernelucbvi is not able to adapt to the changes in the environment, whereas \RSkerns is able to track the behavior of the baseline \RestartBaseline which knows when the changes happen and resets the reward estimator when there is a change. 

\begin{figure}[ht!]
	\centering
	\includegraphics[width=0.3\textwidth]{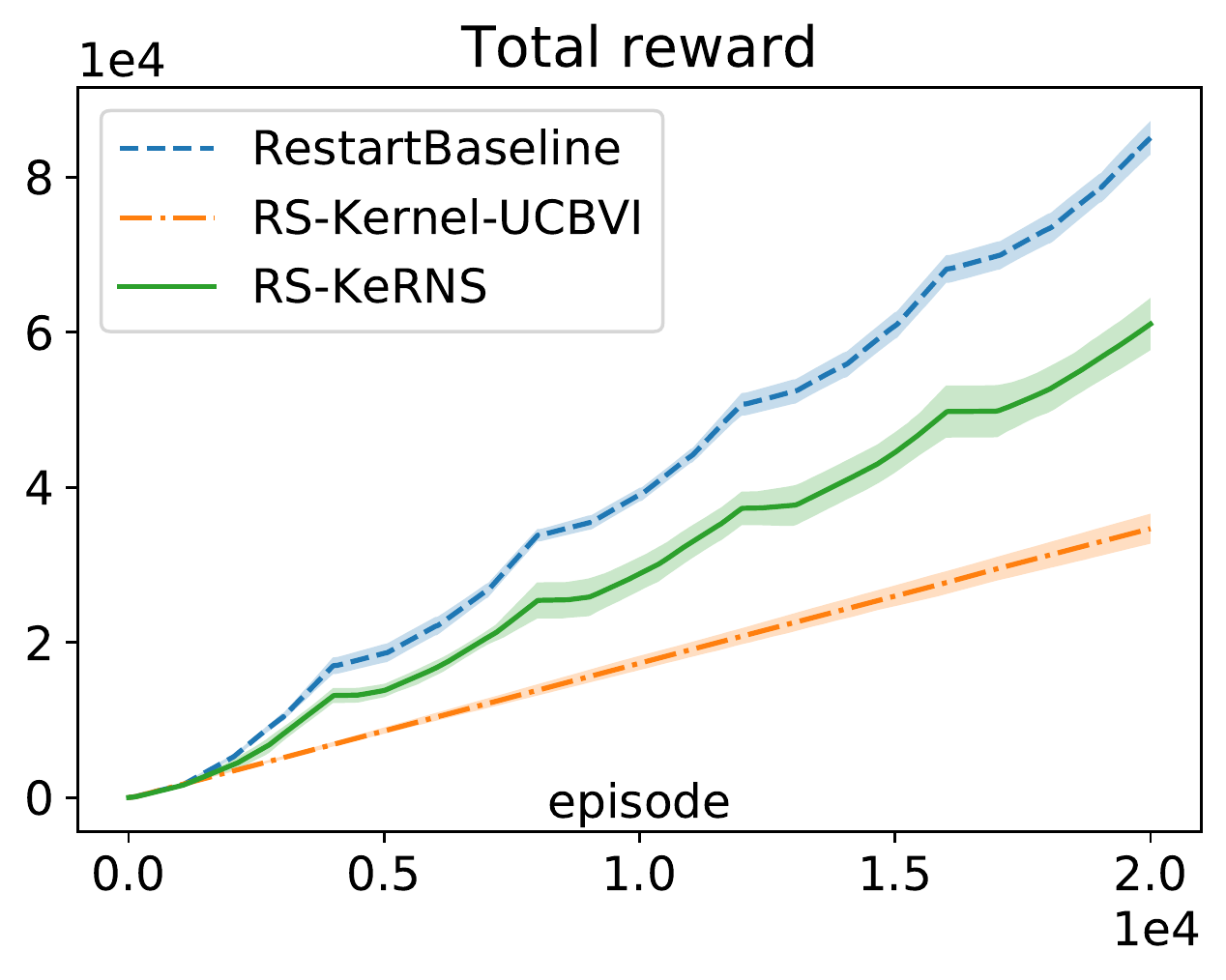}
	\includegraphics[width=0.3\textwidth]{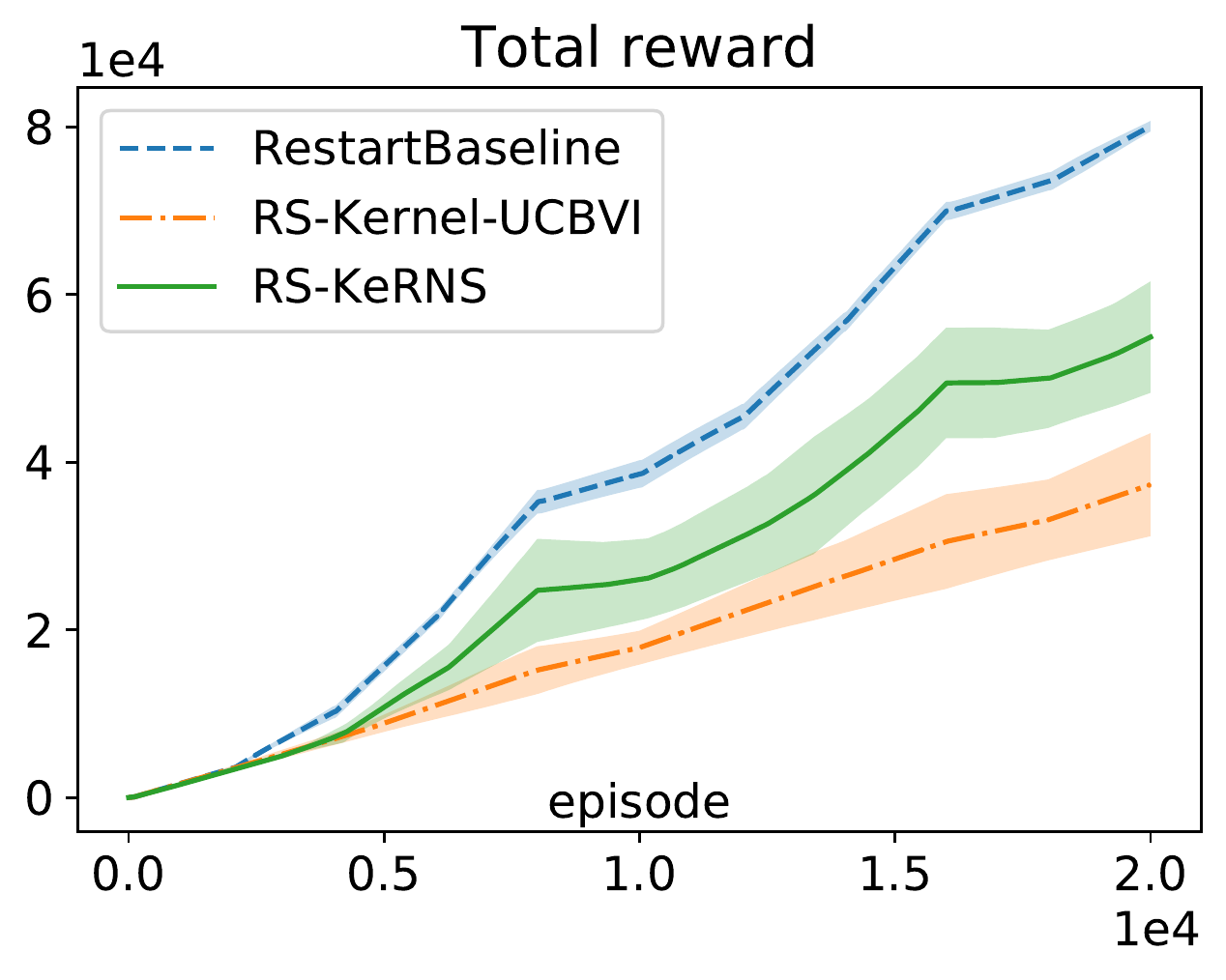}
	\includegraphics[width=0.3\textwidth]{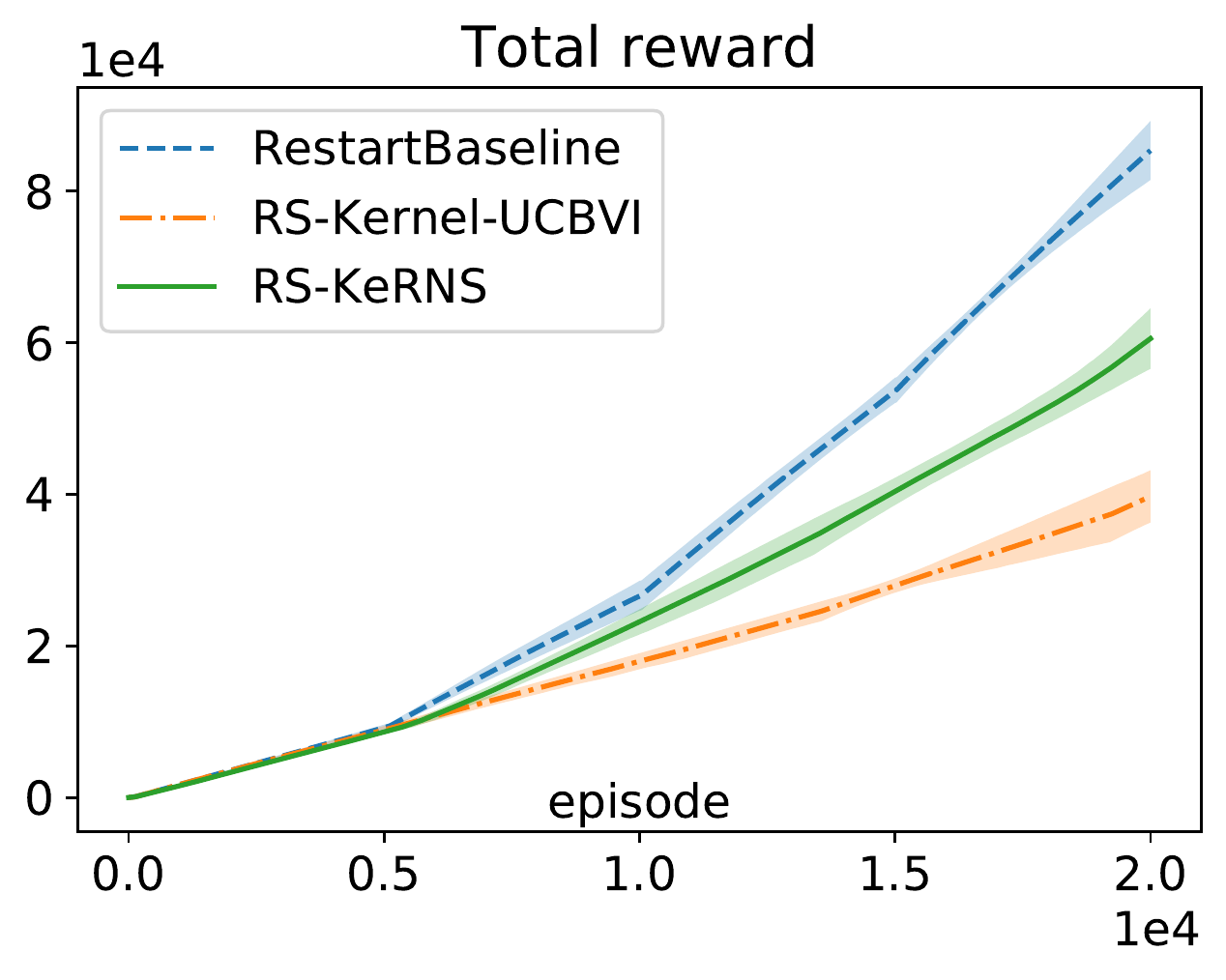}\hfill\vfill
	\includegraphics[width=0.3\textwidth]{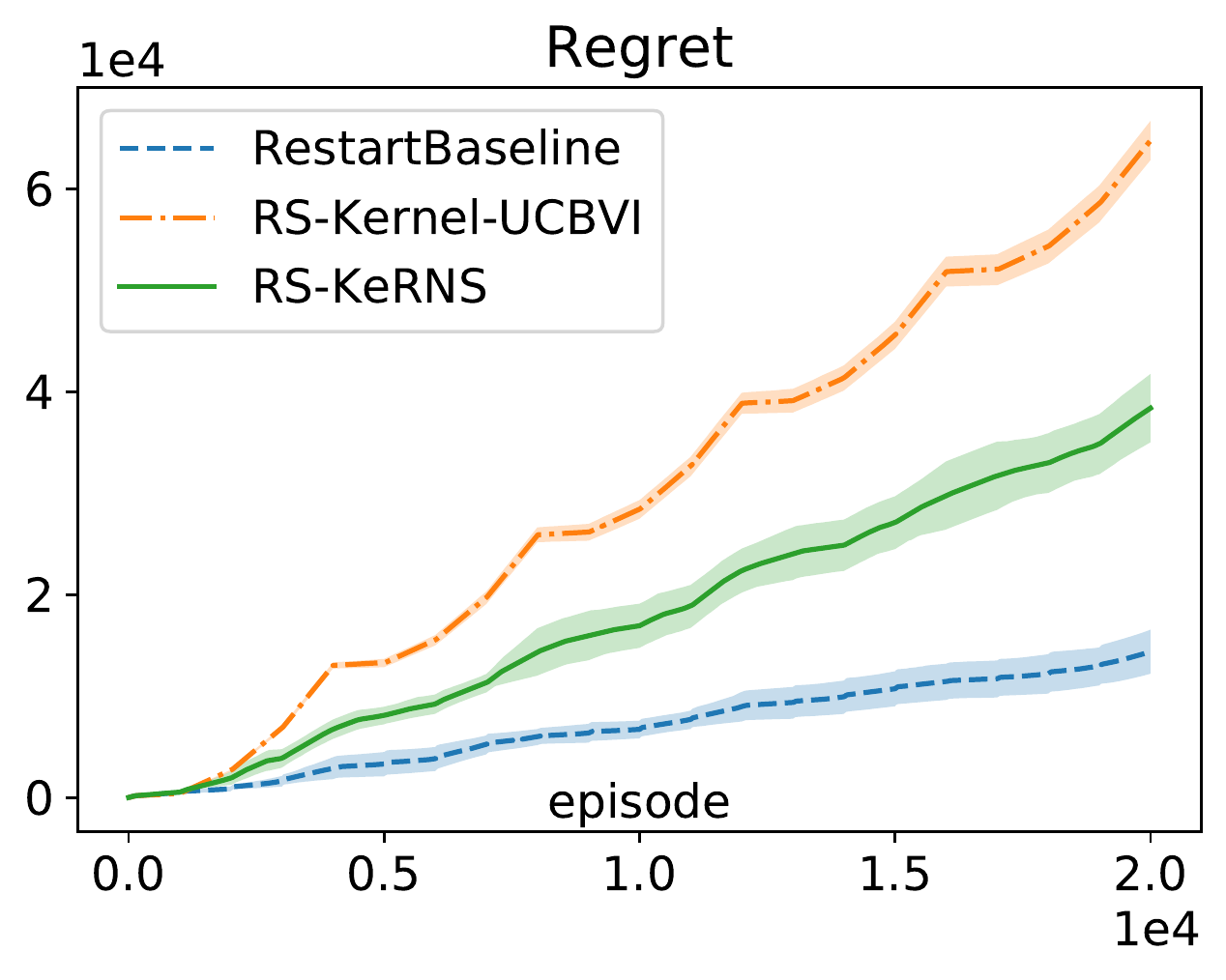}
	\includegraphics[width=0.3\textwidth]{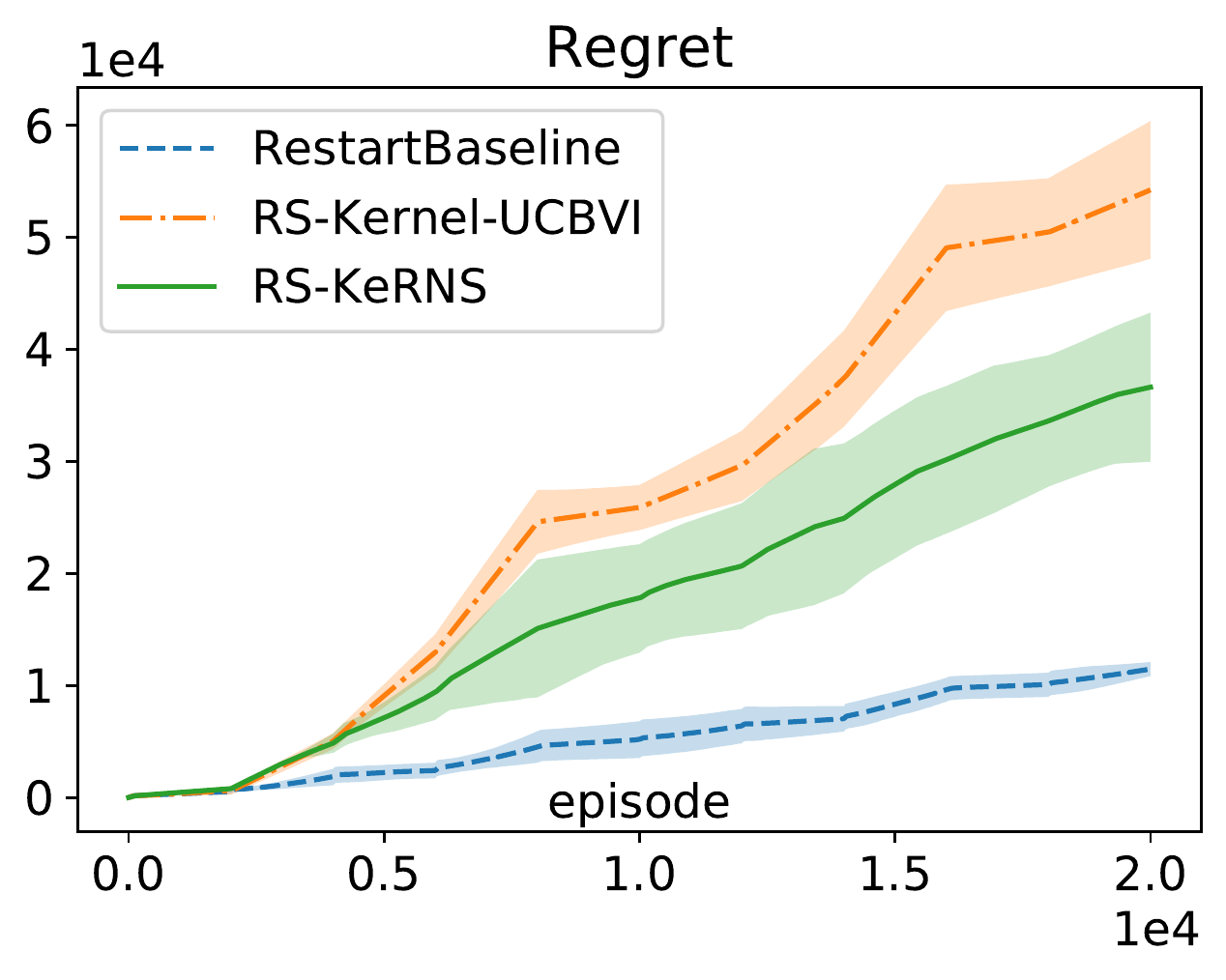}
	\includegraphics[width=0.3\textwidth]{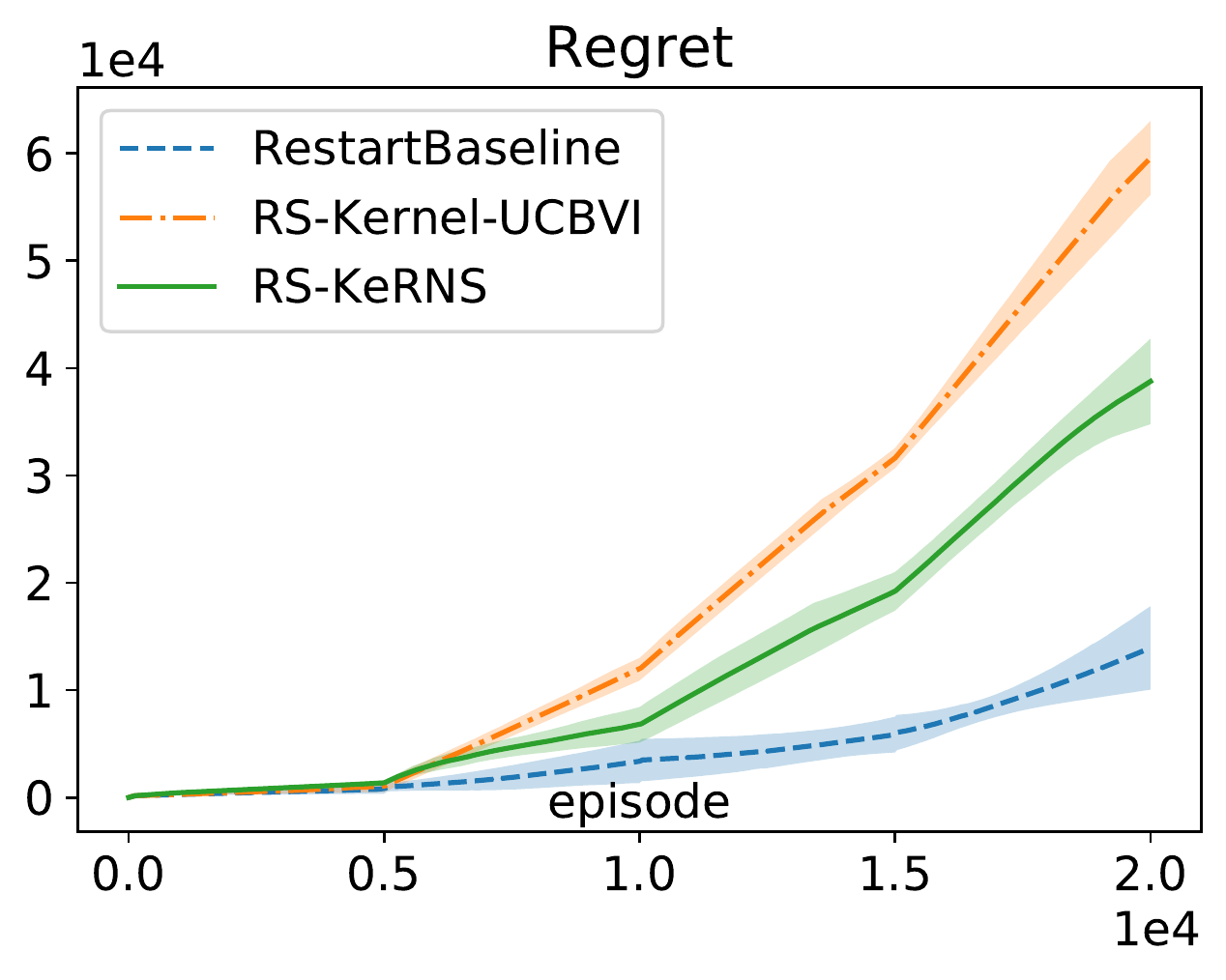}
	\caption{Total reward (top row) and regret(bottom row) of \RSkerns compared to baselines, using the Gaussian kernel $\fullkernel(t, u, v) = \keta^t \exp\pa{-\dist{u, v}^2/(2\ksigma^2)}$. The figures on the left, in the middle, and on the right correspond to $N=1000$, $N=2000$ and $N=5000$, respectively, where $N$ is the period of the changes in the MDP. Average over 4 runs.}
	\label{fig:app-experiments-gaussian}
\end{figure}

\begin{figure}[ht!]
	\centering
	\includegraphics[width=0.3\textwidth]{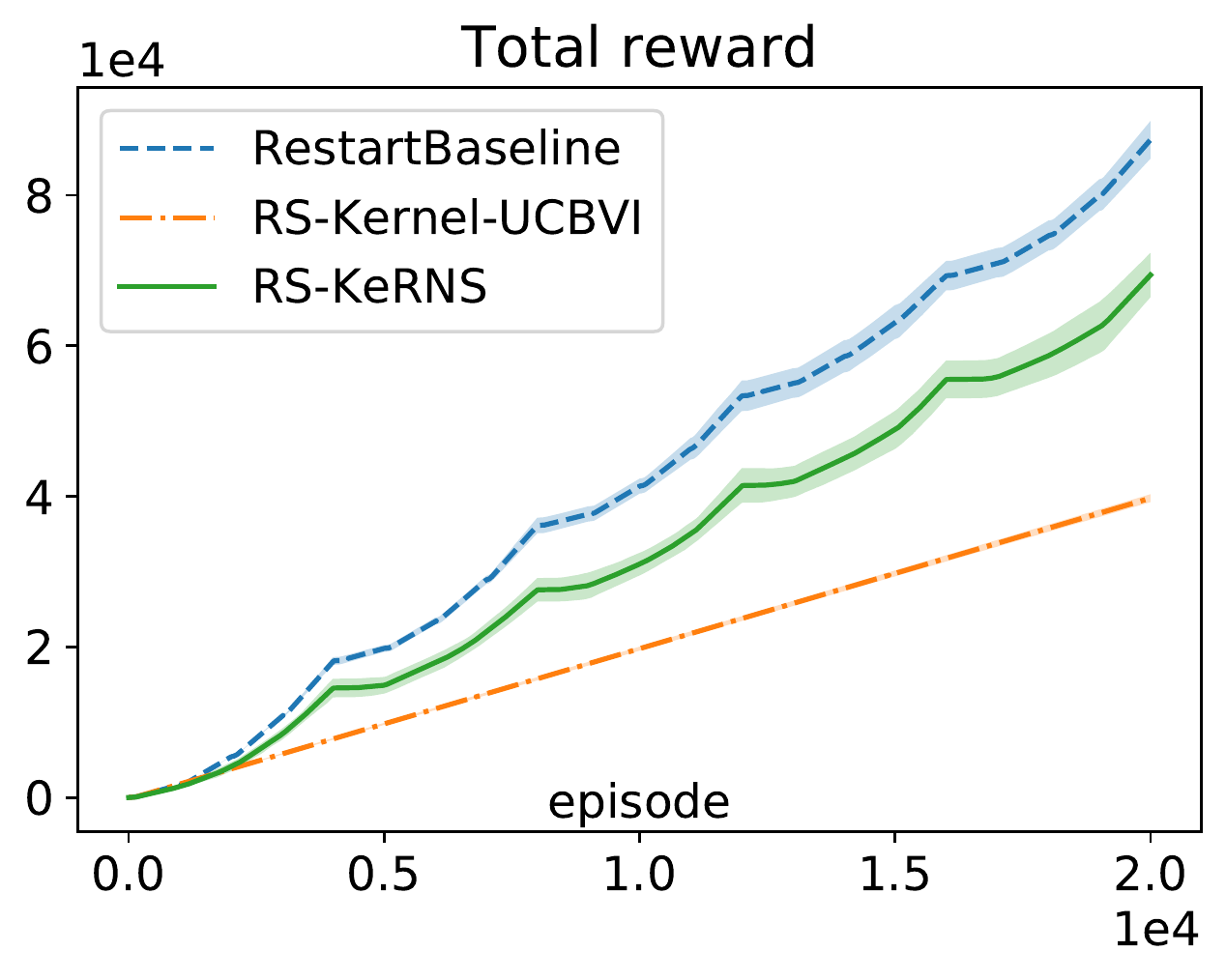}
	\includegraphics[width=0.3\textwidth]{./source_arxiv_march_2022/figures/N2000_ORD4_total_reward.pdf}
	\includegraphics[width=0.3\textwidth]{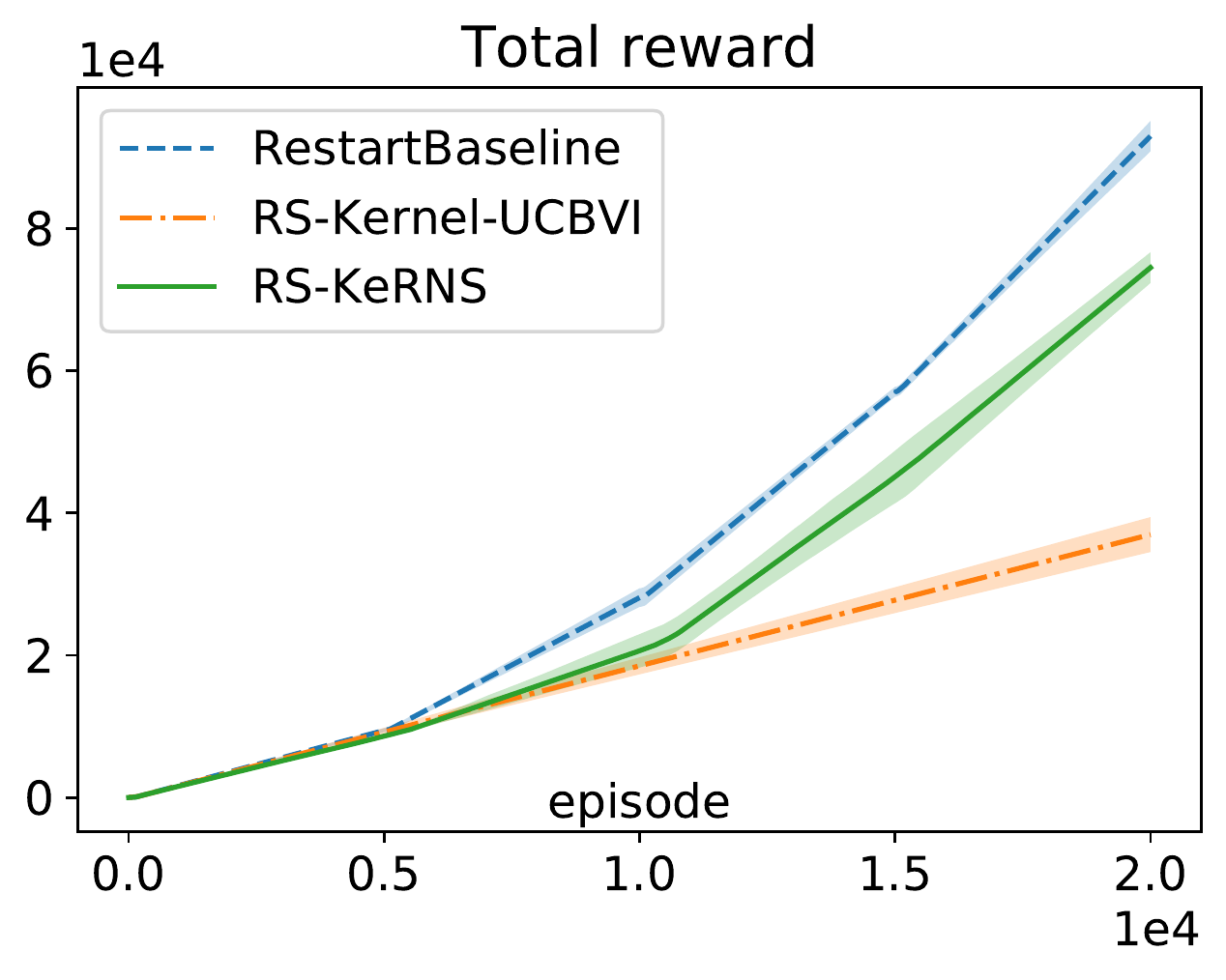}\hfill\vfill
	\includegraphics[width=0.3\textwidth]{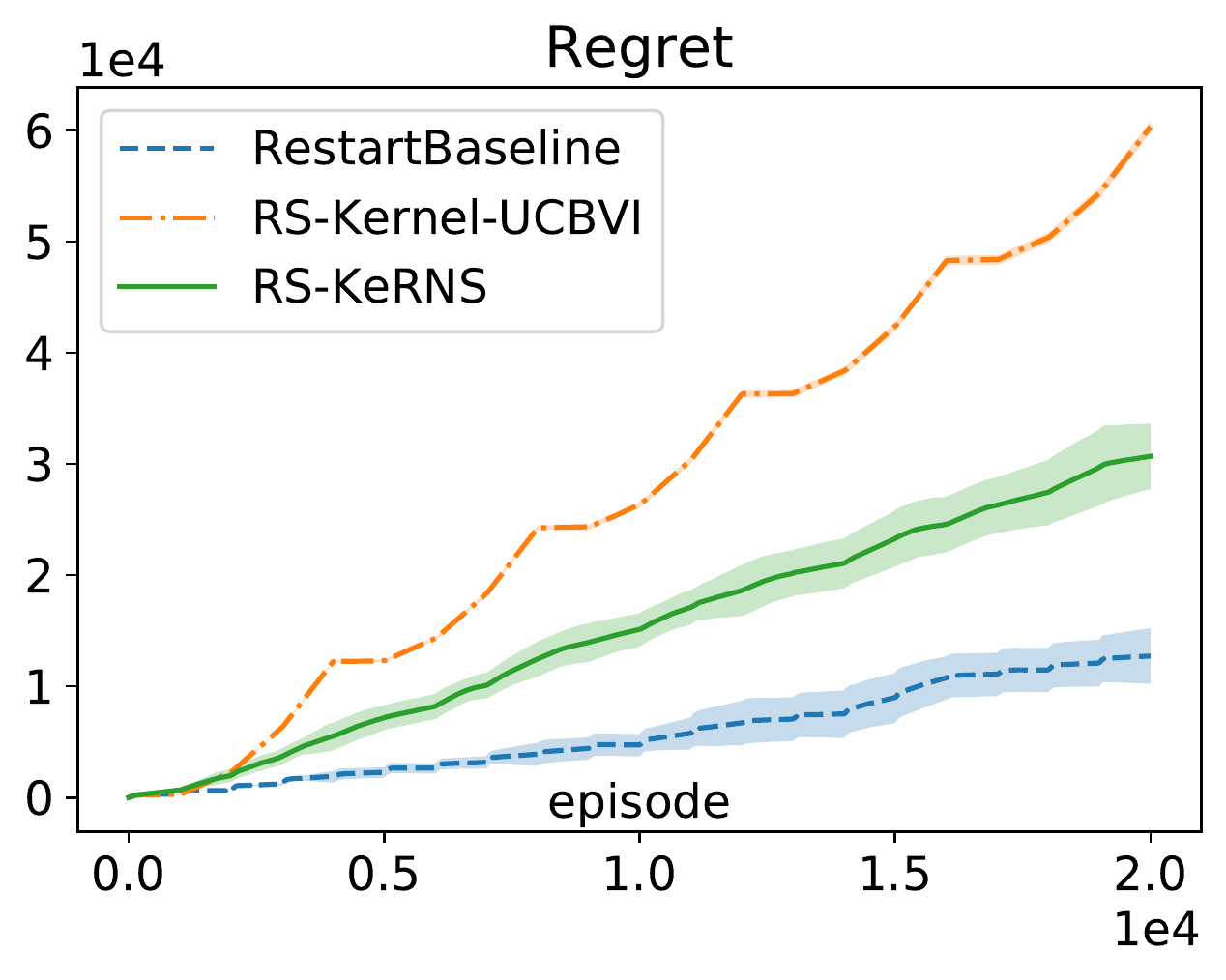}
	\includegraphics[width=0.3\textwidth]{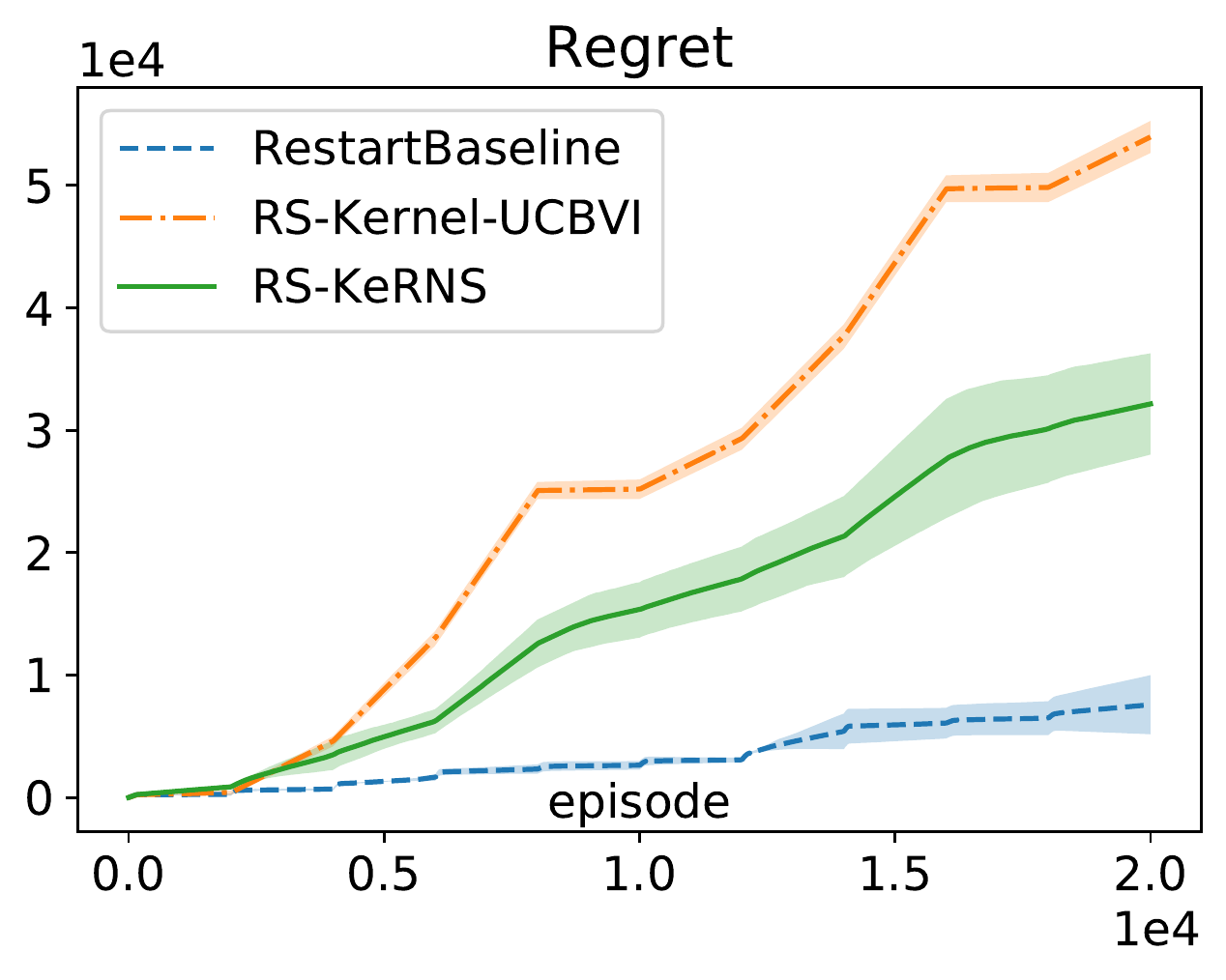}
	\includegraphics[width=0.3\textwidth]{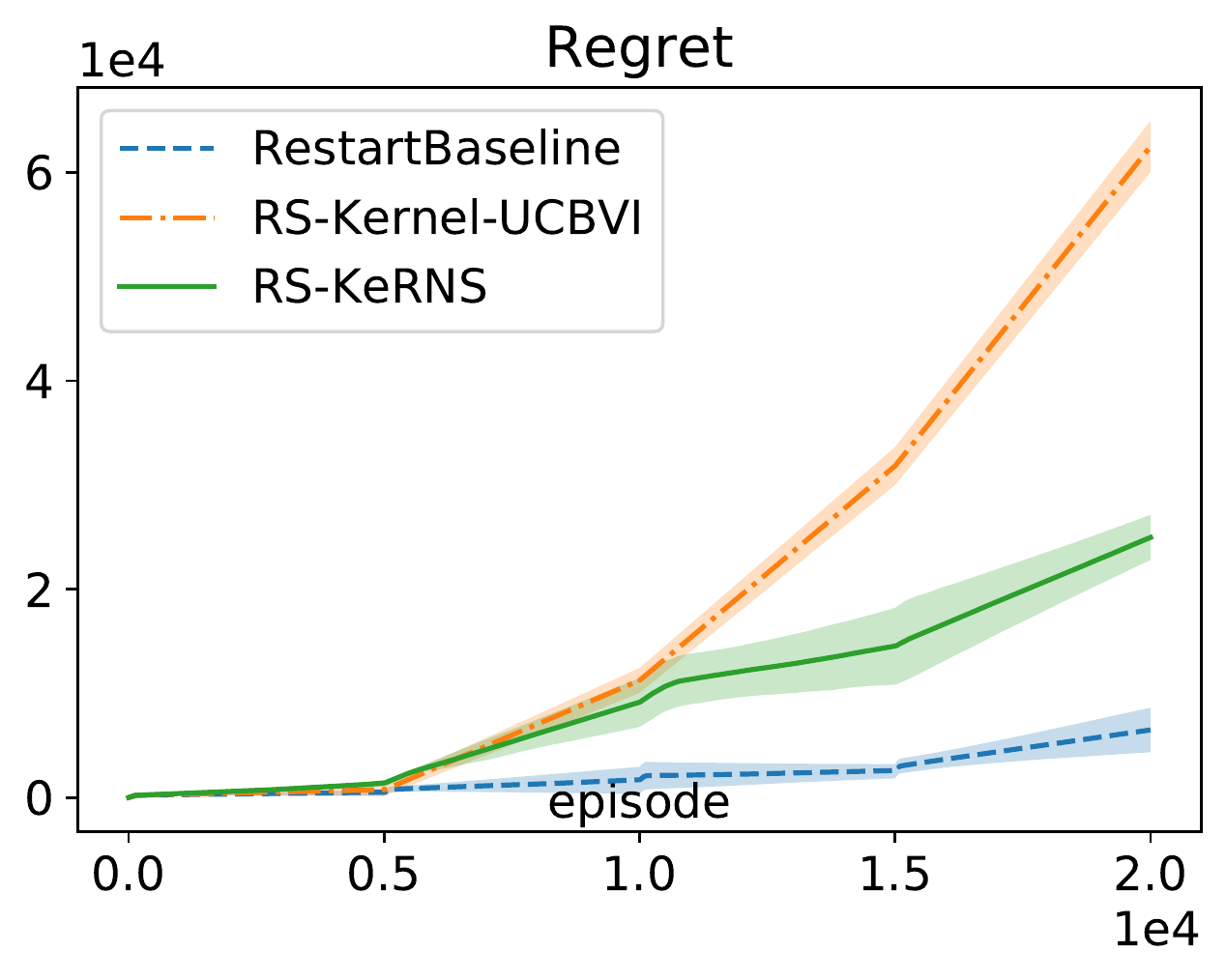}
	\caption{Total reward (top row) and regret(bottom row) of \RSkerns compared to baselines, using the kernel $\fullkernel(t, u, v) = \keta^t \exp\pa{-(\dist{u, v}/\ksigma)^4/2}$. The figures on the left, in the middle, and on the right correspond to $N=1000$, $N=2000$ and $N=5000$, respectively, where $N$ is the period of the changes in the MDP. Average over 4 runs.}
	\label{fig:app-experiments-ord4}
\end{figure}

\end{document}